\documentclass[twoside,11pt]{article}

\usepackage[abbrvbib, preprint]{jmlr2e}

\usepackage{bbm}
\usepackage{hyperref}
\usepackage{amsmath}
\usepackage{amssymb}
\usepackage{mathrsfs}
\usepackage{caption,subcaption}
\usepackage{booktabs}
\usepackage{multicol}
\usepackage{multirow}
\usepackage{xcolor}
\usepackage{color}
\usepackage{algorithm}
\usepackage{algorithmic}
\usepackage{listings}
\usepackage{enumerate}
\usepackage[export]{adjustbox}
\graphicspath{{figure/}}
\newcommand{\dd}{\mathrm{d}}
\newcommand{\one}{\mathbbm{1}}
\newcommand{\E}{\mathbb{E}}
\newcommand{\p}{\mathbb{P}}

\newcommand{\f}{\mathcal{F}}

\newcommand{\sX}{X^{\mathscr G,\bm\pi}}
\newcommand{\sa}{a^{\mathscr G,\bm\pi}}
\newcommand{\sXp}{X^{\mathscr G,\bm\pi'}}
\newcommand{\sap}{a^{\mathscr G,\bm\pi'}}

\newcommand{\bm}[1]{\boldsymbol{#1}}

\newtheorem{assumption}{Assumption}
\newtheorem{eg}{Example}

\usepackage{lastpage}
\jmlrheading{24}{2023}{1-\pageref{LastPage}}{7/22; Revised 4/23}{5/23}{22-0755}{Yanwei Jia and Xun Yu Zhou}

\ShortHeadings{Continuous-Time q-Learning}{Jia and Zhou}
\firstpageno{1}

\begin{document}

\title{q-Learning in Continuous Time}

\author{\name Yanwei Jia \email yanweijia@cuhk.edu.hk \\
\addr 
Department of Systems Engineering and Engineering Management\\
The Chinese University of Hong Kong\\
Shatin, NT, Hong Kong
\AND
\name Xun Yu Zhou \email xz2574@columbia.edu \\
\addr Department of Industrial Engineering and Operations Research \&\\
The Data Science Institute\\
Columbia University\\
New York, NY 10027, USA}

\editor{Marc Bellemare}

\maketitle

\begin{abstract}
We study the continuous-time counterpart of Q-learning for reinforcement learning (RL) under the entropy-regularized,  exploratory diffusion process formulation introduced by \cite{wang2020reinforcement}. As the conventional (big) Q-function collapses in continuous time, we consider its first-order approximation and coin the term ``(little) q-function". This function is related to the instantaneous advantage rate function as well as  the Hamiltonian. We develop a ``q-learning" theory around the q-function that is independent of time discretization. Given a stochastic policy, we jointly characterize the associated q-function and value function by martingale conditions of certain stochastic processes, in both on-policy and off-policy settings. We then apply the theory to devise different actor--critic algorithms for solving underlying RL problems, depending on whether or not the density function of the Gibbs measure generated from the q-function can be computed explicitly. One of our algorithms interprets the well-known Q-learning algorithm SARSA, and another recovers a policy gradient (PG) based continuous-time  algorithm proposed in \cite{jia2021policypg}. Finally, we conduct simulation experiments to compare the performance of our algorithms with those of PG-based algorithms in \cite{jia2021policypg} and time-discretized  conventional Q-learning algorithms.
\end{abstract}


\section{Introduction}

A recent series of papers, \citet{wang2020reinforcement} and \citet{jia2021policy,jia2021policypg}, aim at laying an overarching theoretical foundation for reinforcement learning (RL) in continuous time with continuous state space (diffusion processes) and possibly continuous action space. Specifically, \citet{wang2020reinforcement} study how to explore strategically (instead of blindly) by formulating an entropy-regularized, distribution-valued stochastic control problem for diffusion processes. \citet{jia2021policy} solve the policy evaluation (PE) problem, namely to learn the value function of a given stochastic policy, by characterizing it as a martingale problem. \citet{jia2021policypg} investigate the policy gradient (PG) problem, that is, to determine the gradient of a learned value function with respect to the current policy, and show that PG is mathematically a simpler PE problem and thus solvable by the martingale approach developed in \citet{jia2021policy}. Combining these theoretical results naturally leads to various online and offline actor--critic (AC) algorithms for general model-free (up to the diffusion dynamics) RL tasks, where one learns value functions and stochastic policies simultaneously and alternatingly. Many of these algorithms recover and/or interpret some well-known existing RL algorithms for Markov decision processes (MDPs) that were often put forward  in a heuristic and ad hoc manner.

PG updates a policy along the gradient ascent direction to improve it; so PG is an instance of the general policy improvement (PI) approach. On the other hand, one of the earliest and most popular methods for PI  is \textit{Q-learning} \citep{watkins1989learning,watkins1992q}, whose key ingredient is to learn the \textit{Q-function}, a function of state and action. 
The learned Q-function is then maximized over actions at each state to achieve  improvement of the current policy.
The resulting Q-learning algorithms, such as the Q-table, SARSA and DQN (deep Q-network), are widely used for their simplicity and effectiveness in many applications.\footnote{Q-learning is typically combined with various forms of function approximations. For example, linear function approximations \citep{melo2008analysis,zou2019finite}, kernel-based nearest neighbor regression \citep{shah2018q}, or deep neural networks \citep{mnih2015human,fan2020theoretical}. Q-learning is not necessarily restricted to finite action space in literature; for example, Q-learning with continuous action spaces is studied in \citet{gu2016continuous}. However, \citet{duan2016benchmarking} report and claim that the standard Q-learning algorithm may become less efficient for continuous action spaces.} Most importantly though, compared with PG, one of the key advantages of using Q-learning is that it works both on-policy and off-policy \citep[Chapter 6]{sutton2011reinforcement}.

The Q-function, by definition, is a function of the current state and action, assuming that the agent takes a  particular action at the {\it current time} and  follows through a given control policy starting from the {\it next time step}. Therefore, it is {\it intrinsically} a notion in {\it discrete} time; that is why Q-learning has been predominantly studied for discrete-time MDPs. In a continuous-time setting, the Q-function collapses to the value function that is independent of action and hence cannot be used to rank and choose the {\it current} actions. Indeed, \cite{tallec2019making} opine that ``there is no Q-function in continuous time''. On the other hand, one may propose discretizing time to obtain a discretized Q-function and then apply the existing Q-learning algorithms. However, \cite{tallec2019making} show experimentally that this approach is very sensitive to time discretization and performs poorly with small time steps.
\citet{kim2021hamilton} take a different approach: they  include  the action as a state variable in the continuous-time system by restricting the action process to be absolutely continuous in time with a bounded growth rate. Thus Q-learning becomes a policy evaluation problem with the state--action pair as the new state variable. However, they consider only deterministic dynamic systems and discretize upfront the continuous-time problem. 
Crucially, that the action must be absolutely continuous is unpractically  restrictive because optimal actions are often only measurable in time and have unbounded variations (such as the bang--bang controls) even for deterministic systems.

As a matter of fact, the aforementioned  series of papers, \cite{wang2020reinforcement} and \cite{jia2021policy,jia2021policypg},
are characterized by carrying out all the theoretical analysis within the continuous setting and taking observations at discrete times (for computing the total reward over time)  only when implementing the algorithms. Some advantages of this approach, compared with discretizing time upfront and then applying existing MDP results, are discussed extensively in \cite{doya2000reinforcement,wang2020reinforcement,jia2021policy,jia2021policypg}. More importantly, the pure continuous-time approach  minimizes or eliminates the impact of time-discretization on learning which, as discussed above, becomes critical especially for continuous-time Q-learning.

Now, if we are to study Q-learning strictly within the continuous-time framework without time-discretization or action restrictions, then the first question is what a proper Q-function should be. When time is discretized, \cite{baird1993advantage} and \cite{mnih2016asynchronous} define the so-called \textit{advantage function} that is the difference in value between taking a particular action versus following a given policy at any state; that is, the difference between the state--action value and the state value. \cite{baird1993advantage,baird1994reinforcement} and \cite{tallec2019making} notice that such an advantage function can be properly scaled with respect to the size of time discretization whose continuous-time limit exists,
called the \textit{(instantaneous) advantage rate function} and can be learned. The advantage updating performs better than conventional Q-learning, as shown numerically in \citet{tallec2019making}. However, these
papers again consider only deterministic dynamic systems and do not fully develop a theory of this rate function.
Here in the present paper, we attack the problem directly from the continuous-time perspective. In particular, 
we rename the advantage rate function as the (little) {\it q-function}.\footnote{We use the {\it lower case} letter ``$q$" to denote this rate function, much in the same way as the typical use of the lower case letter $f$ to denote a probability density function which is the first-order derivative of the cumulative distribution function usually denoted by the upper case letter $F$.} This q-function in the finite-time episodic setting is a function of the time--state--action triple under a given policy, analogous to the conventional Q-function. However, it is a continuous-time notion because  {\it it does not depend on time-discretization}. This feature is a vital advantage for learning algorithm design as it avoids the sensitivity with respect to the observation and intervention frequency (the step size of time-discretization).

For the entropy-regularized, exploratory setting for diffusion processes (first formulated in \citealt{wang2020reinforcement}) studied in the paper, the q-function of a given policy turns out to be the \textit{Hamiltonian} that includes the infinitesimal generator of the dynamic system and the instantaneous reward function (see \citealt{YZbook} for details), plus the temporal dispersion that consists of the time-derivative of the value function and the depreciation from discounting.
The paper aims to develop a comprehensive  {\it q-learning} theory around this q-function and accordingly design alternative RL algorithms other than the PG-based AC algorithms in \citet{jia2021policypg}. There are three main questions. The first is to characterize the q-function. For discrete-time MDPs, the PE method is used to learn both the value function and the Q-function because both functions are characterized by similar Bellman equations. By contrast, in the continuous-time diffusion setting, a Bellman equation is only available for the value function (also known as the Feynman--Kac formula, which yields that the value function solves a second-order linear partial differential equation). The second question is to establish a policy improvement theorem based on the q-function and to design corresponding AC  algorithms in a sample/data-driven, model-free fashion. The third question is whether the capability of both on- and off-policy learning, a key advantage of Q-learning, is intact in the continuous-time setting.

We provide complete answers to all these questions. First, given a stochastic policy along with its (learned) value function, the corresponding q-function is characterized by the martingale condition of a certain stochastic process with respect to an enlarged filtration (information field) that includes both the original environmental noise and the action randomization. Alternatively, the value function and the q-function can be {\it jointly} determined by the martingality of the same process. These results suggest that the martingale perspective for PE and PG in \citet{jia2021policy,jia2021policypg} continues to work for q-learning.
In particular, we devise several algorithms to learn these functions in the same way as
the martingale conditions are employed to generate PE and PG algorithms in  \citet{jia2021policy,jia2021policypg}. Interestingly, a temporal--difference (TD) learning algorithm designed from this perspective recovers and, hence, interprets a version of the well-known SARSA algorithm in the discrete-time Q-learning literature.
Moreover, inspired by these continuous-time martingale characterizations, in Appendix A we present
martingale conditions for Q-learning of {\it discrete-time} MDPs, which are new to our best knowledge. This demonstrates  that the continuous-time RL study may offer new perspectives even in discrete time.
Second, we prove that a Gibbs sampler with a properly scaled current q-function (independent of time discretization) as the exponent of its density function improves upon the current policy. This PI result translates into a policy-updating algorithm when the normalizing constant of the Gibbs measure can be explicitly computed. Otherwise, if the normalizing constant is unavailable, we prove a stronger PI theorem in terms of the Kullback--Leibler (KL) divergence, analogous to a result for discrete-time MDPs in \citet{haarnoja2018softconference}. One of the algorithms out of this theorem happens to recover a PG-based AC algorithm in \citet{jia2021policypg}.
Finally, we prove that the aforementioned martingale conditions hold for both the on-policy and off-policy settings. More precisely, the value function and the q-function associated with a given {\it target} policy can be learned based on a dataset generated by either the target policy itself or an arbitrary {\it behavior} policy. So, q-learning indeed works off-policy as well.

%

The rest of the paper is organized as follows. In Section \ref{sec:setup} we review \citet{wang2020reinforcement}'s  entropy-regularized, exploratory formulation for RL in continuous time and space, and present some useful preliminary results. In Section \ref{sec:foundation}, we establish the q-learning theory, including the motivation and definition of the q-function and its on-/off-policy martingale characterizations.
q-learning algorithms are developed, depending on whether or not the normalizing constant of the Gibbs measure is available, in Section \ref{subsec:qv learn} and Section \ref{sec:noncomputable}, respectively.
Extension to ergodic tasks is presented in  Section \ref{sec:extension}. In Section \ref{sec:applications}, we illustrate the proposed algorithms and compare them with PG-based algorithms and time-discretized, conventional Q-learning algorithms on two specific examples with simulated data.\footnote{The code to reproduce our simulation studies is publicly available at \url{https://www.dropbox.com/sh/34cgnupnuaix15l/AAAj2yQYfNCOtPUc1_7VhbkIa?dl=0}. } Finally, Section \ref{sec:conclusion} concludes. The Appendix contains various supplementary materials   and proofs of statements in the main text.

\section{Problem Formulation and Preliminaries}
\label{sec:setup}

Throughout this paper, by convention all vectors are {\it column} vectors unless otherwise specified, and $\mathbb{R}^k$ is the space of all $k$-dimensional  vectors (hence $k\times 1$ matrices). We use $\mathbb{S}^k_{++}$ to denote all the $k\times k$ symmetric and positive definite matrices. Given  two matrices $A$ and $B$ of the same size,
denote by $A \circ B$ their inner product, by $|A|$ the Euclidean/Frobenius norm of $A$, by $\det A$ the determinant when $A$ is a square matrix, and by $A^\top$ the transpose of any matrix $A$. For $A\in \mathbb{S}^k_{++}$, we write $\sqrt{A} = UD^{1/2}U^\top$, where $A = UDU^\top$ is its eigenvalue decomposition with $U$  an orthogonal matrix, $D$ a diagonal matrix, and $D^{1/2}$  the diagonal matrix whose entries are the square root of those of $D$.
We use $f=f(\cdot)$ to denote the {\it function} $f$, and $f(x)$ to denote
the {\it function value} of $f$ at $x$. We denote by $\mathcal{N}(\mu,\Sigma)$ the probability density function of the multivariate normal distribution with mean vector $\mu$ and covariance matrix $\Sigma$. Finally,
for any stochastic process $X=\{X_{s},$ $s\geq 0\}$, we denote by $\{\f^X_s\}_{s\geq 0}$ the natural filtration generated by $X$.


\subsection{Classical model-based formulation}

For readers' convenience, we first recall the classical, model-based stochastic control formulation.

Let $d,n$ be given positive integers, $T>0$, 
and $b: [0,T]\times \mathbb{R}^d\times \mathcal{A} \mapsto \mathbb{R}^d$ and $\sigma:
[0,T]\times \mathbb{R}^d\times \mathcal{A}\mapsto \mathbb{R}^{d\times n}$ be given functions, where $\mathcal{A}\subset\mathbb{R}^m$ is the action/control space. The classical stochastic control problem is to control the {\it state} (or {\it feature}) dynamics governed by  a stochastic differential equation (SDE), defined on a filtered probability space $\left( \Omega ,\mathcal{F},\mathbb{P}^W; \{\mathcal{F}_s^W\}_{s\geq0}\right) $ along with a standard $n$-dimensional Brownian motion  $W=\{W_{s},$ $s\geq 0\}$:
\begin{equation}
\label{eq:model classical}
\dd X_s^{\bm{a}} = b(s,X_s^{\bm{a}},\bm a_s)\dd s + \sigma(s,X_s^{\bm{a}},\bm a_s) \dd W_s,\
s\in [0,T],
\end{equation}
where $\bm a_s$ stands for the agent's action at time $s$. 

The goal of a stochastic control problem is, for each initial time--state pair $(t,x)\in [0,T)\times \mathbb{R}^d$ of \eqref{eq:model classical}, to find the optimal $\{\mathcal{F}_s^W\}_{s\geq0}$-progressively measurable (continuous-time) sequence of actions $\bm a = \{\bm a_s,t\leq s \leq T\}$  -- also called the optimal control or optimal strategy -- that maximizes the expected total discounted reward:
\begin{equation}
\label{eq:objective}
\E^{\p^W}\left[\int_t^T e^{-\beta(s-t)} r(s,X_s^{\bm a},\bm a_s)\dd s + e^{-\beta (T-t)} h(X_T^{\bm a})\Big|X_t^{\bm a}= x\right],
\end{equation}
where $r$ is an (instantaneous) reward function (i.e., the expected rate of reward conditioned on time, state, and action), $h$ is the lump-sum reward function applied at the end of the planning period $T$, and $\beta\geq 0$ is a constant discount factor that measures the time-value of the payoff or the impatience level of the agent.

Note in the above, the state process $X^{\bm a} = \{X^{\bm a}_s, t\leq s \leq T\}$ also depends on $(t,x)$. However, to ease notation, here and henceforth   we use $X^{\bm a}$ instead of $X^{t,x,\bm a}= \{X^{t,x,\bm a}_s, t\leq s \leq T\}$ to denote the solution to SDE \eqref{eq:model classical} with initial condition $X_t^{\bm a} = x$ whenever no ambiguity may arise.


The (generalized) {\it Hamiltonian} is a function $H:[0,T] \times \mathbb{R}^d \times \mathcal{A} \times \mathbb{R}^d \times \mathbb{R}^{d\times d} \to \mathbb{R}$ associated with problem \eqref{eq:model classical}--\eqref{eq:objective} defined as \citep{YZbook}:
\begin{equation}
\label{eq:H}
H(t,x,a,p,q) = b(t,x,a) \circ p + \frac{1}{2}\sigma\sigma^\top(t,x,a) \circ q + r(t,x,a).\end{equation}

We make the same assumptions as in \citet{jia2021policypg} to ensure theoretically the well-posedness of the stochastic control problem \eqref{eq:model classical}--\eqref{eq:objective}.
\begin{assumption}
\label{ass:dynamic}
The following conditions for the state dynamics and reward functions hold true:

\begin{enumerate}[(i)]
\item $b,\sigma,r,h$ are all continuous functions in their respective arguments;
\item $b,\sigma$ are uniformly Lipschitz continuous in $x$, i.e., for $\varphi \in\{ b,\ \sigma\}$, there exists a constant $C>0$ such that
\[ |\varphi(t,x,a) - \varphi(t,x',a)| \leq C|x-x'|,\;\;\forall (t,a)\in [0,T] \times \mathcal{A},\;\forall x,x'\in \mathbb{R}^d; \]
\item $b,\sigma$ have linear growth in $x$, i.e., for $\varphi \in\{ b,\ \sigma\}$, there exists a constant $C>0$ such that
\[|\varphi(t,x,a)| \leq C(1+|x|) ,\;\;\forall (t,x,a)\in [0,T] \times \mathbb{R}^d\times \mathcal{A};\]
\item $r$ and $h$ have polynomial growth  in $(x,a)$ and $x$ respectively,  i.e., 
there exist constants $C>0$ and $\mu\geq 1$ such that
\[
|r(t,x,a)| \leq C(1+|x|^{\mu} + |a|^{\mu}) ,\;\;|h(x)| \leq C(1+|x|^{\mu}),\; \forall (t,x,a)\in [0,T] \times \mathbb{R}^d \times \mathcal{A}.\]
\end{enumerate}
\end{assumption}

Classical model-based stochastic control theory has been well developed (e.g., \citealt{fleming2006controlled} and \citealt{YZbook}) to solve the above problem, {\it assuming} that
the functional forms of $b,\sigma,r,h$ are all given and known. A centerpiece of the theory is the Hamilton--Jacobi--Bellman (HJB) equation
\begin{equation}
\label{eq:hjb pde}
\frac{\partial V^*}{\partial t}(t,x) + \sup_{a\in \mathcal{A}} H\left( t,x,a,\frac{\partial V^*}{\partial x}(t,x),\frac{\partial^2 V^*}{\partial x^2}(t,x) \right) - \beta V^*(t,x) = 0,\ V^*(T,x) = h(x).
\end{equation}
Here, $\frac{\partial V^*}{\partial x}\in \mathbb{R}^d$ is the gradient and $\frac{\partial^2 V^*}{\partial x^2}\in \mathbb{R}^{d\times d}$ is the Hessian matrix.

Under proper conditions, the unique solution (possibly in the sense of viscosity solution) to \eqref{eq:hjb pde} is the optimal value function to the problem \eqref{eq:model classical}--\eqref{eq:objective}, i.e.,
\[ V^*(t,x) = \sup_{\{\bm a_s,t\leq s\leq T\}} \E^{\p^W}\left[\int_t^T e^{-\beta(s-t)} r(s,X_s^{\bm a},\bm a_s)\dd s + e^{-\beta (T-t)} h(X_T^{\bm a})\Big|X_t^{\bm a}= x\right]. \]
Moreover, the following function,  which maps a time--state pair to an action:
\begin{equation}
\label{eq:optimalcontrol}
a^*(t,x) = \arg\sup_{a\in \mathcal{A}} H\left( t,x,a,\frac{\partial V^*}{\partial x}(t,x),\frac{\partial^2 V^*}{\partial x^2}(t,x) \right)
\end{equation}
is the optimal (non-stationary, feedback) control {\it policy} of the problem. In view of the definition of the Hamiltonian (\ref{eq:H}), the maximization in (\ref{eq:optimalcontrol}) indicates that, at any given time and state,  one ought to
maximize a suitably weighted average of the (myopic) instantaneous reward and the risk-adjusted, (long-term) positive impact on the system dynamics. See
\citet{YZbook} for a detailed account of this theory and many discussions on its economic interpretations and implications.

\subsection{Exploratory formulation  in reinforcement learning}
\label{sec:filtration discuss}

We now present the RL formulation of the problem  to be studied in this paper. In the RL setting, the agent has partial or simply no knowledge about the environment (i.e. the functions $b,\sigma,r,h$). What she can do is ``trial and error" -- to try a (continuous-time) sequence of actions ${\bm a} = \{a_s,t\leq s \leq T\}$, observe the corresponding state process $X^{\bm a} = \{X_s^{\bm a},t\leq s \leq T\}$ along with both a stream of discounted running rewards $\{e^{-\beta(s-t)} r(s,X_s^{\bm a},a_s),t\leq s \leq T\}$ and a discounted, end-of-period lump-sum reward $e^{-\beta (T-t)}h(X_T^{\bm a})$ where $\beta$ is a given, known discount factor, and continuously update and improve her actions based on these observations.\footnote{This procedure applies to both the offline and online settings.  In the former, the agent can repeatedly try different sequences of actions over the full time period $[0,T]$, record the corresponding state processes and payoffs, and then learn and update the policy based on the resulting dataset. In the latter, the agent updates the actions as she goes, based on all the up-to-date historical observations.}

A critical question is how to strategically {\it generate} these trial-and-error sequences of actions. The idea is {\it randomization}; namely,
the agent devises and  employs a {\it stochastic policy}, which is a probability distribution on the action space, to generate actions according
to the current time--state pair.  It is important to note that such randomization itself is {\it independent} of the underlying Brownian motion $W$, the random source of the original control problem that stands for the environmental noises.   
Specifically,  assume the probability space is rich enough to
support uniformly distributed random variables on $[0, 1]$ that is independent of $W$, and then such a uniform random variable can be used to generate other random variables with density functions. Let $\{Z_t,0\leq t\leq T\}$ be a process of mutually independent copies of a uniform random variable on $[0,1]$, the construction of which requires a  suitable extension of probability space; see \citet{sun2006exact} for details.
We then further expand the filtered probability space to $\left( \Omega ,\mathcal{F},\mathbb{P}; \{\mathcal{F}_s\}_{s\geq0}\right) $ where
$\mathcal{F}_s=\mathcal{F}_s^W \vee \sigma(Z_t,0\leq t \leq s)$ and the probability measure $\mathbb{P}$, now defined on $\mathcal{F}_T$,  is an extension from $\p^W$ (i.e. the two probability measures coincide when restricted to $\f^W_T$).

Let $\bm{\pi}:(t,x)\in [0,T] \times \mathbb{R}^d \mapsto \bm{\pi}(\cdot|t,x)\in \mathcal{P}(\mathcal{A})$ be a given (feedback) policy, where $\mathcal{P}(\mathcal{A})$ is a suitable collection of probability density functions.\footnote{Here we assume that the action space $\mathcal{A}$ is continuous and randomization is restricted to those distributions that have density functions. The analysis and results of this paper can be easily extended to the cases of discrete action spaces and/or randomization with probability mass functions.}
At each time $s$, an action $a_s$ is generated or sampled from the distribution $\bm{\pi}(\cdot|s,X_s)$. 

Fix a stochastic policy $\bm{\pi}$ and an initial time--state pair $(t,x)$.
Consider the following SDE
\begin{equation}
\label{eq:model pi}
\dd X_s^{\bm{\pi}} = b(s,X_s^{\bm{\pi}},a_s^{\bm{\pi}})\dd s + \sigma(s,X_s^{\bm{\pi}},a_s^{\bm{\pi}}) \dd W_s,\
s\in [t,T]; \;\;X_{t}^{\bm{\pi}} = x
\end{equation}
defined on $\left( \Omega ,\mathcal{F},\mathbb{P}; \{\mathcal{F}_s\}_{s\geq0}\right) $, where $a^{\bm{\pi}} = \{a_s^{\bm{\pi}},t\leq s \leq T\}$ is an $\{\mathcal{F}_s\}_{s\geq0}$-progressively measurable action process  generated from $\bm{\pi}$. \eqref{eq:model pi} is an SDE with random coefficients, whose well-posedness (i.e. existence and uniqueness of solution) is established in \citet[Chapter 1, Theorem 6.16]{YZbook} under Assumption \ref{ass:dynamic}. Fix $a^{\bm{\pi}}$, the unique solution to \eqref{eq:model pi}, $X^{\bm{\pi}} = \{X_s^{\bm{\pi}},t\leq s \leq T\}$, is the sample
state process corresponding to $a^{\bm{\pi}}$ that solves \eqref{eq:model classical}.\footnote{Here, $X^{\bm{\pi}}$ also depends on the {\it specific} copy $a^{\bm{\pi}}$ sampled from $\bm{\pi}$; however, to ease notation we write $X^{\bm{\pi}}$  instead of $X^{\bm{\pi},a^{\bm{\pi}}}$.}
Moreover, following \cite{wang2020reinforcement}, we add an entropy regularizer to the reward function to encourage  exploration (represented by the stochastic policy), leading to
\begin{equation}
\label{eq:objective relaxed action}
\begin{aligned}
J(t,x;\bm{\pi}) = &	\E^{\p}\bigg[\int_t^T e^{-\beta(s-t)}\left[ r(s,X_s^{\bm{\pi}},a_s^{\bm{\pi}}) - \gamma \log\bm{\pi}(a_s^{\bm{\pi}}|s,X_s^{\bm{\pi}}) \right]\dd s \\
& + e^{-\beta (T-t)}h(X_T^{\bm{\pi}})\Big|X_t^{\bm{\pi}} = x \bigg],
\end{aligned}
\end{equation}
where $\E^{\p}$ is the expectation with respect to both the Brownian motion and the action randomization, and $\gamma \geq 0$ is a given weighting parameter on exploration also known as the {\it temperature} parameter.

\citet{wang2020reinforcement} consider the following SDE
\begin{equation}
\label{eq:model relaxed}
\dd X_s = \tilde{b}\big( s,X_s,\bm{\pi}(\cdot|s, X_s) \big)\dd t + \tilde{\sigma}\big( s,X_s,\bm{\pi}(\cdot|s, X_s) \big) \dd W_s,\;s\in[t,T];\;\;\ X_{t} = x,
\end{equation}
where
\[ \tilde{b}\big(s,x,\pi(\cdot)\big) = \int_{\mathcal{A}} b(s,x,a) \pi(a)\dd a,\ \; \tilde{\sigma}\big(s,x,\pi(\cdot)\big) = \sqrt{\int_{\mathcal{A}} \sigma\sigma^\top(s,x,a) \pi(a)\dd a}.\]
Intuitively, based on the law of large number, the solution of \eqref{eq:model relaxed}, denoted by $\{\tilde{X}_s^{\bm{\pi}},t\leq s \leq T\}$, is the limit of the average of the sample trajectories $X^{\bm{\pi}}$ over randomization (i.e., copies of $\bm{\pi}$). Rigorously, it follows from the property of Markovian projection due to \citet[Theorem 3.6]{brunick2013mimicking} that $X^{\bm\pi}_s$ and $\tilde{X}^{\bm\pi}_s$ have the same distribution for each $s\in[t,T]$. Hence, the value function \eqref{eq:objective relaxed action} is identical to
\begin{equation}
\label{eq:objective relaxed}
\begin{aligned}
J(t,x;\bm{\pi}) = & \E^{\p^W}\bigg[\int_t^T e^{-\beta(s-t)}\tilde{\mathcal{R}}(s,\tilde{X}^{\bm\pi}_s,\bm\pi(\cdot|s,\tilde{X}^{\bm\pi}_s )) \dd s + e^{-\beta (T-t)}h(\tilde{X}_T^{\bm{\pi}})\Big|\tilde{X}_t^{\bm{\pi}} = x \bigg].
\end{aligned}
\end{equation}
where $\tilde{\mathcal{R}}(s,x,\pi) = \int_{\mathcal{A}} [r(s,x,a) - \gamma \log\pi(a) ] \pi(a)\dd a$.


The function  $J(\cdot,\cdot;\bm{\pi})$ is called the \textit{value function} of the  policy $\bm{\pi}$, and the task of RL is to find
\begin{equation}
\label{eq:maximize objective}
J^*(t,x) = \max_{\bm{\pi}\in \bm{\Pi}}J(t,x;\bm{\pi}),
\end{equation}
where $\bm{\Pi}$ stands for the set of admissible (stochastic) policies. 
The following gives the precise definition of admissible policies.

\begin{definition}
\label{ass:admissible}
A policy $\bm{\pi}=\bm{\pi}(\cdot|\cdot,\cdot)$ is called {\it admissible} if
\begin{enumerate}[(i)]
\item $\bm{\pi}(\cdot|t,x)\in \mathcal{P}(\mathcal{A})$, $\operatorname{supp}\bm\pi(\cdot|t,x) = \mathcal{A}$ for every $(t,x)\in [0,T]\times \mathbb{R}^d$, and $\bm{\pi}(a|t,x):(t,x,a)\in [0,T] \times \mathbb{R}^d\times \mathcal{A}\to
\mathbb{R}$ is measurable;

\item $\bm{\pi}(a|t,x)$ is continuous in $(t,x)$  and uniformly Lipschitz continuous in $x$ in the total variation distance, i.e., $ \int_{\mathcal{A}} |\bm{\pi}(a|t,x) - \bm{\pi}(a|u,x')|\dd a \to 0$ as $(u,x')\to (t,x)$, and
there is a constant $C>0$ independent of $(t,a)$ such that
\[ \int_{\mathcal{A}} |\bm{\pi}(a|t,x) - \bm{\pi}(a|t,x')|\dd a \leq C|x-x'|,\;\;\forall x,x'\in \mathbb{R}^d; \]
\item 
For any given $\alpha>0$,
the entropy of $\bm\pi$ and its $\alpha$-moment have polynomial growth in $x$, i.e.,
there are constants $C=C(\alpha)>0$ and $\mu'=\mu'(\alpha)\geq 1$ such that
$|\int_{\mathcal{A}} -\log\bm{\pi}(a|t,x)  \bm{\pi}(a|t,x)\dd a | \leq C(1+|x|^{\mu'})$, and $\int_{{\cal A}} |a|^{\alpha} \bm{\pi}(a|t,x)\dd a \leq C(1 + |x|^{\mu'})$ $\forall (t,x)$.
\end{enumerate}
\end{definition}




Under Assumption \ref{ass:dynamic} along with (i) and (ii) in Definition \ref{ass:admissible}, \citet[Lemma 2]{jia2021policypg} show that the SDE \eqref{eq:model relaxed} admits a unique strong solution since its coefficients are locally Lipschitz continuous and have linear growth; see \citet[Chapter 2, Theorem 3.4]{mao2007stochastic} for the general result. In addition, the growth condition (iii) guarantees that the new payoff function \eqref{eq:objective relaxed} is finite.
More discussions on the conditions required in the above definition can be found in \cite{jia2021policypg}. Moreover, Definition \ref{ass:admissible} only contains the conditions on the policy, hence they can be verified.

Note that the problem  \eqref{eq:model relaxed}--\eqref{eq:objective relaxed} is {\it mathematically} equivalent to the problem \eqref{eq:model pi}--\eqref{eq:objective relaxed action}. Nevertheless, they serve different purposes in our study: the former provides a framework for theoretical analysis of the value function, while the latter directly involves  observable samples for algorithm design.

\subsection{Some useful preliminary results}

Lemma 2 in \cite{jia2021policypg} yields that
the value function of a given admissible policy $\bm{\pi}=\bm{\pi}(\cdot|\cdot,\cdot)\in \bm{\Pi}$ satisfies the following PDE
\begin{equation}
\label{eq:explore fk}
\begin{aligned}
& \frac{\partial J}{\partial t}(t,x;\bm\pi)  + \int_{\mathcal{A}} \left[ H\big( t,x,a,\frac{\partial J}{\partial x}(t,x;\bm\pi), \frac{\partial^2 J}{\partial x^2}(t,x;\bm\pi) \big) - \gamma \log \bm{\pi}(a|t,x) \right]  \bm{\pi}(a|t,x)\dd a\\
&-\beta J(t,x;\bm\pi)= 0,\\
& J(T,x;\bm\pi) = h(x).
\end{aligned}
\end{equation}
This is a version of the celebrated Feynman--Kac formula in the current RL setting. Under Assumption \ref{ass:dynamic} as well as the  admissibility conditions in Definition \ref{ass:admissible}, the PDE \eqref{eq:explore fk} admits a unique viscosity solution among polynomially growing functions \citep[Theorem 1.1]{beck2021nonlinear}.

On the other hand,
\citet[Section 5]{tang2021exploratory} derive the following {\it exploratory HJB equation} for the problem  \eqref{eq:model relaxed}--\eqref{eq:objective relaxed} satisfied by the optimal value function: 
\begin{equation}
\label{eq:explore hjb}
\begin{aligned}
& \frac{\partial J^*}{\partial t}(t,x) + \sup_{\bm \pi \in \mathcal{P}(\mathcal{A})} \int_{\mathcal{A}} \left[ H\big( t,x,a,\frac{\partial J^*}{\partial x}(t,x), \frac{\partial^2 J^*}{\partial x^2}(t,x) \big) - \gamma \log \bm{\pi}(a) \right]  \bm{\pi}(a)\dd a -\beta J^*(t,x;\bm\pi)= 0,\\ & J^*(T,x) = h(x).
\end{aligned}
\end{equation}
Moreover, the optimal (stochastic) policy is a {\it Gibbs measure} or {\it Boltzmann distribution}:
\[\bm{\pi}^*(\cdot|t,x) \propto \exp\left\{ \frac{1}{\gamma} H\left( t,x,\cdot,\frac{\partial J^*}{\partial x}(t,x), \frac{\partial^2 J^*}{\partial x^2}(t,x) \right)\right\} \]
or, after normalization,
\begin{equation}
\label{eq:gibbs} \bm{\pi}^*(a|t,x) = \frac{\exp\{ \frac{1}{\gamma} H\big( t,x,a,\frac{\partial J^*}{\partial x}(t,x), \frac{\partial^2 J^*}{\partial x^2}(t,x) \big)\}}{ \int_{\mathcal{A}}  \exp\{ \frac{1}{\gamma} H\big( t,x,a,\frac{\partial J^*}{\partial x}(t,x), \frac{\partial^2 J^*}{\partial x^2}(t,x) \big)\}  \dd a  }.
\end{equation}
This result shows that one should use a Gibbs sampler in general to generate trial-and-error strategies to explore the environment when the regularizer is chosen to be entropy.\footnote{We stress here that the Gibbs sampler is due to the entropy-regularizer, and it is possible to have other types of distribution for exploration. For example, \cite{han2022choquet} propose a family of regularizers that
lead to exponential, uniform, and $\varepsilon$-greedy exploratory policies.
\cite{o2021variational,GXZ2020} study how to choose time and state dependent temperature parameters. One may also carry out exploration via randomizing value function instead of policies; see e.g. \citet{osband2019deep}.}

In the special case when the system dynamics are linear in action and payoffs are quadratic in action,  the Hamiltonian is a quadratic function of action and  Gibbs  thus specializes to Gaussian under mild conditions.

Plugging \eqref{eq:gibbs} to \eqref{eq:explore hjb} to replace the supremum operator therein leads to the following equivalent form of the exploratory HJB equation
\begin{equation}
\label{eq:explore hjb simplified}
\begin{aligned}
&\frac{\partial J^*}{\partial t}(t,x) + \gamma \log\left[ \int_{\mathcal{A}} \exp\left\{ \frac{1}{\gamma} H\left( t,x,a,\frac{\partial J^*}{\partial x}(t,x), \frac{\partial^2 J^*}{\partial x^2}(t,x) \right)\right\} \dd a\right] - \beta J^*(t,x) = 0,\\
& J^*(T,x) = h(x).
\end{aligned}
\end{equation}
More theoretical properties regarding \eqref{eq:explore hjb simplified} and consequently  the problem \eqref{eq:model pi}--\eqref{eq:objective relaxed action} including its well-posedness can be found in \citet{tang2021exploratory}.

The Gibbs measure can also be used to derive a \textit{policy improvement theorem}.
\citet{wang2020continuous} prove such a theorem in the context of continuous-time mean--variance
portfolio selection, which is essentially a linear--quadratic (LQ) control problem. The following extends that result to the general case.
\begin{theorem}
\label{lemma:policy improvement hjb} For any given $\bm{\pi}\in
\bm{\Pi}$, define $\bm\pi'(\cdot|t,x)\propto \exp\{ \frac{1}{\gamma}H\big( t,x,\cdot,\frac{\partial J}{\partial x}(t,x;\bm\pi), \frac{\partial^2 J}{\partial x^2}(t,x;\bm\pi) \big)    \} $. If $\bm\pi'\in \bm\Pi$, then
\[ J(t,x;\bm\pi') \geq J(t,x;\bm\pi) . \]
Moreover, if the following map
$$\mathcal I(\bm\pi)= \frac{\exp\{ \frac{1}{\gamma}H\big( t,x,\cdot,\frac{\partial J}{\partial x}(t,x;\bm\pi), \frac{\partial^2 J}{\partial x^2}(t,x;\bm\pi) \big)    \}}{\int_{{\cal A}}\exp\{ \frac{1}{\gamma}H\big( t,x,a,\frac{\partial J}{\partial x}(t,x;\bm\pi), \frac{\partial^2 J}{\partial x^2}(t,x;\bm\pi) \big)    \}\dd a},
\;\;\bm\pi\in \bm{\Pi}$$
has a fixed point $\bm\pi^*$, then $\bm\pi^*$ is the optimal policy.
\end{theorem}

At this point, Theorem \ref{lemma:policy improvement hjb} remains a {\it theoretical} result that cannot be directly applied to learning procedures, because the Hamiltonian  $H\big( t,x,a,\frac{\partial J}{\partial x}(t,x;\bm\pi), \frac{\partial^2 J}{\partial x^2}(t,x;\bm\pi) \big)$ depends on the knowledge of the model parameters which we do not have in the RL context. To develop implementable algorithms, we turn to Q-learning.

\section{q-Function in Continuous Time: The Theory}
\label{sec:foundation}

This section is the theoretical foundation of the paper, with the analysis entirely in continuous time. We start with defining a Q-function parameterized by a time step $\Delta t>0$, and then motivate the notion of q-function that is independent of $\Delta t$. We further provide martingale characterizations of the q-function.
\subsection{Q-function}
\label{subsec:conventional q function}
\citet{tallec2019making} consider a continuous-time MDP and then discretize it upfront to a discrete-time MDP with time discretization $\delta t$. In that setting, the authors argue that ``there is no Q-function in continuous time'' because  the Q-function collapses to the value function when $\delta t$ is infinitesimally small.
Here, we take an entirely different approach that does not involve time discretization. Incidentally, by comparing the form of the Gibbs measure (\ref{eq:gibbs}) with that of the widely employed Boltzmann exploration for learning in  discrete-time MDPs, \cite{GXZ2020} and \cite{zhou2021coo} {\it conjecture} that the continuous-time counterpart of the Q-function is the Hamiltonian.
Our approach will also provide a rigorous justification on and, indeed, a refinement of this conjecture.

Given $\bm{\pi}\in \bm{\Pi}$ and $(t,x,a)\in [0,T)\times \mathbb{R}^d\times \mathcal{A}$, consider a  ``perturbed" policy of $\bm{\pi}$, denoted by  $\tilde{\bm{\pi}}$, as follows: It takes the action $a\in \mathcal{A}$ on $[t,t+\Delta t)$ where $\Delta t>0$, and then follows $\bm{\pi}$ on $[t+\Delta t,T]$. 
The corresponding state process $X^{\tilde{\bm{\pi}}}$, given $X^{\tilde{\bm{\pi}}}_t=x$, can be broken into two pieces. On $[t,t+\Delta t)$, it is $X^a$ which is the solution to
\[\dd X_s^{{a}} = b(s,X_s^{{a}},a)\dd s + \sigma(s,X_s^{{a}},a) \dd W_s,\
s\in [t,t+\Delta t);  X_t^{{a}}=x,\]
while on $[t+\Delta t,T]$, it is $X^{\bm{\pi}}$ following  \eqref{eq:model pi} but with the initial time--state pair $(t+\Delta t, X_{t+\Delta t}^{a})$.

With $\Delta t>0$ fixed, we introduce the ($\Delta t$-parameterized)  Q-function, denoted by $Q_{\Delta t}(t,x,a;\bm{\pi} )$,  to be the expected reward drawn from the perturbed policy, $\tilde{\bm{\pi}}$:
\footnote{In the existing literature, Q-learning with  entropy regularization  is often referred to as the ``soft Q-learning'' with the associated ``soft Q-function''. A review of and more discussions on soft Q-learning in discrete time can be found in Appendix A as well as in  \citet{haarnoja2018soft, schulman2017equivalence}.
}
\[\begin{aligned}
& Q_{\Delta t}(t,x,a;\bm{\pi} )\\
= & \E^{\p}\bigg[ \int_t^{t+\Delta t} e^{-\beta(s-t)}r(s,X_s^{a},a)\dd s\\
& +  \int_{t+\Delta t}^T e^{-\beta(s-t)}[r(s,X_s^{\bm{\pi}},a_s^{\bm \pi})-\gamma\log \bm{\pi}(a_s^{\bm \pi}|s,X_s^{\bm{\pi}}) ]\dd s + e^{-\beta(T-t)}h(X_T^{\bm{\pi}})\Big|X_t^{\tilde{\bm{\pi}}} = x \bigg] .
\end{aligned} \]

Recall that  our formulation \eqref{eq:objective relaxed action} includes an entropy regularizer term that incentivizes exploration using stochastic policies. However, in defining $Q_{\Delta t}(t,x,a;\bm{\pi} )$ above we do not include such a term on  $[t,t+\Delta t)$ for $\tilde{\bm{\pi}}$ because a deterministic constant action $a$ is applied whose entropy is excluded.

The following proposition provides an expansion of this Q-function in $\Delta t$.
\begin{proposition}
\label{prop:Q expand}
We have
\begin{equation}\label{qdt}
\begin{aligned}
& Q_{\Delta t}(t,x,a;\bm{\pi} ) \\
= & J(t,x;\bm{\pi}) + \left[ \frac{\partial J}{\partial t}(t,x;\bm{\pi}) + H\big( t,x,a,\frac{\partial J}{\partial x}(t,x;\bm{\pi}), \frac{\partial^2 J}{\partial x^2}(t,x;\bm{\pi})\big)  -\beta J(t,x;\bm{\pi}) \right]  \Delta t + o( \Delta t ).
\end{aligned}
\end{equation}
\end{proposition}

So, this $\Delta t$-based Q-function includes three terms. The leading term is the value function $J$, which is {\it independent} of the action $a$ taken. This
is natural because we only apply $a$  for a time window of length $\Delta t$ and hence the impact of this action is only of the order $\Delta t$.
Next, the first-order term of $\Delta t$  is the Hamiltonian plus the temporal change of the value function (consisting of the derivative of the value function in time and the discount term, both of which would disappear for stationary and non-discounted problems), in which {\it only} the Hamiltonian depends on the action $a$. It is interesting to note that these are the same terms in the Feynman--Kac PDE \eqref{eq:explore fk}, less the entropy term. Finally, the higher-order residual  term,  $o(\Delta t)$, comes from the approximation  error of the integral over $[t,t+\Delta t]$ and can be ignored when such errors are aggregated as $\Delta t$ gets smaller.

Proposition \ref{prop:Q expand} yields that this Q-function also collapses to the value function when $\Delta t\rightarrow 0$, similar to what \cite{tallec2019making} claim.
However, we must emphasize that our Q-function is {\it not} the same as the one used in \cite{tallec2019making}, who discretize time {\it upfront} and study the resulting discrete-time MDP. The corresponding Q-function and value function therein exist based on the discrete-time RL theory, denoted respectively by $Q^{\delta t},V^{\delta t}$ where $\delta t$ is the time discretization size.  Then \cite{tallec2019making} argue that the advantage function $A^{\delta t} = Q^{\delta t}-V^{\delta t}$ can be properly scaled by $\delta t$, leading to the existence of $\lim_{\delta t\to 0}\frac{A^{\delta t}}{\delta t}$.
In short, the Q-function  and the resulting algorithms in \cite{tallec2019making} are still in the realm of discrete-time MDPs. By contrast, $Q_{\Delta t}(t,x,a;\bm\pi)$ introduced here is a {\it continuous-time} notion. It is the value function of a policy that applies  a constant action in  $[t,t+\Delta t]$ and follows $\bm\pi$ thereafter. So $\Delta t$ here is a {\it parameter} representing the length of period in which the constant  action is taken, {\it not}
the time discretization size as in \cite{tallec2019making}. Most importantly,
our Q-function is just a tool used to introduce the q-function, the centerpiece  of this paper. Having said all these, we recognize that $\lim_{\delta t\to 0}\frac{A^{\delta t}}{\delta t}$  as identified by \cite{tallec2019making} is the counterpart of our q-function in their setting, and that we arrive at this same object in two different ways (discretizing upfront then taking $\delta t\to 0$, versus a true continuous-time analysis).

\subsection{q-function}
\label{subsec:q rate}
Since the leading term in the Q-function, $Q_{\Delta t}(t,x,a;\bm{\pi} )$,  coincides with the value function of $\bm{\pi}$ and hence can not be used to rank  action $a$, we focus on the first-order approximation, which gives an infinitesimal  state--action value. Motivated by this, we define the ``\textit{q-function}'' as follows:
\begin{definition}
\label{def:q rate function}
The q-function of the problem \eqref{eq:model pi}--\eqref{eq:objective relaxed action} associated with a given policy $\bm\pi\in {\bm\Pi}$ is defined as
\begin{equation}
\label{eq:q rate}
\begin{aligned}
q(t,x,a;\bm{\pi}) = &\frac{\partial J}{\partial t}(t,x;\bm{\pi}) + H\left( t,x,a,\frac{\partial J}{\partial x}(t,x;\bm{\pi}), \frac{\partial^2 J}{\partial x^2}(t,x;\bm{\pi})\right)  -\beta J(t,x;\bm{\pi}),\\
& \;\;\;\;\;(t,x,a)\in [0,T]\times \mathbb{R}^d\times \mathcal{A}.
\end{aligned}
\end{equation}
\end{definition}

Clearly, this function is the first-order {\it derivative} of the Q-function with respect to $\Delta t$, as an immediate consequence of Proposition \ref{prop:Q expand}:
\begin{corollary}
\label{coro:q as derivative}
We have
\begin{equation}
\label{eq:q rate derivative}
q(t,x,a;\bm{\pi}) = \lim_{\Delta t\to 0}\frac{Q_{\Delta t}(t,x,a;\bm{\pi} ) - J(t,x;\bm{\pi})}{\Delta t}.
\end{equation}
\end{corollary}

Some remarks are in order.
First,
the q-function is a function of the time--state--action triple under a given policy, analogous  to the conventional Q-function for MDPs; yet it is a continuous-time notion because  {\it it does not depend on any time-discretization}. This feature is a vital advantage for learning algorithm design as \cite{tallec2019making} point out that the performance of RL algorithms is very sensitive with respect to the time discretization. Second, the q-function is related to the advantage function in the MDP literature (e.g., \citealt{baird1993advantage,mnih2016asynchronous}), which describes  the difference between the state--action value and the state value. Here in this paper, the q-function reflects the \textit{instantaneous advantage rate} of a given action at a given time--state pair under a given policy.

We also notice that in (\ref{eq:q rate}) only the Hamiltonian depends on $a$; hence the improved policy presented in Theorem \ref{lemma:policy improvement hjb} can also be expressed in terms of the q-function:
\[ \bm\pi'(\cdot|t,x)\propto \exp\left\{ \frac{1}{\gamma}H\big( t,x,\cdot,\frac{\partial J}{\partial x}(t,x;\bm\pi), \frac{\partial^2 J}{\partial x^2}(t,x;\bm\pi) \big)    \right\} \propto \exp\left\{  \frac{1}{\gamma}q(t,x,\cdot;\bm{\pi}) \right\} . \]
Consequently, if we can learn the q-function $q(\cdot,\cdot,\cdot;\bm\pi)$ under any policy $\bm\pi$, then it follows from Theorem \ref{lemma:policy improvement hjb} that we can improve $\bm\pi$ by implementing a Gibbs measure over the q-values, analogous to, say, $\varepsilon$-greedy policy in classical  Q-learning \citep[Chapter 6]{sutton2011reinforcement}. Finally, the analysis so far justifies the aforementioned conjecture about the proper continuous-time version of the Q-function, and provides a theoretical interpretation to the widely used Boltzmann exploration in RL.

The main theoretical results of this paper are martingale characterizations of
the q-function, which can in turn be employed to devise algorithms for learning crucial functions including the q-function, in the same way as in applying the martingale approach for PE \citep{jia2021policy}.\footnote{Some discrete-time counterparts of these results are presented in Appendix A.}

The following first result characterizes the q-function associated with a given policy
$\bm\pi$, assuming that its value function has already been learned and known.

\begin{theorem}
\label{thm:q rate function property}
Let a policy $\bm\pi\in \bm\Pi$, its value function $J$  and a continuous function $\hat{q}:[0,T]\times \mathbb{R}^d\times \mathcal{A}\to \mathbb{R}$
be given. Then 
\begin{enumerate}[(i)]
\item $\hat{q}(t,x,a)=q(t,x,a;\bm\pi)$ for all $(t,x,a)\in[0,T]\times\mathbb{R}^d\times \mathcal{A}$ if and only if for all $(t,x)\in[0,T]\times\mathbb{R}^d$, the following process
\begin{equation}
\label{eq:martingale with q function}
e^{-\beta s} J(s,{X}_s^{\bm\pi};\bm\pi) + \int_t^s e^{-\beta u} [r(u,{X}_{u}^{\bm\pi},a^{\bm\pi}_{u}) - \hat{q}(u,{X}_{u}^{\bm\pi},a^{\bm\pi}_{u}) ]\dd u
\end{equation}
is an $(\{\f_s\}_{s\geq 0},\p)$-martingale, where $\{{X}_s^{\bm\pi}, t\leq s\leq T\}$ is the solution to (\ref{eq:model pi}) under $\bm\pi$ with ${X}_t^{\bm\pi}=x$.
\item If $\hat{q}(t,x,a)=q(t,x,a;\bm\pi)$ for all $(t,x,a)\in[0,T]\times\mathbb{R}^d\times \mathcal{A}$, then for any $\bm\pi'\in \bm\Pi$ and for all $(t,x)\in[0,T]\times\mathbb{R}^d$, the following process
\begin{equation}
\label{eq:martingale with q function-off}
e^{-\beta s} J(s,{X}_s^{\bm\pi'};\bm\pi) + \int_t^s e^{-\beta u} [r(u,{X}_{u}^{\bm\pi'},a^{\bm\pi'}_{u}) - \hat{q}(u,{X}_{u}^{\bm\pi'},a^{\bm\pi'}_{u}) ]\dd u
\end{equation}
is an $(\{\f_s\}_{s\geq 0},\p)$-martingale, where $\{{X}_s^{\bm\pi'}, t\leq s\leq T\}$ is the solution to (\ref{eq:model pi}) under $\bm\pi'$ with initial condition ${X}_{t}^{\bm\pi'} = x$.
\item If there exists $\bm\pi'\in \bm\Pi$ such that for all $(t,x)\in[0,T]\times\mathbb{R}^d$, \eqref{eq:martingale with q function-off} is an $(\{\f_s\}_{s\geq 0}, \p)$-martingale where ${X}_{t}^{\bm\pi'} = x$, then  $\hat{q}(t,x,a)=q(t,x,a;\bm\pi)$ for all $(t,x,a)\in[0,T]\times\mathbb{R}^d\times \mathcal{A}$.
\end{enumerate}
Moreover, in any of the three cases above, the q-function satisfies
\begin{equation}
\label{eq:q hjb}
\int_{\mathcal{A}} \big[ q(t,x,a;\bm{\pi} ) -\gamma \log\bm\pi(a|t,x) \big] \bm\pi(a|t,x)\dd a =0, \;\;\forall (t,x)\in[0,T]\times\mathbb{R}^d.
\end{equation}


\end{theorem}

\cite{jia2021policy} unify and characterize the learning of value function (i.e., PE) by martingale conditions of certain stochastic processes.
Theorem \ref{thm:q rate function property} shows that learning the q-function
again boils down to maintaining the martingality of the processes \eqref{eq:martingale with q function} or \eqref{eq:martingale with q function-off}.
However, there are subtle differences between these martingale conditions.
\cite{jia2021policy} consider only deterministic policies so the martingales therein are with respect to $(\{\f_s^W\}_{s\geq 0},\p^W)$. \cite{jia2021policypg} extends the policy evaluation to include stochastic policies but the related martingales are in terms of the {\it averaged} state $\tilde {X}^{\bm\pi}$ and hence also with respect to $(\{\f_s^W\}_{s\geq 0},\p^W)$.
By contrast, the martingality in Theorem \ref{thm:q rate function property} 
is with respect to $(\{\f_t\}_{t\geq 0},\p)$, where the filtration $\{\f_t\}_{t\geq 0}$ is the {\it enlarged} one that includes the randomness for generating actions. 
Note that the filtration determines the class of test functions necessary to characterize a martingale through the so-called martingale orthogonality conditions \citep{jia2021policy}. So the above martingale conditions  suggest that one ought to  choose test functions dependent of the past and current actions when designing algorithms to learn the q-function.
Moreover, the q-function can be interpreted as a compensator to guarantee the martingality
over this larger information field.\footnote{More precisely, $\int_0^t e^{-\beta s} q(s,X_s^{\bm\pi},a^{\bm\pi}_s;\bm\pi) \dd s$ is the compensator of $e^{-\beta t} J(t,{X}_t^{\bm\pi};\bm\pi) + \int_0^t e^{-\beta s} r(s,{X}_s^{\bm\pi},a^{\bm\pi}_s) \dd s$.  Recall that a compensator of an adapted stochastic process $Y_t$ is a predictable process $A_t$ such that $Y_t-A_t$ is a local martingale. Intuitively speaking, the compensator is the drift part of a diffusion process, extracting the trend of the process.}

Theorem \ref{thm:q rate function property}-(i) informs {\it on-policy} learning, namely, learning the q-function of the  given policy $\bm\pi$, called the {\it target policy}, based on data $\{(s,{X}_{s}^{\bm\pi},a^{\bm\pi}_{s}), t\leq s\leq T\}$ generated
by $\bm\pi$ itself. Nevertheless, one of the key advantages of classical Q-learning for MDPs is that it also works for {\it off-policy} learning, namely, learning the q-function of a given target policy $\bm\pi$ based on data generated by a different admissible policy $\bm\pi'$, called a {\it behavior policy}. On-policy and off-policy reflect two different learning settings depending on the availability of data and/or the choice of a learning agent. When data generation can be controlled by the agent, conducting on-policy learning is possible although she could still elect
off-policy learning. By contrast, when data generation is not controlled by the agent and she has to rely on data under other policies, then it becomes off-policy learning. Theorem \ref{thm:q rate function property}-(ii) and -(iii) stipulate that
our q-learning in continuous time can also be off-policy.

Finally, (\ref{eq:q hjb}) is a consistency condition to uniquely identify the q-function. If we only consider deterministic policies of the form $a(\cdot,\cdot)$
in which case $\gamma=0$ and $\bm\pi(\cdot|t,x)$ degenerates into the Dirac measure concentrating on $a(t,x)$, then  (\ref{eq:q hjb}) reduces to
\[ q(t,x,a(t,x);a)=0,\]
which corresponds to equation (23) in \citet{tallec2019making}.

The following result strengthens Theorem \ref{thm:q rate function property}, in the sense that it characterizes the value function and the q-function under a given policy {\it simultaneously}.
\begin{theorem}
\label{thm:qv learn}
Let a policy $\bm\pi\in \bm\Pi$, a function $\hat{J}\in C^{1,2}\big([0,T)\times \mathbb{R}^d \big) \cap C\big([0,T]\times \mathbb{R}^d \big)$ with polynomial growth, and a continuous function $\hat{q}:[0,T]\times \mathbb{R}^d\times \mathcal{A}\to \mathbb{R}$
be given satisfying
\begin{equation}
\label{eq:q hjb2}
\hat{J}(T,x) = h(x),\;\;\; \int_{\mathcal{A}} \big[ \hat{q}(t,x,a) -\gamma \log{\bm\pi}(a|t,x) \big] {\bm\pi}(a|t,x)\dd a =0,\;\;\forall (t,x)\in[0,T]\times\mathbb{R}^d.
\end{equation}
Then 
\begin{enumerate}[(i)]
\item $\hat{J}$ and $\hat{q}$ are respectively the value function and the q-function associated with ${\bm\pi}$ if and only if for all $(t,x)\in[0,T]\times\mathbb{R}^d$, the following process
\begin{equation}
\label{eq:martingale with q function2}
e^{-\beta s} \hat J(s,{X}_s^{\bm\pi}) + \int_t^s e^{-\beta u} [r(u,{X}_{u}^{\bm\pi},a^{\bm\pi}_{u}) - \hat{q}(u,{X}_{u}^{\bm\pi},a^{\bm\pi}_{u}) ]\dd u
\end{equation}
is an $(\{\f_s\}_{s\geq 0},\p)$-martingale, where $\{{X}_s^{\bm\pi}, t\leq s\leq T\}$ is the solution to (\ref{eq:model pi}) with ${X}_t^{\bm\pi}=x$.
\item If $\hat{J}$ and $\hat{q}$ are respectively the value function and the q-function associated with ${\bm\pi}$, then for any $\bm\pi'\in \bm\Pi$ and all $(t,x)\in[0,T]\times\mathbb{R}^d$, the following process
\begin{equation}
\label{eq:martingale with q function2-off}
e^{-\beta s} \hat J(s,{X}_s^{\bm\pi'}) + \int_t^s e^{-\beta u} [r(u,{X}_{u}^{\bm\pi'},a^{\bm\pi'}_{u}) - \hat{q}(u,{X}_{u}^{\bm\pi'},a^{\bm\pi'}_{u}) ]\dd u
\end{equation}
is an $(\{\f_s\}_{s\geq 0}, \p)$-martingale, where $\{{X}_s^{\bm\pi'}, t\leq s\leq T\}$ is the solution to (\ref{eq:model pi}) under $\bm\pi'$ with initial condition ${X}_{t}^{\bm\pi'} = x$.
\item If there exists $\bm\pi'\in \bm\Pi$ such that for all $(t,x)\in[0,T]\times\mathbb{R}^d$, \eqref{eq:martingale with q function2-off} is an $(\{\f_s\}_{s\geq 0}, \p)$-martingale where ${X}_{t}^{\bm\pi'} = x$,  then  $\hat{J}$ and $\hat{q}$ are respectively the value function and the q-function associated with ${\bm\pi}$.
\end{enumerate}
Moreover, in any of the three cases above,
if it holds further that ${\bm\pi}(a|t,x) = \frac{\exp\{  \frac{1}{\gamma}\hat{q}(t,x,a) \}}{\int_{{\cal A}} \exp\{  \frac{1}{\gamma}\hat{q}(t,x,a) \} \dd a }$, then ${\bm\pi}$ is the optimal policy and $\hat{J}$ is the optimal value function.
\end{theorem}

Theorem \ref{thm:qv learn} characterizes the value function and the q-function in terms of a {\it single} martingale condition, in each of the on-policy and off-policy settings, which will be the foundation for designing learning algorithms in this paper. The conditions in  \eqref{eq:q hjb2} ensure  that  $\hat J$ corresponds to the correct terminal payoff function and $\hat{q}$  corresponds to the soft q-function with the entropy regularizer. In particular, if the policy is the Gibbs measure generated by   $\frac{1}{\gamma}\hat{q}$, then   $\hat J$ and $\hat{q}$ are respectively the value function and the q-function under the optimal policy. Finally,  note the (subtle) difference between Theorem \ref{thm:qv learn} and Theorem \ref{thm:q rate function property}. There may be multiple $(\hat J,\hat q)$ pairs satisfying the martingale conditions of \eqref{eq:martingale with q function2} or \eqref{eq:martingale with q function2-off} in Theorem \ref{thm:qv learn}; so the conditions in \eqref{eq:q hjb2} are {\it required} for identifying the correct value function and q-function.\footnote{If the terminal condition $\hat{J}(T,x) = h(x)$ is changed to $\hat{J}(T,x) = \hat h(x)$, then $(\hat{J},\hat{q})$ satisfying the martingale condition \eqref{eq:martingale with q function2} or \eqref{eq:martingale with q function2-off}  would be the value function and q-function corresponding to a different learning task with the pair of running and terminal reward functions being $(r,\hat{h})$. If the normalization condition on $\hat{q}$ is missing, say $\int_{\mathcal{A}} \big[ \hat{q}(t,x,a) -\gamma \log{\bm\pi}(a|t,x) \big] {\bm\pi}(a|t,x)\dd a  = \hat r(t,x)$ holds instead, then $(\hat{J},\hat{q})$ would be the value function and q-function corresponding to the learning task with the running and terminal reward functions $(r - \hat{r},h)$.} By contrast, these conditions are {\it implied} if the correct value function has already been known and given, as in Theorem \ref{thm:q rate function property}.

\subsection{Optimal q-function}

We now focus on  the q-function associated with the {\it optimal} policy $\bm\pi^*$ in \eqref{eq:gibbs}. Based on Definition \ref{def:q rate function}, it is defined as
\begin{equation}
\label{eq:optimal q}
q^*(t,x,a) = \frac{\partial J^*}{\partial t}(t,x) + H\left( t,x,a,\frac{\partial J^*}{\partial x}(t,x), \frac{\partial^2 J^*}{\partial x^2}(t,x)\right)  -\beta J^*(t,x),
\end{equation}
where $J^*$ is the optimal value function satisfying the exploratory HJB equation \eqref{eq:explore hjb simplified}.

\begin{proposition}\label{qstar1}
We have
\begin{equation}
\label{eq:optimal q hjb}
\int_{\mathcal A} \exp\{ \frac{1}{\gamma} q^*(t,x,a) \}\dd a = 1,
\end{equation}
for all $(t,x)$, and consequently
the optimal policy $\bm\pi^*$ is  
\begin{equation}
\label{eq:optimal pi and q}
\bm\pi^*(a|t,x) = \exp\{ \frac{1}{\gamma} q^*(t,x,a) \} .
\end{equation}
\end{proposition}

As seen from its proof (in Appendix C), Proposition \ref{qstar1} is an immediate  consequence of the exploratory HJB equation \eqref{eq:explore hjb simplified}. Indeed, \eqref{eq:optimal q hjb} is equivalent to \eqref{eq:explore hjb simplified} when viewed as an equation for $J^*$ due to \eqref{eq:optimal q}. However, in terms of  $q^*$, satisfying \eqref{eq:optimal q hjb} is only a {\it necessary} condition or a minimum requirement for being the optimal q-function, and by no means sufficient, in the current RL setting. This is because in the absence of knowledge about the primitives $b,\sigma,r,h$, we are unable to determine $J^*$ from $q^*$ by their relationship \eqref{eq:optimal q} which can be viewed as a PDE for $J^*$.
To fully characterize $J^*$ as well as $q^*$, we will have to resort to martingale condition, as stipulated in the following result.

\begin{theorem}
\label{thm:q optimal}
Let a function $\widehat{J^*}\in C^{1,2}\big([0,T)\times \mathbb{R}^d \big) \cap C\big([0,T]\times \mathbb{R}^d \big)$ with polynomial growth and a continuous function $\widehat{q^*}:[0,T]\times \mathbb{R}^d\times \mathcal{A}\to \mathbb{R}$
be given satisfying
\begin{equation}
\label{eq:q hjb2 optimal}
\widehat{J^*}(T,x) = h(x),\;\;\; \int_{\mathcal{A}} \exp\{ \frac{1}{\gamma} \widehat{q^*}(t,x,a) \} \dd a =1,\;\;\forall (t,x)\in[0,T]\times\mathbb{R}^d.
\end{equation}
Then 
\begin{enumerate}[(i)]
\item If $\widehat{J^*}$ and $\widehat{q^*}$ are respectively the optimal value function and the optimal q-function, then for any $\bm\pi\in \bm\Pi$ and all $(t,x)\in[0,T]\times\mathbb{R}^d$, the following process
\begin{equation}
\label{eq:martingale with q function2 optimal}
e^{-\beta s} \widehat{J^*}(s,{X}_s^{\bm\pi}) + \int_t^s e^{-\beta u} [r(u,{X}_{u}^{\bm\pi},a^{\bm\pi}_{u}) - \widehat{q^*}(u,{X}_{u}^{\bm\pi},a^{\bm\pi}_{u}) ]\dd u
\end{equation}
is an $(\{\f_s\}_{s\geq 0},\p)$-martingale, where $\{{X}_s^{\bm\pi}, t\leq s\leq T\}$ is the solution to (\ref{eq:model pi}) under $\bm\pi$ with ${X}_t^{\bm\pi}=x$.
Moreover, in this case, $\widehat{\bm\pi^*}(a|t,x) = \exp\{  \frac{1}{\gamma}\widehat{q^*}(t,x,a) \}$ is the optimal policy.
\item If there exists one $\bm\pi\in \bm\Pi$ such that for all $(t,x)\in[0,T]\times\mathbb{R}^d$,  \eqref{eq:martingale with q function2 optimal} is an $(\{\f_s\}_{s\geq 0},\p)$-martingale, then $\widehat{J^*}$ and $\widehat{q^*}$ are respectively the optimal value function and the optimal q-function. 
\end{enumerate}
\end{theorem}


Theorem \ref{thm:q optimal} lays a foundation for off-policy learning without having to go through iterative policy improvement.
We emphasize here that to learn the optimal policy and the optimal q-function, it is infeasible to conduct on-policy learning because $\bm\pi^*$ is unknown and hence the constraints \eqref{eq:q hjb2} can not be checked nor enforced. The new constraints in  \eqref{eq:q hjb2 optimal} no longer depend on policies and instead give basic requirements for the candidate value function and q-function. Maintaining the martingality of \eqref{eq:martingale with q function2 optimal} under any admissible policy is equivalent to the optimality of the candidate value function and q-function. Moreover, Theorem \ref{thm:q optimal}-(ii) suggests that we do not need to use data generated by {\it all} admissible policies. Instead, those generated by {\it one} -- any one -- policy  should suffice; e.g. we could take $\bm\pi(a|t,x) = \exp\{  \frac{1}{\gamma}\widehat{q^*}(t,x,a) \}$ as the behavior policy.

Finally, let us remark that Theorems \ref{thm:qv learn} and \ref{thm:q optimal} are independent of each other in the sense that one does not imply the other. One may design different algorithms from them depending on different cases and needs.

\section{q-Learning Algorithms When Normalizing Constant Is Available}
\label{subsec:qv learn}

This section and the next discuss algorithm design based on the theoretical results in the previous section. 
Theorems \ref{thm:qv learn} and \ref{thm:q optimal} provide the theoretical foundation for designing both on-policy and off-policy algorithms to {\it simultaneously} learn and update the value function (critic) and the  policy (actor) with proper function approximations of $J$ and $q$. In this section, we present these actor--critic, q-learning algorithms when the normalizing constant involved in the Gibbs measure is available or computable.\footnote{Conventional Q-learning or SARSA is often said to be value-based in which the Q-function is a state--action value function serving as a critic. In q-learning, as Theorem \ref{thm:q rate function property} stipulates, the q-function is uniquely  determined by the value function $J$  as its compensator. Hence, $q$ plays a dual role here: it is both a critic (as it can be derived endogenously from the value function) and an actor (as it derives an  improved policy). This is why we call the q-learning algorithms actor--critic, if with a slight abuse of the term because the ``actor'' here is not purely exogenous as with, say, policy gradient.}

\subsection{q-Learning  algorithms}
\label{subsec:AC_algorithms}

Given a policy $\bm\pi$, to learn its associated value function and q-function, we denote by $J^{\theta}$ and ${q}^{\psi}$  the parameterized function approximators that satisfy the two constraints:
\begin{equation}
\label{eq:constraints qv}
{J}^{\theta}(T,x) = h(x),\ \int_{\mathcal{A}} \big[ q^{\psi}(t,x,a) -\gamma \log\bm\pi(a|t,x) \big] \bm\pi(a|t,x)\dd a =0,
\end{equation}
for all $\theta\in \Theta\subset\mathbb{R}^{L_{\theta}},\psi\in\Psi\subset\mathbb{R}^{L_{\psi}}$. Then Theorem \ref{lemma:policy improvement hjb} suggests that the policy can be improved by
\begin{equation}
\label{eq:pipsi} \bm\pi^{\psi}(a|t,x) = \frac{\exp\{  \frac{1}{\gamma}q^{\psi}(t,x,a) \}}{\int_{{\cal A}} \exp\{  \frac{1}{\gamma}q^{\psi}(t,x,a) \} \dd a }, \end{equation}
while by assumption the normalizing constant $\int_{{\cal A}} \exp\{  \frac{1}{\gamma}q^{\psi}(t,x,a) \} \dd a $ for any $\psi\in\Psi$ can be explicitly computed. Next, we can  use the data generated by $\bm\pi^{\psi}$ and apply the martingale condition in Theorem \ref{thm:qv learn} to learn its value function and q-function, leading to an iterative actor--critic algorithm.

Alternatively, we may choose to directly learn the optimal value function and q-function based on Theorem \ref{thm:q optimal}. In this case, the approximators $J^{\theta}$ and ${q}^{\psi}$ should satisfy
\begin{equation}
\label{eq:constraints qv optimal}
{J}^{\theta}(T,x) = h(x),\ \int_{\mathcal{A}} \exp\{  \frac{1}{\gamma}q^{\psi}(t,x,a) \}  \dd a = 1.
\end{equation}

Note that if the policy $\bm\pi$ in \eqref{eq:constraints qv} is taken in the form \eqref{eq:pipsi}, then the second equation in  \eqref{eq:constraints qv optimal} is automatically satisfied. Henceforth in this section, we focus  on deriving algorithms based on Theorem \ref{thm:q optimal} and hence always impose the constraints \eqref{eq:constraints qv optimal}. In this case, any approximator of the q-function directly gives that of the policy via $\bm\pi^{\psi}(a|t,x) = \exp\{  \frac{1}{\gamma}q^{\psi}(t,x,a) \}$;  thereby we avoid learning the q-function and the policy separately.
Moreover, making use of  \eqref{eq:constraints qv optimal} typically results in more special parametric form of the q-function approximator $q^{\psi}$, potentially facilitating more efficient learning. Here is an example.

\begin{eg}
\label{eg:linear quadratic}
{\rm
When the system dynamic is linear in action $a$ and reward is quadratic in $a$, the Hamiltonian, and hence the q-function, is quadratic in $a$. So we can parameterize
\[ q^{\psi}(t,x,a) = -\frac{1}{2}{q}_2^{\psi}(t,x) \circ \left( a -{q}_1^{\psi}(t,x) \right)\left( a -{q}_1^{\psi}(t,x) \right)^\top + {q}_0^{\psi}(t,x),\;\;(t,x,a)\in [0,T]\times \mathbb{R}^d \times \mathbb{R}^m,\]
with ${q}_1^{\psi}(t,x) \in \mathbb{R}^m$, ${q}_2^{\psi}(t,x)\in \mathbb{S}^m_{++}$ and ${q}_0^{\psi}(t,x)\in \mathbb{R}$. The corresponding policy is a multivariate normal distribution for which the normalizing constant can be computed:
\[ \bm\pi^\psi(\cdot|t,x) = \mathcal{N}\left( {q}_1^{\psi}(t,x), \gamma \left( {q}_2^{\psi}(t,x) \right)^{-1}  \right),  \]
with its entropy value being $-\frac{1}{2}\log\det{q}_2^{\psi}(t,x) + \frac{m}{2}\log2\pi e\gamma$. The second constraint  on $q^{\psi}$ in \eqref{eq:constraints qv optimal} then yields $q^{\psi}_0(t,x) = \frac{\gamma}{2}\log\left(\det{q}_2^{\psi}(t,x)\right) - \frac{m\gamma}{2}\log2\pi $. This in turn gives rise to a more specific parametric form
\[ q^{\psi}(t,x,a) = -\frac{1}{2}{q}_2^{\psi}(t,x) \circ \left( a -{q}_1^{\psi}(t,x) \right)\left( a -{q}_1^{\psi}(t,x) \right)^\top + \frac{\gamma}{2}\log\left(\det{q}_2^{\psi}(t,x)\right) - \frac{m\gamma}{2}\log2\pi . \]}
\end{eg}

The next step in algorithm design is to update $(\theta,\psi)$ by enforcing the martingale condition stipulated in Theorem \ref{thm:q optimal} and applying the techniques developed  in \cite{jia2021policy}. A number of algorithms can be developed based on  two types of objectives: to minimize the martingale loss function or to satisfy the martingale orthogonality  conditions. 
The latter calls for solving a system of equations for which there are further two different techniques:  applying stochastic approximation as in the classical temporal--difference (TD) algorithms or minimizing a quadratic function in lieu of the system of equations as in the generalized methods of moment (GMM). For reader's
convenience, we summarize these methods in the q-learning context below.
\begin{itemize}
\item Minimize the martingale loss function:
{\small \[\frac{1}{2} \E^{\p}\left[ \int_0^T  \left[e^{-\beta (T-t)} h(X_T^{\bm\pi^{\psi}}) -  J^{\theta}(t,X_t^{\bm\pi^{\psi}})  + \int_t^T e^{-\beta (s-t)} [r(s,X_s^{\bm\pi^{\psi}},a_s^{\bm\pi^{\psi}}) -q^{\psi}(s,X_s^{\bm\pi^{\psi}},a_s^{\bm\pi^{\psi}})] \dd s  \right]^2 \dd t \right].  \]}
This method is intrinsically  offline because the loss function involves the whole horizon $[0,T]$; however we are free to choose optimization algorithms to update the parameters $(\theta,\psi)$. For example, we can apply  stochastic gradient decent to update
\[\begin{aligned}
& \theta \leftarrow \theta + \alpha_{\theta} \int_0^T  \frac{\partial J^{\theta}}{\partial \theta}(t,X_t^{\bm\pi^{\psi}}) G_{t:T} \dd t \\
& \psi \leftarrow \psi + \alpha_{\psi}\int_0^T \int_t^T e^{-\beta(s-t)} \frac{\partial q^{\psi}}{\partial \psi}(s,X_s^{\bm\pi^{\psi}},a_s^{\bm\pi^{\psi}}) \dd s  G_{t:T} \dd t ,
\end{aligned}    \]
where $G_{t:T} = e^{-\beta (T-t)} h(X_T^{\bm\pi^{\psi}}) -  J^{\theta}(t,X_t^{\bm\pi^{\psi}}) + \int_t^T e^{-\beta (s-t)} [r(s,X_s^{\bm\pi^{\psi}},a_s^{\bm\pi^{\psi}}) -q^{\psi}(s,X_s^{\bm\pi^{\psi}},a_s^{\bm\pi^{\psi}})] \dd s$, and $\alpha_{\theta}$ and $\alpha_{\psi}$ are the learning rates. We present Algorithm \ref{algo:offline episodic ml} based on this updating rule. Note that this algorithm is analogous  to the classical gradient Monte Carlo method or TD(1) for MDPs \citep{sutton2011reinforcement} because  full sample trajectories are used to compute gradients.
\item Choose two different test functions $\xi_t$ and $\zeta_t$ that are both $\f_t$-adapted vector-valued stochastic processes, and consider the following system of equations in $(\theta,\psi)$:
\[ \E^{\p}\left[ \int_0^T  \xi_t\left[\dd J^{\theta}(t,X_t^{\bm\pi^{\psi}}) + r(t,X_t^{\bm\pi^{\psi}},a_t^{\bm\pi^{\psi}})\dd t -q^{\psi}(t,X_t^{\bm\pi^{\psi}},a_t^{\bm\pi^{\psi}})\dd t - \beta J^{\theta}(t,X_t^{\bm\pi^{\psi}})\dd t\right]   \right] = 0, \]
and
\[ \E^{\p}\left[ \int_0^T  \zeta_t\left[\dd J^{\theta}(t,X_t^{\bm\pi^{\psi}}) + r(t,X_t^{\bm\pi^{\psi}},a_t^{\bm\pi^{\psi}})\dd t -q^{\psi}(t,X_t^{\bm\pi^{\psi}},a_t^{\bm\pi^{\psi}})\dd t - \beta J^{\theta}(t,X_t^{\bm\pi^{\psi}})\dd t\right]   \right] = 0. \]

To solve these equations iteratively, we use stochastic approximation to update $(\theta,\psi)$ either offline by\footnote{Here, when implementing in a computer program, $\dd J^{\theta}$ is the timestep-wise difference in $J^{\theta}$ when time is discretized, or indeed it is the temporal-difference (TD) of the value function  approximator $J^{\theta}$.}
\[\begin{aligned}
&\theta \leftarrow \theta + \alpha_{\theta} \int_0^T  \xi_t\left[\dd J^{\theta}(t,X_t^{\bm\pi^{\psi}}) + r(t,X_t^{\bm\pi^{\psi}},a_t^{\bm\pi^{\psi}})\dd t -q^{\psi}(t,X_t^{\bm\pi^{\psi}},a_t^{\bm\pi^{\psi}})\dd t - \beta J^{\theta}(t,X_t^{\bm\pi^{\psi}})\dd t\right],   \\
&\psi \leftarrow \psi + \alpha_{\psi} \int_0^T  \zeta_t\left[\dd J^{\theta}(t,X_t^{\bm\pi^{\psi}}) + r(t,X_t^{\bm\pi^{\psi}},a_t^{\bm\pi^{\psi}})\dd t -q^{\psi}(t,X_t^{\bm\pi^{\psi}},a_t^{\bm\pi^{\psi}})\dd t - \beta J^{\theta}(t,X_t^{\bm\pi^{\psi}})\dd t\right] ,
\end{aligned}\]
or online by
\[\begin{aligned}
&\theta \leftarrow \theta + \alpha_{\theta}  \xi_t\left[\dd J^{\theta}(t,X_t^{\bm\pi^{\psi}}) + r(t,X_t^{\bm\pi^{\psi}},a_t^{\bm\pi^{\psi}})\dd t -q^{\psi}(t,X_t^{\bm\pi^{\psi}},a_t^{\bm\pi^{\psi}})\dd t - \beta J^{\theta}(t,X_t^{\bm\pi^{\psi}}) \dd t\right],   \\
&\psi \leftarrow \psi + \alpha_{\psi} \zeta_t\left[\dd J^{\theta}(t,X_t^{\bm\pi^{\psi}}) + r(t,X_t^{\bm\pi^{\psi}},a_t^{\bm\pi^{\psi}})\dd t-q^{\psi}(t,X_t^{\bm\pi^{\psi}},a_t^{\bm\pi^{\psi}})\dd t - \beta J^{\theta}(t,X_t^{\bm\pi^{\psi}}) \dd t\right] .
\end{aligned}\]
Typical choices of test functions are $\xi_t = \frac{\partial J^{\theta}}{\partial \theta}(t,X_t^{\bm\pi^{\psi}})$, $\zeta_t = \frac{\partial q^{\psi}}{\partial \psi}(t,X_t^{\bm\pi^{\psi}},a_t^{\bm\pi^{\psi}})$ yielding algorithms  that would be closest to TD-style policy evaluation and Q-learning (e.g., in \citealt{tallec2019making}) and at the same time belong to the more general semi-gradient methods \citep{sutton2011reinforcement}. We present the online and offline q-learning algorithms, Algorithms \ref{algo:offline episodic} and \ref{algo:online incremental} respectively, based on these test functions. We must, however,  stress that testing against these two specific functions is theoretically {\it not} sufficient  to guarantee the martingale condition. Using them with a rich, large-dimensional parametric
family for $J$ and $q$ makes approximation to the martingale condition finer and finer. Moreover, the corresponding stochastic approximation algorithms are not guaranteed to converge in general, and the test functions have to be carefully selected \citep{jia2021policy} depending on $(J,q)$. Some of the different test functions  leading to different types of algorithms in the context of policy evaluation are discussed in \cite{jia2021policy}.

\item Choose the same types of  test functions $\xi_t$ and $\zeta_t$ as above but now minimize the GMM objective functions:
\[ \begin{aligned}
& \E^{\p}\left[ \int_0^T  \xi_t\left[\dd J^{\theta}(t,X_t^{\bm\pi^{\psi}}) + r(t,X_t^{\bm\pi^{\psi}},a_t^{\bm\pi^{\psi}})\dd t -q^{\psi}(t,X_t^{\bm\pi^{\psi}},a_t^{\bm\pi^{\psi}})\dd t - \beta J^{\theta}(t,X_t^{\bm\pi^{\psi}})\dd t\right]   \right]^\top \\
& A_{\theta} \E^{\p}\left[ \int_0^T  \xi_t\left[\dd J^{\theta}(t,X_t^{\bm\pi^{\psi}}) + r(t,X_t^{\bm\pi^{\psi}},a_t^{\bm\pi^{\psi}})\dd t -q^{\psi}(t,X_t^{\bm\pi^{\psi}},a_t^{\bm\pi^{\psi}})\dd t - \beta J^{\theta}(t,X_t^{\bm\pi^{\psi}})\dd t\right]   \right],
\end{aligned} \]
and
\[\begin{aligned}
& \E^{\p}\left[ \int_0^T  \zeta_t\left[\dd J^{\theta}(t,X_t^{\bm\pi^{\psi}}) + r(t,X_t^{\bm\pi^{\psi}},a_t^{\bm\pi^{\psi}})\dd t -q^{\psi}(t,X_t^{\bm\pi^{\psi}},a_t^{\bm\pi^{\psi}})\dd t - \beta J^{\theta}(t,X_t^{\bm\pi^{\psi}})\dd t\right]   \right]^\top \\
& A_{\psi}\E^{\p}\left[ \int_0^T  \zeta_t\left[\dd J^{\theta}(t,X_t^{\bm\pi^{\psi}}) + r(t,X_t^{\bm\pi^{\psi}},a_t^{\bm\pi^{\psi}})\dd t -q^{\psi}(t,X_t^{\bm\pi^{\psi}},a_t^{\bm\pi^{\psi}})\dd t - \beta J^{\theta}(t,X_t^{\bm\pi^{\psi}})\dd t\right]   \right],
\end{aligned}  \]
where $A_{\theta}\in \mathbb{S}^{L_{\theta}}_{++},A_{\psi}\in \mathbb{S}^{L_{\psi}}_{++}$. Typical choices of these matrices are $A_{\theta}
=I_{L_{\theta}}$ and $A_{\psi}=I_{L_{\psi}}$, or
$A_{\theta} = (\E^{\p}[\int_0^T \xi_t\xi_t^\top\dd t])^{-1}$ and $A_{\psi} = (\E^{\p}[\int_0^T \zeta_t\zeta_t^\top\dd t])^{-1}$. Again, we refer the reader to
\cite{jia2021policy} for discussions on these choices and the connection with the classical GTD algorithms and GMM method.


\end{itemize}

%
%

\begin{algorithm}[hbtp]
\caption{Offline--Episodic q-Learning ML Algorithm}
\textbf{Inputs}: initial state $x_0$,  horizon $T$, time step $\Delta t$, number of episodes $N$, number of mesh grids $K$, initial learning rates $\alpha_{\theta},\alpha_{\psi}$ and a learning rate schedule function $l(\cdot)$ (a function of the number of episodes), functional forms of parameterized  value function $J^{\theta}(\cdot,\cdot)$ and  q-function $q^{\psi}(\cdot,\cdot,\cdot)$ satisfying \eqref{eq:constraints qv optimal}, and temperature parameter $\gamma$.

\textbf{Required program (on-policy)}: environment simulator $(x',r) = \textit{Environment}_{\Delta t}(t,x,a)$ that takes current time--state pair $(t,x)$ and action $a$ as inputs and generates state $x'$ at time $t+\Delta t$ and  instantaneous reward $r$ at time $t$ as outputs. Policy $\bm\pi^{\psi}(a|t,x) = \exp\{  \frac{1}{\gamma}q^{\psi}(t,x,a) \}$.

\textbf{Required program (off-policy)}: observations $ \{a_{t_k}, r_{t_{k}}, x_{t_{k+1}}\}_{k = 0,\cdots, K-1}\cup \{ x_{t_K}, h(x_{t_K})\} = \textit{Observation}(\Delta t)$ including the observed actions, rewards, and state trajectories under the given behavior policy  at the sampling time grids with step size $\Delta t$.

\textbf{Learning procedure}:
\begin{algorithmic}
\STATE Initialize $\theta,\psi$.
\FOR{episode $j=1$ \TO $N$} \STATE{Initialize $k = 0$. Observe  initial state $x_0$ and store $x_{t_k} \leftarrow  x_0$.

\COMMENT{\textbf{On-policy case}

\WHILE{$k < K$} \STATE{
Generate action $a_{t_k}\sim \bm{\pi}^{\psi}(\cdot|t_k,x_{t_k})$.


Apply $a_{t_k}$ to  environment simulator $(x,r) = Environment_{\Delta t}(t_k, x_{t_k}, a_{t_k})$, and observe  new state $x$ and reward $r$ as output. Store $x_{t_{k+1}} \leftarrow x$ and $r_{t_k} \leftarrow r$.

Update $k \leftarrow k + 1$.
}
\ENDWHILE	

}

\COMMENT{\textbf{Off-policy case}
Obtain one observation $\{a_{t_k}, r_{t_{k}}, x_{t_{k+1}}\}_{k = 0,\cdots, K-1}\cup \{ x_{t_K}, h(x_{t_K})\} = \textit{Observation}(\Delta t)$.

}

For every $k=0,1,\cdots,K-1$, compute
\[ G_{t_k:T} = e^{-\beta (T - t_k)}h(x_{t_K}) - J^{\theta}(t_k,x_{t_k}) + \sum_{i=k}^{K-1}e^{-\beta (t_i - t_k)}[r_{t_i} - q^{\psi}(t_i,x_{t_i},a_{t_i})]\Delta t.
\]

Update $\theta$ and $\psi$ by
\[ \theta \leftarrow \theta + l(j)\alpha_{\theta} \sum_{k=0}^{K-1} \frac{\partial J^{\theta}}{\partial \theta}(t_k,x_{t_k}) G_{t_k:T} \Delta t .\]
\[ \psi \leftarrow \psi + l(j)\alpha_{\psi} \sum_{k=0}^{K-1}\left[ \sum_{i=k}^{K-1}\frac{\partial q^{\psi}}{\partial \psi}(t_k,x_{t_k},a_{t_k})\Delta t \right] G_{t_k:T} \Delta t .  \]

}
\ENDFOR
\end{algorithmic}
\label{algo:offline episodic ml}
\end{algorithm}

\begin{algorithm}[hbtp]
\caption{Offline--Episodic q-Learning Algorithm}
\textbf{Inputs}: initial state $x_0$,  horizon $T$, time step $\Delta t$, number of episodes $N$, number of mesh grids $K$, initial learning rates $\alpha_{\theta},\alpha_{\psi}$ and a learning rate schedule function $l(\cdot)$ (a function of the number of episodes), functional forms  of parameterized  value function $J^{\theta}(\cdot,\cdot)$ and  q-function $q^{\psi}(\cdot,\cdot,\cdot)$ satisfying \eqref{eq:constraints qv optimal}, functional forms of test functions $\bm{\xi}(t,x_{\cdot \wedge t},a_{\cdot \wedge t})$ and $\bm{\zeta}(t,x_{\cdot \wedge t},a_{\cdot \wedge t})$, and temperature parameter $\gamma$.

\textbf{Required program (on-policy)}: environment simulator $(x',r) = \textit{Environment}_{\Delta t}(t,x,a)$ that takes current time--state pair $(t,x)$ and action $a$ as inputs and generates state $x'$ at time $t+\Delta t$ and  instantaneous reward $r$ at time $t$ as outputs. Policy $\bm\pi^{\psi}(a|t,x) = \exp\{  \frac{1}{\gamma}q^{\psi}(t,x,a) \}$.

\textbf{Required program (off-policy)}: observations $ \{a_{t_k}, r_{t_{k}}, x_{t_{k+1}}\}_{k = 0,\cdots, K-1}\cup \{ x_{t_K}, h(x_{t_K})\} = \textit{Observation}(\Delta t)$ including the observed actions, rewards, and state trajectories under the given behavior policy  at the sampling time grids with step size $\Delta t$.

\textbf{Learning procedure}:
\begin{algorithmic}
\STATE Initialize $\theta,\psi$.
\FOR{episode $j=1$ \TO $N$} \STATE{Initialize $k = 0$. Observe  initial state $x_0$ and store $x_{t_k} \leftarrow  x_0$.

\COMMENT{\textbf{On-policy case}

\WHILE{$k < K$} \STATE{
Generate action $a_{t_k}\sim \bm{\pi}^{\psi}(\cdot|t_k,x_{t_k})$.

Apply $a_{t_k}$ to environment simulator $(x,r) = Environment_{\Delta t}(t_k, x_{t_k}, a_{t_k})$, and observe new state $x$ and reward $r$ as outputs. Store $x_{t_{k+1}} \leftarrow x$ and $r_{t_k} \leftarrow r$.

Update $k \leftarrow k + 1$.
}
\ENDWHILE	

}

\COMMENT{\textbf{Off-policy case}
Obtain one observation $\{a_{t_k}, r_{t_{k}}, x_{t_{k+1}}\}_{k = 0,\cdots, K-1}\cup \{ x_{t_K}, h(x_{t_K})\} = \textit{Observation}(\Delta t)$.

}

For every $k = 0,1,\cdots,K-1$, compute and store test functions $\xi_{t_k} = \bm{\xi}(t_k, x_{t_0},\cdots, x_{t_k},a_{t_0},\cdots, a_{t_k})$, $\zeta_{t_k} = \bm{\zeta}(t_k, x_{t_0},\cdots, x_{t_k},a_{t_0},\cdots, a_{t_k})$.	

Compute
\[ \Delta \theta = \sum_{i=0}^{K-1} \xi_{t_i} \big[ J^{\theta}(t_{i+1},x_{t_{i+1}}) - J^{\theta}(t_{i},x_{t_{i}}) + r_{t_i}\Delta t -q^{\psi}(t_{i},x_{t_{i}},a_{t_i})\Delta t - \beta J^{\theta}(t_{i},x_{t_{i}}) \Delta t  \big], \]
\[
\Delta \psi =   \sum_{i=0}^{K-1}\zeta_{t_i}\big[ J^{\theta}(t_{i+1},x_{t_{i+1}}) - J^{\theta}(t_{i},x_{t_{i}}) + r_{t_i}\Delta t - q^{\psi}(t_{i},x_{t_{i}},a_{t_i})\Delta t - \beta J^{\theta}(t_{i},x_{t_{i}}) \Delta t\big] .
\]

Update $\theta$ and $\psi$ by
\[ \theta \leftarrow \theta + l(j)\alpha_{\theta} \Delta \theta .\]
\[ \psi \leftarrow \psi + l(j)\alpha_{\psi} \Delta \psi .  \]

}
\ENDFOR
\end{algorithmic}
\label{algo:offline episodic}
\end{algorithm}

\begin{algorithm}[hbtp]
\caption{Online-Incremental q-Learning Algorithm}
\textbf{Inputs}: initial state $x_0$, horizon $T$, time step $\Delta t$, number of mesh grids $K$, initial learning rates $\alpha_{\theta},\alpha_{\psi}$ and learning rate schedule function $l(\cdot)$ (a function of the number of episodes), functional forms  of parameterized  value function $J^{\theta}(\cdot,\cdot)$ and  q-function $q^{\psi}(\cdot,\cdot,\cdot)$ satisfying \eqref{eq:constraints qv optimal}, functional forms of test functions $\bm{\xi}(t,x_{\cdot \wedge t},a_{\cdot \wedge t})$ and $\bm{\zeta}(t,x_{\cdot \wedge t},a_{\cdot \wedge t})$, and temperature parameter $\gamma$.

\textbf{Required program (on-policy)}: environment simulator $(x',r) = \textit{Environment}_{\Delta t}(t,x,a)$ that takes current time--state pair $(t,x)$ and action $a$ as inputs and generates state $x'$ at time $t+\Delta t$ and  instantaneous reward $r$ at time $t$ as outputs. Policy $\bm\pi^{\psi}(a|t,x) = \exp\{  \frac{1}{\gamma}q^{\psi}(t,x,a) \}$.

\textbf{Required program (off-policy)}: observations $ \{a, r, x'\} = \textit{Observation}(t, x;\Delta t)$ including the observed actions, rewards, and state when the current time-state pair is $(t,x)$ under the given behavior policy at the sampling time grids with step size $\Delta t$.

\textbf{Learning procedure}:
\begin{algorithmic}
\STATE Initialize $\theta,\psi$.
\FOR{episode $j=1$ \TO $\infty$} \STATE{Initialize $k = 0$. Observe  initial state $x_0$ and store $x_{t_k} \leftarrow  x_0$.
\WHILE{$k < K$} \STATE{
\COMMENT{\textbf{On-policy case}

Generate action $a_{t_k}\sim \bm{\pi}^{\psi}(\cdot|t_k,x_{t_k})$.

Apply $a_{t_k}$ to  environment simulator $(x,r) = Environment_{\Delta t}(t_k, x_{t_k}, a_{t_k})$, and observe new state $x$ and reward $r$ as outputs. Store $x_{t_{k+1}} \leftarrow x$ and $r_{t_k} \leftarrow r$.

}

\COMMENT{\textbf{Off-policy case}

Obtain one observation $a_{t_k}, r_{t_k}, x_{t_{k+1}} = \textit{Observation}(t_k, x_{t_k};\Delta t)$.

}

Compute test functions $\xi_{t_k} = \bm{\xi}(t_k, x_{t_0},\cdots, x_{t_k},a_{t_0},\cdots, a_{t_k})$, $\zeta_{t_k} = \bm{\zeta}(t_k, x_{t_0},\cdots, x_{t_k},a_{t_0},\cdots, a_{t_k})$.

Compute
\[ \begin{aligned}
& \delta = J^{\theta}(t_{k+1},x_{t_{k+1}}) - J^{\theta}(t_{k},x_{t_{k}}) + r_{t_k}\Delta t -q^{\psi}(t_{k},x_{t_{k}},a_{t_k})\Delta t - \beta J^{\theta}(t_{k},x_{t_{k}}) \Delta t ,\\
& \Delta \theta = \xi_{t_k} \delta, \\
& \Delta \psi = \zeta_{t_k} \delta.
\end{aligned} \]

Update $\theta$ and $\psi$ by
\[ \theta \leftarrow \theta + l(j)\alpha_{\theta} \Delta \theta .\]
\[ \psi \leftarrow \psi + l(j)\alpha_{\psi} \Delta \psi .  \]

Update $k \leftarrow k + 1$
}
\ENDWHILE	

}
\ENDFOR
\end{algorithmic}
\label{algo:online incremental}
\end{algorithm}

\subsection{Connections with SARSA}
\label{subsec:connect with q}


With a fixed time-step size $\Delta t$, we can define our 
Q-function, $Q_{\Delta t}(t,x,a;\bm\pi)$, by (\ref{qdt}),
parameterize it by $Q_{\Delta t}^{\varphi}(t,x,a;\bm\pi)$, and then apply existing
(big) Q-learning algorithms such as SARSA (cf., \citealt[Chapter 10]{sutton2011reinforcement}) to learn the parameter $\varphi$. Now, how do we compare this approach of $\Delta t$-based Q-learning  with that  of q-learning?

Equation \eqref{qdt} suggests that  $Q_{\Delta t}(t,x,a;\bm\pi)$  can be approximated by
\[
Q_{\Delta t}(t,x,a;\bm\pi) \approx J(t,x;\bm\pi) + q(t,x,a;\bm\pi) \Delta t.
\]
Our q-learning method is to learn separately the zeroth-order and first-order terms of
$Q_{\Delta t}(t,x,a;\bm\pi)$ in $\Delta t$, and the two terms are in themselves {\it independent} of $\Delta t$. Our approach  therefore has a true ``continuous-time" nature
without having to rely on $\Delta t$ or any time disretization, which theoretically  facilitates the analysis carried out in the previous section and algorithmically mitigates the high sensitivity with respect to time discretization. Our approach is similar to that introduced in \cite{baird1994reinforcement} and \cite{tallec2019making} where the value function and the (rescaled) advantage function are approximated separately. However, as pointed out already, as much as the two approaches may   lead to certain similar algorithms, they are different conceptually. The ones in \cite{baird1994reinforcement} and \cite{tallec2019making} are still based on  discrete-time MDPs and hence depend on the size of time-discretization. By contrast, the value function and q-function in q-learning are well-defined quantities in continuous time and independent of any time-discretization.

On the other hand,
approximating the value function by $J^{\theta}$ and the q-function by $q^{\psi}$ separately yields a specific approximator of the   Q-function by
\begin{equation}\label{qdt2} Q_{\Delta t}^{\theta,\psi}(t,x,a) \approx J^{\theta}(t,x) + q^{\psi}(t,x,a) \Delta t.
\end{equation}
With this Q-function approximator ${Q}^{\theta,\psi}$, one of our q-learning based algorithms actually  recovers (a modification of) SARSA, one of the most well-known conventional Q-learning algorithm.

To see this, consider state  $X_t$ at time $t$. Then the conventional SARSA with the approximator ${Q}^{\theta,\psi}$ updates parameters $(\theta,\psi)$ by
\begin{equation}
\label{eq:conventional q learning}
\begin{aligned}
\theta \leftarrow \theta + \alpha_{\theta} \bigg( & {Q}^{\theta,\psi}_{\Delta t}(t+\Delta t,X_{t+\Delta t}^{a_t^{\bm\pi^{\psi}}},a_{t+\Delta t}^{\bm\pi^{\psi}}) - \gamma\log\bm\pi^{\psi}(a_{t+\Delta t}^{\bm\pi^{\psi}}|t+\Delta t,X_{t+\Delta t}^{a_t^{\bm\pi^{\psi}}})\Delta t \\
& - {Q}^{\theta,\psi}_{\Delta t}(t,X_t,a_t^{\bm\pi^{\psi}}) +  r(t,X_t,a_t^{\bm\pi^{\psi}}) \Delta t - \beta {Q}^{\theta,\psi}_{\Delta t}(t,X_t,a_t^{\bm\pi^{\psi}}) \Delta t \bigg) \frac{\partial {Q}^{\theta,\psi}_{\Delta t}}{\partial \theta}(t,X_t,a_t^{\bm\pi^{\psi}}), \\
\psi \leftarrow \psi + \alpha_{\psi} \bigg( & {Q}^{\theta,\psi}_{\Delta t}(t+\Delta t,X_{t+\Delta t}^{a_t^{\bm\pi^{\psi}}},a_{t+\Delta t}^{\bm\pi^{\psi}}) - \gamma\log\bm\pi^{\psi}(a_{t+\Delta t}^{\bm\pi^{\psi}}|t+\Delta t,X_{t+\Delta t}^{a_t^{\bm\pi^{\psi}}})\Delta t \\
& - {Q}^{\theta,\psi}_{\Delta t}(t,X_t,a_t^{\bm\pi^{\psi}}) +  r(t,X_t,a_t^{\bm\pi^{\psi}}) \Delta t - \beta {Q}^{\theta,\psi}_{\Delta t}(t,X_t,a_t^{\bm\pi^{\psi}}) \Delta t \bigg) \frac{\partial {Q}^{\theta,\psi}_{\Delta t}}{\partial \psi}(t,X_t,a_t^{\bm\pi^{\psi}}), \\
\end{aligned}
\end{equation}
where $a_t^{\bm\pi^{\psi}}\sim \bm\pi^{\psi}(\cdot|t,X_t)$ and $a_{t+\Delta t}^{\bm\pi^{\psi}}\sim \bm\pi^{\psi}(\cdot|t+\Delta t,X^{a_t}_{t+\Delta t})$.

Note that by \eqref{qdt2}, we have
\[ \frac{\partial {Q}^{\theta,\psi}_{\Delta t}}{\partial \theta} \approx  \frac{\partial {J}^{\theta}_{\Delta t}}{\partial \theta}, \;\;\;
\frac{\partial {Q}^{\theta,\psi}_{\Delta t}}{\partial \psi} \approx  \frac{\partial {q}^{\psi}_{\Delta t}}{\partial \psi} \Delta t. \]
This means that $\frac{\partial {Q}^{\theta,\psi}_{\Delta t}}{\partial \psi}$ has an additional order of $\Delta t$ when compared with $\frac{\partial {Q}^{\theta,\psi}_{\Delta t}}{\partial \theta}$, resulting in a much slower rate of update on $\psi$ than $\theta$. This observation is also made by
\citet{tallec2019making}, who  suggest as a remedy modifying the update rule for Q-function by replacing the partial derivative of the Q-function  with that of the advantage function. In our setting, this means we change $\frac{\partial {Q}^{\theta,\psi}_{\Delta t}}{\partial \psi}$ to
$\frac{\partial {q}^{\psi}_{\Delta t}}{\partial \psi} $
in the second update rule of \eqref{eq:conventional q learning}.

Now, the terms in the brackets in \eqref{eq:conventional q learning}  can be further rewritten as:
\[\begin{aligned}
& {Q}^{\theta,\psi}_{\Delta t}(t+\Delta t,X_{t+\Delta t}^{a_t^{\bm\pi^{\psi}}},a_{t+\Delta t}^{\bm\pi^{\psi}}) - \gamma\log\bm\pi^{\psi}(a_{t+\Delta t}^{\bm\pi^{\psi}}|t+\Delta t,X_{t+\Delta t}^{a_t^{\bm\pi^{\psi}}})\Delta t\\
& - {Q}^{\theta,\psi}_{\Delta t}(t,X_t,a_t^{\bm\pi^{\psi}}) +  r(t,X_t,a_t^{\bm\pi^{\psi}}) \Delta t - \beta {Q}^{\theta,\psi}_{\Delta t}(t,X_t,a_t^{\bm\pi^{\psi}}) \Delta t \\
= & J^{\theta}(t+\Delta t,X_{t+\Delta t}^{a_t^{\bm\pi^{\psi}}}) - J^{\theta}(t,X_t) + [ q^{\psi}(t+\Delta t,X_{t+\Delta t}^{a_t^{\bm\pi^{\psi}}},a_{t+\Delta t}^{\bm\pi^{\psi}}) - q^{\psi}(t,X_t,a_t^{\bm\pi^{\psi}}) ]\Delta t \\
& + r(t,X_t,a_t^{\bm\pi^{\psi}}) \Delta t - \beta  J^{\theta}(t,X_t) \Delta t - \beta q^{\psi}(t,X_t,a_t^{\bm\pi^{\psi}})  (\Delta t)^2 - \gamma\log\bm\pi^{\psi}(a_{t+\Delta t}^{\bm\pi^{\psi}}|t+\Delta t,X_{t+\Delta t}^{a_t^{\bm\pi^{\psi}}})\Delta t \\
= &  J^{\theta}(t+\Delta t,X_{t+\Delta t}^{a_t^{\bm\pi^{\psi}}}) - J^{\theta}(t,X_t)  - q^{\psi}(t,X_t,a_t^{\bm\pi^{\psi}}) \Delta t + r(t,X_t,a_t^{\bm\pi^{\psi}}) \Delta t - \beta  J^{\theta}(t,X_t) \Delta t \\
& + \left[q^{\psi}(t+\Delta t,X_{t+\Delta t}^{a_t^{\bm\pi^{\psi}}},a_{t+\Delta t}^{\bm\pi^{\psi}}) - \gamma\log\bm\pi^{\psi}(a_{t+\Delta t}^{\bm\pi^{\psi}}|t+\Delta t,X_{t+\Delta t}^{a_t^{\bm\pi^{\psi}}})\right] \Delta t - \beta q^{\psi}(t,X_t,a_t^{\bm\pi^{\psi}})  (\Delta t)^2 \\
\approx & \dd J^{\theta}(t,X_t)  - q^{\psi}(t,X_t,a_t^{\bm\pi^{\psi}}) \dd t + r(t,X_t,a_t^{\bm\pi^{\psi}}) \dd t - \beta  J^{\theta}(t,X_t) \dd t \\
& + \underbrace{\left[q^{\psi}(t+\Delta t,X_{t+\Delta t}^{a_t^{\bm\pi^{\psi}}},a_{t+\Delta t}^{\bm\pi^{\psi}}) - \gamma\log\bm\pi^{\psi}(a_{t+\Delta t}^{\bm\pi^{\psi}}|t+\Delta t,X_{t+\Delta t}^{a_t^{\bm\pi^{\psi}}})\right]}_{\text{mean-0 term due to \eqref{eq:constraints qv}}} \dd t ,
\end{aligned} \]
where in the last step we drop the higher-order small term $(\Delta t)^2$  while replacing the difference terms  with differential ones.

Comparing \eqref{eq:conventional q learning} with the updating rule in Algorithm \ref{algo:online incremental} using test functions $\xi_t = \frac{\partial J^{\theta}}{\partial \theta}(t,X_t^{\bm\pi^{\psi}})$ and $\zeta_t = \frac{\partial q^{\psi}}{\partial \psi}(t,X_t^{\bm\pi^{\psi}},a_t^{\bm\pi^{\psi}})$, we find that
the modified SARSA in \cite{tallec2019making} and Algorithm \ref{algo:online incremental} 
differ only by a term 
whose mean is 0 due to the constraint \eqref{eq:constraints qv} on the q-function. This term is solely driven by the action randomization at time $t+\Delta t$. Hence,
the algorithm in \cite{tallec2019making} is noisier and may be slower in convergence speed compared with the q-learning  one. In Section \ref{sec:applications}, we will numerically compare the results of the above  Q-learning SARSA  and q-learning.

\section{q-Learning Algorithms When Normalizing Constant Is Unavailable}
\label{sec:noncomputable}

Theorem \ref{thm:qv learn} 
assumes that the normalizing constant in the Gibbs measure is available so that one can enforce the constraint \eqref{eq:q hjb2}. So does Theorem \ref{thm:q optimal} because the second constraint in \eqref{eq:q hjb2 optimal} is exactly to verify the normalizing constant to be 1. But it is well known that in most high-dimensional  cases computing this constant is a daunting, and often impossible task.

\subsection{A stronger policy improvement theorem}
\label{subsec:stronger PIT}

To overcome the difficulty arising from an unavailable  normalizing constant in soft Q-learning,
\cite{haarnoja2018soft} propose a general method of using a family of stochastic policies whose densities can be easily computed  to approximate $\bm\pi'$. In our current setting, denote by $\{ \bm\pi^{\phi}(\cdot|t,x)\}_{\phi\in \Phi}$  the family of density functions of some tractable distributions, e.g., the multivariate normal distribution
$ \bm\pi^{\phi}(\cdot|t,x) =  \mathcal{N}\left( \mu^{\phi}(t,x), \Sigma^{\phi}(t,x) \right)$,
where $\mu^{\phi}(t,x)\in \mathbb{R}^m$ and $\Sigma^{\phi}(t,x)\in \mathbb{S}^m_{++}$.
The learning procedure then proceeds as follows: starting from a policy $\bm\pi^{\phi}$ within this family, project the desired policy $\frac{\exp\{ \frac{1}{\gamma}q(t,x,\cdot;\bm\pi^{\phi})\} }{\int_{\mathcal{A}} \exp\{ \frac{1}{\gamma}q(t,x,a;\bm\pi^{\phi})\} \dd a}$ to the set $\{ \bm\pi^{\phi}(a|t,x)\}_{\phi\in \Phi}$ by minimizing
{\small \[ \min_{\phi'\in \Phi} D_{KL}\left( \bm\pi^{\phi'}(\cdot|t,x)\Big|\Big| \frac{\exp\{ \frac{1}{\gamma}q(t,x,\cdot;\bm\pi^{\phi})\} }{\int_{\mathcal{A}} \exp\{ \frac{1}{\gamma}q(t,x,a;\bm\pi^{\phi})\} \dd a} \right) \equiv  \min_{\phi'\in \Phi} D_{KL}\left( \bm\pi^{\phi'}(\cdot|t,x)\Big|\Big| \exp\{ \frac{1}{\gamma}q(t,x,\cdot;\bm\pi^{\phi})\} \right)  , \]}
where $D_{KL}\left(f\Big|\Big|g\right): = \int_{\mathcal{A}} \log\frac{f(a)}{g(a)} f(a)\dd a $ is the Kullback--Leibler (KL) divergence of two positive functions $f,g$ with the same support on $\mathcal{A}$, where $f\in \mathcal{P}(\mathcal{A})$ is a probability density function on $\mathcal{A}$.

This procedure may not eventually lead to the desired target policy, but the newly defined policy still provably improves the previous policy, as indicated in the following stronger version of the policy improvement theorem that generalizes Theorem \ref{lemma:policy improvement hjb}.
\begin{theorem}
\label{lemma:policy improvement hjb constraint}
Given $(t,x)\in [0,T]\times \mathbb{R}^d$, if two policies $\bm\pi\in \bm\Pi$ and $\bm\pi'\in \bm\Pi$ satisfy
\[
\begin{aligned}
& D_{KL}\left( \bm\pi'(\cdot|t,x)\Big|\Big| \exp\{ \frac{1}{\gamma}H\big( t,x,\cdot,\frac{\partial J}{\partial x}(t,x;\bm\pi), \frac{\partial^2 J}{\partial x^2}(t,x;\bm\pi) \big)    \} \right) \\
\leq & D_{KL}\left( \bm\pi(\cdot|t,x)\Big|\Big| \exp\{ \frac{1}{\gamma}H\big( t,x,\cdot,\frac{\partial J}{\partial x}(t,x;\bm\pi), \frac{\partial^2 J}{\partial x^2}(t,x;\bm\pi) \big)    \} \right),
\end{aligned}
\]
then $ J(t,x;\bm\pi') \geq J(t,x;\bm\pi)$.
\end{theorem}

Theorem \ref{lemma:policy improvement hjb constraint} in itself is a general result comparing two given policies not necessarily within any tractable family of densities. However, an implication of the result is that,
given a current tractable policy $\bm\pi^{\phi}$,  as long as we update it to a new tractable policy $\bm\pi^{\phi'}$ with
\[
D_{KL}\left( \bm\pi^{\phi'}(\cdot|t,x)\Big|\Big| \exp\{ \frac{1}{\gamma }q(t,x,\cdot;\bm\pi^{\phi})\} \right) \leq D_{KL}\left( \bm\pi^{\phi}(\cdot|t,x)\Big|\Big| \exp\{ \frac{1}{\gamma }q(t,x,\cdot;\bm\pi^{\phi})\} \right),
\]
then $\bm\pi^{\phi'}$ improves upon $\bm\pi^{\phi}$.
Thus, to update $\bm\pi^{\phi}$  it suffices to solve the optimization problem
\[\begin{aligned}
& \min_{\phi'\in \Phi} D_{KL}\left( \bm\pi^{\phi'}(\cdot|t,x)\Big|\Big| \exp\{ \frac{1}{\gamma }q(t,x,\cdot;\bm{\pi}^{\phi})\} \right)   \\
= & \min_{\phi'\in \Phi} \int_{\mathcal{A}} \left[\log\bm\pi^{\phi'}(a|t,x) - \frac{1}{\gamma }q(t,x,a;\bm{\pi}^{\phi})   \right] \bm\pi^{\phi'}(a|t,x)\dd a .
\end{aligned} \]

The gradient of the above objective function in $\phi'$ is
\[ \begin{aligned}
& \frac{\partial }{\partial \phi'} \int_{\mathcal{A}} \left[\log\bm\pi^{\phi'}(a|t,x) - \frac{1}{\gamma }q(t,x,a;\bm\pi^{\phi})   \right] \bm\pi^{\phi'}(a|t,x)\dd a \\
= & \int_{\mathcal{A}} \left[\log\bm\pi^{\phi'}(a|t,x) - \frac{1}{\gamma }q(t,x,a;\bm\pi^{\phi})     \right] \frac{\partial \bm\pi^{\phi'}(a|t,x)}{\partial \phi'}  \dd a  + \int_{\mathcal{A}}\left[ \frac{\partial }{\partial \phi'}\log\bm\pi^{\phi'}(a|t,x)\right]   \bm\pi^{\phi'}(a|t,x)\dd a \\
= & \int_{\mathcal{A}}  \left[\log\bm\pi^{\phi'}(a|t,x) - \frac{1}{\gamma }q(t,x,a;\bm\pi^{\phi})  \right] \left[ \frac{\partial }{\partial \phi'}\log\bm\pi^{\phi'}(a|t,x)\right]   \bm\pi^{\phi'}(a|t,x)\dd a,
\end{aligned} \]
where we have noted
$$\int_{\mathcal{A}}\left[ \frac{\partial }{\partial \phi'}\log\bm\pi^{\phi'}(a|t,x)\right]   \bm\pi^{\phi'}(a|t,x)\dd a = \int_{\mathcal{A}} \frac{\partial }{\partial \phi'} \bm\pi^{\phi'}(a|t,x)\dd a = \frac{\partial }{\partial \phi'} \int_{\mathcal{A}}  \bm\pi^{\phi'}(a|t,x)\dd a = 0.$$

Therefore, we may update $\phi$ each step incrementally by
\begin{equation}
\label{eq:phi_update}
\phi \leftarrow \phi - \gamma\alpha_{\phi}\dd t\left[\log\bm\pi^{\phi}(a^{\bm\pi^{\phi}}_t|t,X_t) - \frac{1}{\gamma }q(t,X_t,a^{\bm\pi^{\phi}}_t;\bm\pi^{\phi})  \right] \frac{\partial }{\partial \phi}\log\bm\pi^{\phi}(a_t^{\bm\pi^{\phi}}|t,X_t), \end{equation}
where we choose the step size to be specifically $\gamma \alpha_{\phi}\dd t$, for the reason that will become evident momentarily.

The above updating rule seems to indicate that we would need to learn the q-function $q(t,x,a;\bm\pi^{\phi})$ associated with the policy $\bm\pi^{\phi}$. This is doable  in view of Theorem \ref{thm:qv learn}, noting that $\{ \bm\pi^{\phi}(\cdot|t,x)\}_{\phi\in \Phi}$ is a tractable family of distributions so that the constraint \eqref{eq:q hjb2} or \eqref{eq:constraints qv} can be computed and enforced. However, we do not need to do so given that \eqref{eq:phi_update} only requires the q-value along the trajectory $\{(t,X_t,a^{\bm\pi^{\phi}}_t);0\leq t\leq T\}$, instead of the full functional form $q(\cdot,\cdot,\cdot;\bm\pi^{\phi})$. The former is easier to computed  from the temporal difference learning, as will be evident in the next subsection.

\subsection{Connections with policy gradient}
\label{subsec:connect with pg}

Applying It\^o's lemma to $J(\cdot,\cdot;\bm\pi)$ and recalling Definition \ref{def:q rate function},  we have the following  important relation between the q-function and the temporal difference of the value function:
\begin{equation}
\label{eq:q rate approximation}
q(t,X_t^{\bm\pi},a_t^{\bm\pi};\bm\pi)\dd t = \dd J(t,X_t^{\bm\pi};\bm\pi) + r(t,X_t^{\bm\pi},a_t^{\bm\pi})\dd t - \beta J(t,X_t^{\bm\pi};\bm\pi)\dd t + \{\cdots\} \dd W_t.
\end{equation}
Plugging the above to (\ref{eq:phi_update}) and ignoring the martingale difference term $\{\cdots\} \dd W_t$ whose mean is 0, the gradient-based updating rule for $\phi$ becomes
\begin{equation}
\label{eq:pg update}
\begin{aligned}
\phi \leftarrow \phi & + \alpha_{\phi}\left[-\gamma\log\bm\pi^{\phi}(a^{\bm\pi^{\phi}}_t|t,X^{\bm\pi^{\phi}}_t)\dd t + \dd J(t,X_t^{\bm\pi^{\phi}};\bm\pi^{\phi}) + r(t,X_t^{\bm\pi^{\phi}},a_t^{\bm\pi^\phi})\dd t - \beta J(t,X_t^{\bm\pi^{\phi}};\bm\pi^\phi)\dd t \right]\\
& \times \frac{\partial }{\partial \phi}\log\bm\pi^{\phi}(a_t^{\bm\pi^{\phi}}|t,X_t^{\bm\pi^{\phi}}).
\end{aligned}
\end{equation}
In this updating rule, we only need to approximate $J(\cdot,\cdot;\bm\pi^\phi)$, which is a problem of PE. Specifically, denote by $J^{\theta}(\cdot,\cdot)$ the function approximator of $J(\cdot,\cdot;\bm\pi^\phi)$, where $\theta\in \Theta$, and we may apply any PE methods developed in \cite{jia2021policy} to learn $J^{\theta}$. Given this value function approximation, the updating rule \eqref{eq:pg update} further specializes to
\begin{equation}
\label{eq:pg update simplify}
\begin{aligned}
\phi \leftarrow \phi & + \alpha_{\phi}\left[-\gamma\log\bm\pi^{\phi}(a^{\bm\pi^{\phi}}_t|t,X_t^{\bm\pi^{\phi}})\dd t + \dd J^{\theta}(t,X_t^{\bm\pi^{\phi}}) + r(t,X_t^{\bm\pi^{\phi}},a_t^{\bm\pi^\phi})\dd t - \beta J^{\theta}(t,X_t^{\bm\pi^{\phi}})\dd t \right]\\
& \times \frac{\partial }{\partial \phi}\log\bm\pi^{\phi}(a_t^{\bm\pi^{\phi}}|t,X_t^{\bm\pi^{\phi}}).
\end{aligned}
\end{equation}
Expression \eqref{eq:pg update simplify} coincides with the updating rule in the {\it policy gradient} method established in \cite{jia2021policypg} when the regularizer therein is taken to be the entropy. Moreover, since we learn the value function and the policy simultaneously, the updating rule leads to actor--critic type of algorithms. Finally, with  this method, while learning the value function may involve martingale conditions as in \cite{jia2021policy}, learning the policy does not.

\cite{schulman2017equivalence} note
the equivalence between soft Q-learning and  PG in discrete time. Here we present the continuous-time counterpart of the equivalence, nevertheless with a theoretical justification, which  recovers the PG based algorithms in \cite{jia2021policypg} when combined with suitable PE methods. 

\section{Extension to Ergodic Tasks}
\label{sec:extension}

We now extend the previous study to ergodic tasks, in which  the objective is to maximize the long-run average in the infinite time horizon $[0,\infty)$, the functions $b$, $\sigma$ and $r$ do not depend on  time $t$ explicitly, and $h=0$.  The set of admissible (stationary) policies can be similarly defined
in a straighforward manner.
The  regularized  ergodic objective function is the long-run average:
\[\begin{aligned}
& \liminf_{T\to \infty}\frac{1}{T}\E^{\p^W}\bigg[ \int_t^T \int_{\mathcal{A}} [ r(\tilde{X}_s^{\bm{\pi}},a) - \gamma \log\bm{\pi}(a|\tilde{X}_s^{\bm{\pi}})  ]\bm{\pi}(a|\tilde{X}_{s}^{\bm{\pi}})\dd a  \dd s \Big| \tilde{X}_t^{\bm{\pi}} = x\bigg] \\
= & \liminf_{T\to \infty}\frac{1}{T}\E^{\p}\bigg[ \int_t^T [ r(X_s^{\bm{\pi}},a_s^{\bm{\pi}}) - \gamma \log\bm{\pi}(a_s^{\bm{\pi}}|X_s^{\bm{\pi}}) ] \dd s \Big| X_t^{\bm{\pi}} = x\bigg],
\end{aligned}  \]
where $\gamma \geq 0$ is the temperature parameter. 

We first present the ergodic version of the Feynman--Kac formula and the corresponding policy improvement theorem.
\begin{lemma}
\label{lemma:ergodic feynmann-kac}
Let $\bm{\pi}=\bm{\pi}(\cdot|\cdot)$ be a given admissible (stationary) policy.\footnote{For such ergodic tasks, typically an admissible policies also require the process $X^{\bm\pi}$ is ergodic (cf. \citealt{meyn1993stability}).}
Suppose there is a function $J(\cdot;\bm{\pi})\in C^2(\mathbb{R}^d)$ with polynomial growth and a scalar $V(\bm{\pi}) \in \mathbb{R}$ satisfying 
\begin{equation}
\label{eq:relaxed control f-k formula ergodic}
\int_{\mathcal{A}} \left[ H\big(x,a, \frac{\partial J}{\partial x}(x;\bm{\pi}),\frac{\partial^2 J}{\partial x^2}(x;\bm{\pi}) \big) -\gamma \log\bm\pi(a|x) \right]\bm{\pi}(a|x) \dd a - V(\bm{\pi})= 0,\;\;x\in \mathbb{R}^d.
\end{equation}
Then for any $t\geq0$, 
\begin{equation}
\label{con1}
\begin{array}{rl}
V(\bm{\pi}) = & \liminf_{T\to \infty}\frac{1}{T}\E^{\p^W}\bigg( \int_t^T \int_{\mathcal{A}} [ r(\tilde{X}_s^{\bm{\pi}},a) - \gamma \log\bm\pi(a|\tilde{X}_s^{\bm{\pi}}) ]\bm{\pi}(a|\tilde{X}_{s}^{\bm{\pi}})\dd a  \dd s \Big| \tilde{X}_t^{\bm{\pi}} = x\bigg) \\
= & \liminf_{T\to \infty}\frac{1}{T}\E^{\p}\bigg( \int_t^T [ r(X_s^{\bm{\pi}},a_s^{\bm{\pi}}) - \gamma \log\bm\pi(a_s^{\bm{\pi}}|X_s^{\bm{\pi}}) ] \dd s \Big| X_t^{\bm{\pi}} = x\bigg).
\end{array}
\end{equation}

Moreover, if there are two admissible policies $\bm\pi$ and $\bm\pi'$ such that for all $x\in \mathbb{R}^d$,
\[
\begin{aligned}
& D_{KL}\left( \bm\pi'(\cdot|x)\Big|\Big| \exp\{ \frac{1}{\gamma}H\big( x,\cdot,\frac{\partial J}{\partial x}(x;\bm\pi), \frac{\partial^2 J}{\partial x^2}(x;\bm\pi) \big)    \} \right) \\
\leq & D_{KL}\left( \bm\pi(\cdot|x)\Big|\Big| \exp\{ \frac{1}{\gamma}H\big( x,\cdot,\frac{\partial J}{\partial x}(x;\bm\pi), \frac{\partial^2 J}{\partial x^2}(x;\bm\pi) \big)    \} \right) ,
\end{aligned}
\]
then $V(\bm\pi') \geq V(\bm\pi)$.

\end{lemma}

We emphasize that the solution to \eqref{eq:relaxed control f-k formula ergodic} is a {\it pair} of $(J, V)$, where $J$ is a function of the state and $V\in \mathbb{R}$ is a scalar. The long term average of the payoff does not depend on the initial state $x$ nor the initial time $t$ due to  ergodicity, and hence is a constant as \eqref{con1} implies. The function $J$, on the other hand, represents the first-order approximation of the long-run average and is not unique. Indeed, for any constant $c$, $(J+c, V)$ is also a solution to \eqref{eq:relaxed control f-k formula ergodic}. We  refer to $V$ as the {\it value} of the underlying problem and still refer to $J$ as the value function. Lastly, since the value  does not depend on the  initial time,  we will fix the latter  as 0 in the following discussions and applications of ergodic learning tasks.

As with the episodic case, for any admissible policy $\bm{\pi}$, we can define the $\Delta t$-parameterized   Q-function as
\[
\begin{aligned}
& Q_{\Delta t}(x,a;\bm{\pi} )\\
= & \E^{\p^W}\bigg[ \int_t^{t+\Delta t} [r(X_s^{a},a) - V(\bm\pi)]\dd s + \lim_{T\to \infty}\E^{\p}\big[ \int_{t+\Delta t}^T [r(X_s^{\bm{\pi}},a_s^{\bm \pi})-\gamma\log \bm{\pi}(a_s^{\bm \pi}|X_s^{\bm{\pi}}) \\
& - V(\bm\pi) ]\dd s |X_{t+\Delta t}^a\big]\Big|X_t^{\bm \pi} = x \bigg] \\
= & \E^{\p^W}\bigg[ \int_t^{t+\Delta t} [r(X_s^{a},a) - V(\bm\pi)]\dd s + J(X_{t+ \Delta t}^{a};\bm{\pi})\Big|X_t^{\bm \pi} = x \bigg]\\
& - \E^{\p^W}\bigg[ \lim_{T\to \infty} \E^{\p}\big[ J(X_T^{\bm\pi};\bm\pi) | X_{t+\Delta t}^a \big] \Big|X_t^{\bm \pi} = x \bigg] \\
= & \E^{\p^W}\bigg[ \int_t^{t+\Delta t} [r(X_s^{a},a)- V(\bm\pi)]\dd s + J( X_{t+ \Delta t}^{a};\bm{\pi}) - J(x;\bm{\pi})\Big|X_t^{\bm \pi} = x \bigg] + J(x;\bm{\pi})\\
& -\E^{\p^W}\bigg[ \lim_{T\to \infty} \E^{\p}\big[ J(X_T^{\bm\pi};\bm\pi) | X_{t+\Delta t}^a \big] \Big|X_t^{\bm \pi} = x \bigg] \\
= & \E^{\p^W}\bigg[ \int_t^{t+\Delta t}  \big[  H\big( X_s^{a},a,\frac{\partial J}{\partial x}(X_s^a;\bm{\pi}), \frac{\partial^2 J}{\partial x^2}(X_s^a;\bm{\pi})\big)   -V(\bm\pi) \big]\dd s\Big|X_t^{\bm \pi} = x \bigg] + J(x;\bm{\pi}) \\
& - \E^{\p^W}\bigg[ \lim_{T\to \infty}\E^{\p}\big[ J(X_T^{\bm\pi};\bm\pi) | X_{t+\Delta t}^a \big] \Big|X_t^{\bm \pi} = x \bigg] \\
= & J(x;\bm{\pi}) + \left[  H\left( x,a,\frac{\partial J}{\partial x}(x;\bm{\pi}), \frac{\partial^2 J}{\partial x^2}(x;\bm{\pi})\right) - V(\bm\pi) \right]  \Delta t - \bar{J} + O\big( (\Delta t)^2 \big),
\end{aligned}
\]
where we assumed that $\lim_{T\to \infty} \E^{\p}\big[ J(X_T^{\bm\pi};\bm\pi) | X_{t}^{\bm\pi} = x \big] = \bar{J}$ is a constant for any $(t,x)\in [0,+\infty) \times \mathbb{R}^d$.

The corresponding q-function is then defined as
\[ q(x,a;\bm\pi) =  H\left( x,a,\frac{\partial J}{\partial x}(x;\bm{\pi}), \frac{\partial^2 J}{\partial x^2}(x;\bm{\pi})\right) - V(\bm\pi) . \]
As a counterpart to Theorem \ref{thm:qv learn}, we have the following theorem to characterize the value function, the value, and the q-function for both on-policy and off-policy ergodic learning problems. A counterpart to Theorem \ref{thm:q optimal} can be similarly established, and we leave details to the reader.

\begin{theorem}
\label{thm:qv learn ergodic}
Let an admissible policy $\bm\pi$, a number $\hat{V}$, a function $\hat{J}\in C^{2}\big(\mathbb{R}^d \big)$ with polynomial growth, and a continuous function  $\hat{q}: \mathbb{R}^d\times \mathcal{A}\to \mathbb{R}$
be given satisfying
\begin{equation}
\label{eq:q hjb2eg}
\lim_{T\to \infty}\frac{1}{T}\E[\hat{J}(X_T^{\bm\pi})] = 0, \;\;\;\int_{\mathcal{A}} \big[ \hat{q}(x,a) -\gamma \log{\bm\pi}(a|x) \big] {\bm\pi}(a|x)\dd a =0,\;\;\forall x\in \mathbb{R}^d.
\end{equation}
Then
\begin{enumerate}[(i)]
\item $\hat{V}$, $\hat{J}$ and $\hat{q}$ are respectively the value, the value function and the q-function associated with ${\bm\pi}$ if and only if for all $x\in\mathbb{R}^d$, the following process
\begin{equation}
\label{eq:martingale with q function2eg}
\hat{J}(X_t^{{\bm\pi}}) + \int_0^t  [r(X_s^{{\bm\pi}},a_s^{{\bm\pi}}) - \hat{q}(X_s^{{\bm\pi}},a_s^{{\bm\pi}})  - \hat{V}]\dd s
\end{equation}
is an $(\{\f_t\}_{t\geq 0},\p)$-martingale, where $\{{X}_t^{\bm\pi}, 0\leq t<\infty\}$ is the solution to (\ref{eq:model pi}) with ${X}_0^{\bm\pi}=x$.
\item If $\hat V$, $\hat{J}$ and $\hat{q}$ are respectively the value, value function and the q-function associated with ${\bm\pi}$, then for any
admissible  $\bm\pi'$ and any $x\in\mathbb{R}^d$, the following process
\begin{equation}
\label{eq:martingale with q function2eg-off}
\hat J({X}_t^{\bm\pi'}) + \int_0^t  [r({X}_{u}^{\bm\pi'},a^{\bm\pi'}_{u}) - \hat{q}({X}_{u}^{\bm\pi'},a^{\bm\pi'}_{u}) - \hat V ]\dd u
\end{equation}
is an $(\{\f_t\}_{t\geq 0}, \p)$-martingale, where $\{{X}_t^{\bm\pi'}, 0\leq t<\infty $ is the solution to (\ref{eq:model pi}) under $\bm\pi'$ with initial condition ${X}_{0}^{\bm\pi'} = x$.
\item If there exists an admissible $\bm\pi'$ such that for all $x\in\mathbb{R}^d$, \eqref{eq:martingale with q function2eg-off} is an $(\{\f_t\}_{t\geq 0}, \p)$-martingale where ${X}_{0}^{\bm\pi'} = x$, then  $\hat V$, $\hat{J}$ and $\hat{q}$ are respectively the value, value function and the q-function associated with ${\bm\pi}$.
\end{enumerate}
Moreover, in any of the three cases above,
if it holds further that  ${\bm\pi}(a|x) = \frac{\exp\{  \frac{1}{\gamma}\hat{q}(x,a) \}}{\int_{{\cal A}} \exp\{  \frac{1}{\gamma}\hat{q}(x,a) \} \dd a }$, then ${\bm\pi} $ is the optimal policy and $\hat{V}$ is the optimal value.

\end{theorem}

Based on Theorem \ref{thm:qv learn ergodic}, we can design corresponding q-learning algorithms for ergodic problems, in which we learn $V,J^{\theta}$ and $q^{\psi}$ at the same time. These algorithms have natural connections with SARSA as well as  the PG-based algorithms developed in \cite{jia2021policypg}. The discussions are similar to those  in Subsections \ref{subsec:connect with q} and \ref{subsec:connect with pg} and hence are omitted here. An online algorithm for ergodic tasks is presented as an example; see Algorithm \ref{algo:ergodic incremental}.

\begin{algorithm}[h]
\caption{q-Learning Algorithm for Ergodic Tasks}
\textbf{Inputs}: initial state $x_0$, time step $\Delta t$, initial learning rates $\alpha_{\theta},\alpha_{\psi},\alpha_{V}$ and learning rate schedule function $l(\cdot)$ (a function of time), functional forms of the parameterized value function $J^{\theta}(\cdot)$ and q-function $q^{\psi}(\cdot,\cdot)$ satisfying \eqref{eq:q hjb2eg}, functional forms of  test functions $\bm{\xi}(t,x_{\cdot \wedge t},a_{\cdot \wedge t})$, $\bm{\zeta}(t,x_{\cdot \wedge t},a_{\cdot \wedge t})$, and temperature parameter $\gamma$.

\textbf{Required program (on-policy)}: an environment simulator $(x',r) = \textit{Environment}_{\Delta t}(x,a)$ that takes initial state $x$ and action $a$ as inputs and generates a new state $x'$ at $\Delta t$ and an instantaneous reward $r$ as outputs. Policy $\bm\pi^{\psi}(a|x) = \frac{\exp\{  \frac{1}{\gamma}q^{\psi}(x,a) \}}{\int_{{\cal A}} \exp\{  \frac{1}{\gamma}q^{\psi}(x,a) \} \dd a }$.

\textbf{Required program (off-policy)}: observations $ \{a, r, x'\} = \textit{Observation}( x;\Delta t)$ including  the observed actions, rewards, and state when the current state is $x$ under the given behavior policy at the sampling time grids with step size $\Delta t$.

\textbf{Learning procedure}:
\begin{algorithmic}
\STATE Initialize $\theta,\psi,V$. Initialize $k = 0$. Observe the initial state $x_0$ and store $x_{t_k} \leftarrow  x_0$.

\LOOP \STATE{
\COMMENT{\textbf{On-policy case}

Generate action $a\sim \bm{\pi}^{\psi}(\cdot|x)$.

Apply $a$ to  environment simulator $(x',r) = Environment_{\Delta t}(x, a)$, and observe new state $x'$ and reward $r$ as outputs. Store $x_{t_{k+1}} \leftarrow x'$.

}

\COMMENT{\textbf{Off-policy case}

Obtain one observation $a_{t_k}, r_{t_k}, x_{t_{k+1}} = \textit{Observation}(x_{t_k};\Delta t)$.

}

Compute test functions $\xi_{t_k} = \bm{\xi}(t_k,x_{t_0},\cdots, x_{t_k},a_{t_0},\cdots, a_{t_k})$, $\zeta_{t_k} = \bm{\zeta}(t_k, x_{t_0},\cdots, x_{t_k},a_{t_0},\cdots, a_{t_k})$.		

Compute
\[\begin{aligned}
& \delta = J^{\theta}(x') - J^{\theta}(x) + r\Delta t -q^{\psi}(x,a)\Delta t - V \Delta t, \\
& \Delta \theta = \xi_{t_k} \delta, \\
& \Delta V = \delta, \\
& \Delta \psi =\zeta_{t_k} \delta.
\end{aligned}  \]

Update $\theta$, $V$ and $\psi$ by
\[ \theta \leftarrow \theta + l(k\Delta t)\alpha_{\theta} \Delta \theta,\]
\[ V \leftarrow V + l(k\Delta t)\alpha_{V} \Delta V,\]
\[ \psi \leftarrow \psi + l(k\Delta t)\alpha_{\psi}   \Delta \psi.  \]

Update $x\leftarrow x'$ and $k \leftarrow k + 1$.
}
\ENDLOOP	

\end{algorithmic}
\label{algo:ergodic incremental}
\end{algorithm}

\section{Applications}
\label{sec:applications}
\subsection{Mean--variance portfolio selection}
\label{sec:application mv}
We first review the formulation of the exploratory mean--variance portfolio selection problem, originally  proposed by \citet{wang2020continuous} and later revisited by
\citet{jia2021policypg}.\footnote{There is a vast literature on classical model-based (non-exploratory) continuous-time  mean--variance models formulated as stochastic control problems; see e.g.  \cite{zhou2000continuous,lim2002mean, zhou2003markowitz} and the references therein. } The investment universe consists of one risky asset (e.g., a stock index) and one risk-free asset (e.g., a saving account) whose risk-free interest rate is $r$. The price of the risky asset $\{S_t, 0\leq t \leq T\}$ is governed by a geometric Brownian motion with drift $\mu$ and volatility $\sigma>0$, defined on a filtered probability space $(\Omega,\f,\p^W;\{\f_t^W\}_{0\leq t\leq T})$.
An agent has a fixed investment horizon $0<T<\infty$ and an initial endowment $x_0$.
A self-financing portfolio is represented by the real-valued adapted process $a = \{a_t, 0\leq t \leq T\}$, where $a_t$ is the discounted dollar value invested in the risky asset at time $t$.
Then her discounted wealth process follows
\[ \dd X_t = a_t[(\mu-r)\dd t+\sigma\dd W_t] = a_t \frac{\dd (e^{-rt}S_t)}{e^{-rt} S_t},\;\; X_0 = x_0. \]

The agent aims to minimize the variance of the discounted value of the portfolio at time $T$ while maintaining the expected return to be a certain level; that is,
\begin{equation}
\label{eq:mv} \min_{a} \text{Var}(X_T^{a}),\ \mbox{subject to }\ \E[X_T^{a}] = z,
\end{equation}
where $z$ is the target expected return, and the variance and expectation throughout this subsection are with respect to the probability measure $\p^W$.

The exploratory formulation of this problem with entropy regularizer is equivalent to
\begin{equation}
\label{eq:entropy mv objective function}
J(t,x;w) = -\max_{\bm{\pi}} \E\bigg[-(\tilde{X}_T^{\bm{\pi}} - w)^2 - \gamma \int_t^T \log\bm\pi(a_s^{\bm\pi}|s,\tilde{X}^{\bm\pi})\dd s\Big| \tilde{X}_t^{\bm{\pi}} = x \bigg] + (w-z)^2,
\end{equation}
subject to
\[ \dd \tilde{X}^{\bm{\pi}}_s = (\mu-r)\int_{\mathbb{R}} a \bm{\pi}(a|s, \tilde{X}^{\bm{\pi}}_s)  \dd a \dd s + \sigma \sqrt{\int_{\mathbb{R}} a^2 \bm{\pi}(a|s, \tilde{X}^{\bm{\pi}}_s)  \dd a } \dd W_s;\;\tilde{X}_t^{\bm{\pi}} = x, \]
where $w$ is the Lagrange multiplier introduced to relax the expected return constraint; see \citet{wang2020continuous} for a derivation of this formulation.  Note here we artificially add the minus sign to transform the variance minimization to a  maximization problem to be consistent with our general formulation.

We now follow \cite{wang2020continuous} and \cite{jia2021policypg} to parameterize the value function by
\[ J^{\theta}(t,x;w) = (x - w)^2 e^{-\theta_3(T-t)} + \theta_2(t^2 - T^2) + \theta_1(t-T) - (w-z)^2.  \]
Moreover, we parameterize the q-function by
\[ q^{\psi}(t,x,a;w) = -\frac{e^{-\psi_1 - \psi_3(T-t)}}{2} \big(a + \psi_2 (x - w) \big)^2 - \frac{\gamma}{2}[\log2\pi\gamma + \psi_1 + \psi_3(T-t) ], \]
which is derived to satisfy the constraint \eqref{eq:constraints qv}, as explained in Example \ref{eg:linear quadratic}. The policy associated with this parametric q-function is $\bm\pi^{\psi}(\cdot|x;w) = \mathcal{N}(-\psi_2(x - w), \gamma e^{\psi_1 + \psi_3(T-t)})$. In addition to $\theta=(\theta_1,\theta_2,\theta_3)^\top$ and $\psi=(\psi_1,\psi_2,\psi_3)^\top$, we also learn the Lagrange multiplier $w$ by stochastic approximation in the same way as in \cite{wang2020continuous} and \cite{jia2021policypg}. The full algorithm is summarized in Algorithm \ref{algo:offline episodic mv}.\footnote{Here we present an offline algorithm as example. Online algorithms can also be devised following the general
study in the previous sections.}

\begin{algorithm}[htbp]
\caption{Offline--Episodic q-Learning Mean--Variance Algorithm}
\textbf{Inputs}: initial state $x_0$, horizon $T$, time step $\Delta t$, number of episodes $N$, number of time grids $K$, initial learning rates $\alpha_{\theta},\alpha_{\psi},\alpha_w$ and learning rate schedule function $l(\cdot)$ (a function of the number of episodes), and temperature parameter $\gamma$.

\textbf{Required program}: a market simulator $x' = \textit{Market}_{\Delta t}(t,x,a)$ that takes current time-state pair $(t,x)$ and action $a$ as inputs and generates state $x'$ at time $t+\Delta t$.

\textbf{Learning procedure}:
\begin{algorithmic}
\STATE Initialize $\theta,\psi,w$.
\FOR{episode $j=1$ \TO $N$} \STATE{Initialize $k = 0$. Observe the initial state $x$ and store $x_{t_k} \leftarrow  x$.
\WHILE{$k < K$} \STATE{
Generate action $a_{t_k}\sim \bm{\pi}^{\psi}(\cdot|t_k,x_{t_k};w)$.

Compute and store the test function $\xi_{t_k} = \frac{\partial J^{\theta}}{\partial \theta}(t_k,x_{t_k};w)$, $\zeta_{t_k} = \frac{\partial q^{\psi}}{\partial \psi}(t_k,x_{t_k},a_{t_k};w)$.

Apply $a_{t_k}$ to the market simulator $x = Market_{\Delta t}(t_k, x_{t_k}, a_{t_k})$, and observe the output new state $x$. Store $x_{t_{k+1}}$.

Update $k \leftarrow k + 1$.
}
\ENDWHILE	

Store the terminal wealth $X_T^{(j)} \leftarrow x_{t_K}$.

Compute
\[ \Delta \theta = \sum_{i=0}^{K-1} \xi_{t_i} \big[ J^{\theta}(t_{i+1},x_{t_{i+1}};w) - J^{\theta}(t_{i},x_{t_{i}};w)  -q^{\psi}(t_i,x_{t_i},a_{t_i};w)\Delta t  \big], \]
\[ \Delta \psi = \sum_{i=0}^{K-1} \zeta_{t_i} \big[ J^{\theta}(t_{i+1},x_{t_{i+1}};w) - J^{\theta}(t_{i},x_{t_{i}};w)  -q^{\psi}(t_i,x_{t_i},a_{t_i};w)\Delta t  \big]. \]

Update $\theta$ and $\psi$ by
\[ \theta \leftarrow \theta + l(j)\alpha_{\theta} \Delta \theta .\]
\[ \psi \leftarrow \psi + l(j)\alpha_{\psi} \Delta \psi .  \]

\textbf{Update $w$ (Lagrange multiplier) every $m$ episodes:}
\IF{$j \equiv 0 \mod m$} \STATE{
\[ w\leftarrow w - \alpha_w \frac{1}{m}\sum_{i=j-m+1}^{j} X_T^{(i)}.\]
}
\ENDIF
}
\ENDFOR
\end{algorithmic}
\label{algo:offline episodic mv}
\end{algorithm}

We then compare by simulations the performances of three learning algorithms: the q-learning based Algorithm \ref{algo:offline episodic mv} in this paper, the PG-based  Algorithm 4 in \cite{jia2021policypg}, and the $\Delta t$-parameterized Q-learning algorithm presented in Appendix B.
We conduct simulations with the following configurations: $\mu\in\{0,\pm 0.1,\pm 0.3,\pm 0.5\}$, $\sigma\in \{0.1,0.2,0.3,0.4\}$, $T = 1$, $x_0 = 1$, $z=1.4$. Other tuning parameters in all the algorithms are chosen as $\gamma=0.1$, $m=10$, $\alpha_{\theta} = \alpha_{\psi} = 0.001$, $\alpha_w = 0.005$, and $l(j) = \frac{1}{j^{0.51}}$. To have more realistic scenarios, we generate 20 years of training data and compare the three algorithms with the same dataset for $N=20,000$ episodes with a batch size 32. More precisely, 32 trajectories with length $T$ are drawn from the training set to update the parameters to be learned for $N$ times. After training completes, for each market configuration we apply the learned policy out-of-sample repeatedly for 100 times and compute the average mean, variance and Sharpe ratio for each algorithm. We are particularly interested in the impact of time discretization; so we experiment with three different time discretization steps: $\Delta t\in\{\frac{1}{25}, \frac{1}{250}, \frac{1}{2500}\}$. Note that all the algorithms rely on time discretization at the implementation stage, where $\Delta t$ is the sampling frequency required for execution or computation of numerical integral for PG and q-learning. For the Q-learning, $\Delta t$ plays a dual role: it is both the parameter in its definition and the time discretization size in implementation.

The numerical results are reported in Tables \ref{tab:mv 1-25} -- \ref{tab:mv 1-2500}, each table corresponding to a different value of $\Delta t$ . 
We observe that for any market configuration, the Q-learning is almost always  the worst performer in terms of all the metrics, and even divergent in certain high volatility and low return environment (when keeping the same learning rate as in the other cases). 
The other two algorithms have very close overall performances, although the q-learning tends to outperform in high volatile market environments. On the other hand, notably, the results change only slightly with different time discretizations, for {\it all} the three algorithms. This observation does not align with the claim in \cite{tallec2019making} that the classical Q-learning is very sensitive to time-discretization, at least for this specific application problem. In addition, we notice that q-learning and PG-based algorithm produce comparable results under most market scenarios, while the terminal variance produced by q-learning tends to be significantly higher than that by PG when $|\mu|$ is small or $\sigma$ is large.
In such cases, the ground truth solution requires higher leverage in risky assets in order to reach the high target expected return (40\% annual return).
However, it appears that overall q-learning strives to learn to keep up with the high target return and as a result ends up with high terminal variance in those cases.
By contrast, the PG tends to maintain lower variance at the cost of significantly underachieving the target return.

\begin{table}[htbp]
\centering
\caption{\textbf{Out-of-sample performance comparison in terms of mean, variance and Sharpe ratio when data are generated by geometric Brownian motion with $\Delta t = \frac{1}{25}$.} We compare three algorithms: ``Policy Gradient" described in
Algorithm 4  in \cite{jia2021policypg}, ``Q-Learning" described  in Appendix B, and ``q-Learning" described in Algorithm \ref{algo:offline episodic mv}.
The first two columns specify market configurations (parameters of the geometric Brownian motion simulated). ``Mean'', ``Variance'', and ``SR'' refer respectively to the average mean, variance, and  Sharpe ratio of the terminal wealth under  the corresponding learned policy over 100 independent runs.  We highlight the largest average Sharpe ratio in \textbf{bold}. ``NA" refers to the case where the algorithm diverges. The discretization size is $\Delta t = \frac{1}{25}$. }
\begin{tabular}{ccccccccccc}
\toprule
\multirow{2}{*}{$\mu$}    & \multirow{2}{*}{$\sigma$}  & \multicolumn{3}{c}{Policy Gradient} & \multicolumn{3}{c}{Q-Learning}   & \multicolumn{3}{c}{q-Learning}   \\
\cline{3-11}
&  & Mean    & Variance      &  SR     &  Mean    & Variance      &   SR     & Mean    &   Variance    & SR  \\
\midrule
-0.5  & 0.1   & 1.409 & 0.003 & \textbf{7.134} & 1.409 & 0.004 & 6.663 & 1.408 & 0.004 & 6.805 \\
-0.3  & 0.1   & 1.421 & 0.012 & \textbf{3.880} & 1.409 & 0.012 & 3.707 & 1.412 & 0.012 & 3.762 \\
-0.1  & 0.1   & 1.382 & 0.094 & \textbf{1.285} & 1.327 & 0.070 & 1.263 & 1.351 & 0.081 & 1.272 \\
0.0   & 0.1   & 1.090 & 0.216 & 0.193 & 1.098 & 0.302 & 0.201 & 1.108 & 0.321 & \textbf{0.202} \\
0.1   & 0.1   & 1.312 & 0.147 & \textbf{0.835} & 1.265 & 0.108 & 0.827 & 1.292 & 0.133 & 0.831 \\
0.3   & 0.1   & 1.420 & 0.016 & \textbf{3.312} & 1.402 & 0.016 & 3.181 & 1.407 & 0.016 & 3.224 \\
0.5   & 0.1   & 1.405 & 0.004 & \textbf{6.422} & 1.403 & 0.005 & 6.026 & 1.403 & 0.004 & 6.147 \\
-0.5  & 0.2   & 1.417 & 0.014 & \textbf{3.526} & 1.416 & 0.016 & 3.308 & 1.416 & 0.016 & 3.377 \\
-0.3  & 0.2   & 1.445 & 0.060 & \textbf{1.915} & 1.428 & 0.059 & 1.841 & 1.434 & 0.059 & 1.867 \\
-0.1  & 0.2   & 1.335 & 0.341 & 0.608 & 1.402 & 0.801 & 0.628 & 1.405 & 0.589 & \textbf{0.631} \\
0.0   & 0.2   & 1.053 & 0.539 & 0.070 & 1.168 & 7.995 & 0.100 & 1.136 & 3.422 & \textbf{0.100} \\
0.1   & 0.2   & 1.248 & 0.437 & 0.380 & 1.349 & 2.046 & 0.411 & 1.341 & 1.005 & \textbf{0.412} \\
0.3   & 0.2   & 1.445 & 0.083 & \textbf{1.634} & 1.422 & 0.078 & 1.580 & 1.431 & 0.081 & 1.600 \\
0.5   & 0.2   & 1.413 & 0.018 & \textbf{3.173} & 1.412 & 0.020 & 2.992 & 1.411 & 0.019 & 3.050 \\
-0.5  & 0.3   & 1.433 & 0.039 & \textbf{2.305} & 1.432 & 0.043 & 2.180 & 1.432 & 0.041 & 2.223 \\
-0.3  & 0.3   & 1.458 & 0.162 & \textbf{1.247} & 1.468 & 0.196 & 1.214 & 1.484 & 0.217 & 1.229 \\
-0.1  & 0.3   & 1.226 & 0.569 & 0.335 & 1.624 & 8.959 & 0.415 & 1.556 & 5.558 & \textbf{0.415} \\
0.0   & 0.3   & 1.030 & 0.777 & 0.036 & NA    & NA    & NA    & 1.187 & 23.269 & \textbf{0.066} \\
0.1   & 0.3   & 1.151 & 0.670 & 0.202 & 1.539 & 13.741 & 0.271 & 1.482 & 13.321 & \textbf{0.271} \\
0.3   & 0.3   & 1.448 & 0.231 & 1.046 & 1.475 & 0.322 & 1.042 & 1.485 & 0.313 & \textbf{1.053} \\
0.5   & 0.3   & 1.431 & 0.048 & \textbf{2.073} & 1.428 & 0.052 & 1.972 & 1.429 & 0.050 & 2.008 \\
-0.5  & 0.4   & 1.457 & 0.093 & \textbf{1.679} & 1.460 & 0.102 & 1.608 & 1.462 & 0.102 & 1.638 \\
-0.3  & 0.4   & 1.417 & 0.346 & 0.877 & 1.670 & 2.363 & 0.898 & 1.602 & 1.176 & \textbf{0.905} \\
-0.1  & 0.4   & 1.147 & 0.805 & 0.205 & NA    & NA    & NA    & 1.711 & 20.852 & \textbf{0.306} \\
0.0   & 0.4   & 1.017 & 0.893 & 0.016 & NA    & NA    & NA    & 1.211 & 76.040 & \textbf{0.048} \\
0.1   & 0.4   & 1.084 & 0.825 & 0.109 & NA    & NA    & NA    & 1.539 & 25.371 & \textbf{0.196} \\
0.3   & 0.4   & 1.408 & 0.379 & 0.747 & NA    & NA    & NA    & 1.609 & 1.992 & \textbf{0.776} \\
0.5   & 0.4   & 1.446 & 0.102 & \textbf{1.510} & 1.456 & 0.125 & 1.455 & 1.461 & 0.128 & 1.479 \\
\bottomrule
\end{tabular}%
\label{tab:mv 1-25}%
\end{table}%

\begin{table}[htbp]
\centering
\caption{\textbf{Out-of-sample performance comparison in terms of mean, variance and Sharpe ratio when data are generated by geometric Brownian motion with $\Delta t = \frac{1}{250}$.} We compare three algorithms: ``Policy Gradient" described in
Algorithm 4  in \cite{jia2021policypg}, ``Q-Learning" described  in Appendix B, and ``q-Learning" described in Algorithm \ref{algo:offline episodic mv}.
The first two columns specify market configurations (parameters of the geometric Brownian motion simulated). ``Mean'', ``Variance'', and ``SR'' refer respectively to the average mean, variance, and  Sharpe ratio of the terminal wealth under  the corresponding learned policy over 100 independent runs.  We highlight the largest average Sharpe ratio in \textbf{bold}. ``NA" refers to the case where the algorithm diverges. The discretization size is $\Delta t = \frac{1}{250}$. }
\begin{tabular}{ccccccccccc}
\toprule
\multirow{2}{*}{$\mu$}    & \multirow{2}{*}{$\sigma$}  & \multicolumn{3}{c}{Policy Gradient} & \multicolumn{3}{c}{Q-Learning}   & \multicolumn{3}{c}{q-Learning}   \\
\cline{3-11}
&  & Mean    & Variance      &  SR     &  Mean    & Variance      &   SR     & Mean    &   Variance    & SR  \\
\midrule
-0.5  & 0.1   & 1.409 & 0.003 & \textbf{7.130} & 1.407 & 0.004 & 6.671 & 1.407 & 0.004 & 6.805 \\
-0.3  & 0.1   & 1.419 & 0.012 & \textbf{3.878} & 1.407 & 0.012 & 3.710 & 1.410 & 0.012 & 3.762 \\
-0.1  & 0.1   & 1.380 & 0.092 & \textbf{1.285} & 1.325 & 0.069 & 1.264 & 1.347 & 0.079 & 1.272 \\
0     & 0.1   & 1.092 & 0.218 & \textbf{0.201} & 1.096 & 0.253 & 0.201 & 1.106 & 0.306 & 0.202 \\
0.1   & 0.1   & 1.314 & 0.148 & \textbf{0.835} & 1.267 & 0.109 & 0.827 & 1.293 & 0.133 & 0.831 \\
0.3   & 0.1   & 1.424 & 0.017 & \textbf{3.311} & 1.406 & 0.017 & 3.183 & 1.411 & 0.017 & 3.224 \\
0.5   & 0.1   & 1.411 & 0.004 & \textbf{6.421} & 1.409 & 0.005 & 6.032 & 1.409 & 0.004 & 6.147 \\
-0.5  & 0.2   & 1.415 & 0.014 & \textbf{3.524} & 1.414 & 0.016 & 3.312 & 1.413 & 0.015 & 3.377 \\
-0.3  & 0.2   & 1.439 & 0.057 & \textbf{1.915} & 1.423 & 0.056 & 1.843 & 1.428 & 0.056 & 1.867 \\
-0.1  & 0.2   & 1.343 & 0.329 & 0.620 & 1.369 & 0.460 & 0.628 & 1.390 & 0.493 & \textbf{0.631} \\
0     & 0.2   & 1.054 & 0.546 & 0.070 & 1.173 & 9.036 & 0.100 & 1.128 & 2.519 & \textbf{0.100} \\
0.1   & 0.2   & 1.241 & 0.413 & 0.372 & 1.350 & 1.356 & 0.411 & 1.340 & 0.925 & \textbf{0.412} \\
0.3   & 0.2   & 1.450 & 0.084 & \textbf{1.634} & 1.426 & 0.079 & 1.581 & 1.435 & 0.081 & 1.600 \\
0.5   & 0.2   & 1.420 & 0.018 & \textbf{3.173} & 1.418 & 0.020 & 2.995 & 1.418 & 0.019 & 3.050 \\
-0.5  & 0.3   & 1.428 & 0.037 & \textbf{2.304} & 1.426 & 0.041 & 2.183 & 1.426 & 0.039 & 2.223 \\
-0.3  & 0.3   & 1.467 & 0.171 & \textbf{1.248} & 1.454 & 0.169 & 1.215 & 1.466 & 0.178 & 1.229 \\
-0.1  & 0.3   & 1.230 & 0.552 & 0.326 & 1.661 & 10.321 & 0.415 & 1.495 & 2.852 & \textbf{0.415} \\
0     & 0.3   & 1.031 & 0.845 & 0.036 & 1.236 & 38.868 & 0.066 & 1.184 & 21.710 & \textbf{0.066} \\
0.1   & 0.3   & 1.142 & 0.716 & 0.181 & 1.541 & 16.335 & 0.266 & 1.456 & 6.433 & \textbf{0.271} \\
0.3   & 0.3   & 1.466 & 0.230 & \textbf{1.064} & 1.470 & 0.260 & 1.043 & 1.486 & 0.281 & 1.053 \\
0.5   & 0.3   & 1.438 & 0.049 & \textbf{2.073} & 1.434 & 0.053 & 1.974 & 1.435 & 0.051 & 2.008 \\
-0.5  & 0.4   & 1.451 & 0.085 & \textbf{1.679} & 1.448 & 0.090 & 1.611 & 1.449 & 0.087 & 1.638 \\
-0.3  & 0.4   & 1.424 & 0.281 & 0.890 & 1.535 & 0.638 & 0.898 & 1.544 & 0.612 & \textbf{0.905} \\
-0.1  & 0.4   & 1.153 & 0.777 & 0.192 & NA    & NA    & NA    & 1.747 & 22.299 & \textbf{0.306} \\
0     & 0.4   & 1.019 & 1.000 & 0.021 & NA    & NA    & NA    & 1.233 & 77.875 & \textbf{0.048} \\
0.1   & 0.4   & 1.091 & 1.006 & 0.110 & NA    & NA    & NA    & 1.696 & 48.274 & \textbf{0.200} \\
0.3   & 0.4   & 1.400 & 0.384 & 0.704 & 1.685 & 2.737 & 0.771 & 1.595 & 1.189 & \textbf{0.776} \\
0.5   & 0.4   & 1.463 & 0.114 & \textbf{1.510} & 1.463 & 0.122 & 1.457 & 1.466 & 0.121 & 1.479 \\
\bottomrule
\end{tabular}%
\label{tab:mv 1-250}%
\end{table}%

\begin{table}[htbp]
\centering
\caption{\textbf{Out-of-sample performance comparison in terms of mean, variance and Sharpe ratio when data are generated by geometric Brownian motion with $\Delta t = \frac{1}{2500}$.} We compare three algorithms: ``Policy Gradient" described in
Algorithm 4  in \cite{jia2021policypg}, ``Q-Learning" described  in Appendix B, and ``q-Learning" described in Algorithm \ref{algo:offline episodic mv}.
The first two columns specify market configurations (parameters of the geometric Brownian motion simulated). ``Mean'', ``Variance'', and ``SR'' refer respectively to the average mean, variance, and Sharpe ratio of the terminal wealth under  the corresponding learned policy over 100 independent runs.  We highlight the largest average Sharpe ratio in \textbf{bold}. ``NA" refers to the case where the algorithm diverges. The discretization size is $\Delta t = \frac{1}{2500}$.}
\begin{tabular}{ccccccccccc}
\toprule
\multirow{2}{*}{$\mu$}    & \multirow{2}{*}{$\sigma$}  & \multicolumn{3}{c}{Policy Gradient} & \multicolumn{3}{c}{Q-Learning}   & \multicolumn{3}{c}{q-Learning}   \\
\cline{3-11}
&  & Mean    & Variance      &  SR     &  Mean    & Variance      &   SR     & Mean    &   Variance    & SR  \\
\midrule
-0.5  & 0.1   & 1.408 & 0.003 & \textbf{7.130} & 1.406 & 0.004 & 6.677 & 1.406 & 0.004 & 6.805 \\
-0.3  & 0.1   & 1.418 & 0.012 & \textbf{3.878} & 1.406 & 0.012 & 3.713 & 1.409 & 0.012 & 3.762 \\
-0.1  & 0.1   & 1.376 & 0.090 & \textbf{1.285} & 1.324 & 0.069 & 1.264 & 1.345 & 0.077 & 1.272 \\
0.0   & 0.1   & 1.092 & 0.218 & 0.201 & 1.096 & 0.259 & 0.201 & 1.105 & 0.300 & \textbf{0.202} \\
0.1   & 0.1   & 1.316 & 0.149 & \textbf{0.835} & 1.269 & 0.111 & 0.827 & 1.294 & 0.133 & 0.831 \\
0.3   & 0.1   & 1.425 & 0.017 & \textbf{3.311} & 1.407 & 0.017 & 3.185 & 1.412 & 0.017 & 3.224 \\
0.5   & 0.1   & 1.412 & 0.004 & \textbf{6.420} & 1.410 & 0.005 & 6.037 & 1.410 & 0.004 & 6.147 \\
-0.5  & 0.2   & 1.413 & 0.014 & \textbf{3.524} & 1.411 & 0.016 & 3.315 & 1.411 & 0.015 & 3.377 \\
-0.3  & 0.2   & 1.435 & 0.056 & \textbf{1.915} & 1.419 & 0.055 & 1.844 & 1.424 & 0.056 & 1.867 \\
-0.1  & 0.2   & 1.328 & 0.295 & \textbf{0.632} & 1.366 & 0.455 & 0.629 & 1.384 & 0.486 & 0.631 \\
0.0   & 0.2   & 1.059 & 0.610 & 0.074 & 1.174 & 9.894 & 0.100 & 1.125 & 2.521 & \textbf{0.100} \\
0.1   & 0.2   & 1.254 & 0.468 & 0.380 & 1.360 & 1.867 & 0.411 & 1.340 & 0.952 & \textbf{0.412} \\
0.3   & 0.2   & 1.450 & 0.084 & \textbf{1.634} & 1.428 & 0.080 & 1.582 & 1.436 & 0.081 & 1.600 \\
0.5   & 0.2   & 1.422 & 0.018 & \textbf{3.173} & 1.419 & 0.020 & 2.998 & 1.419 & 0.019 & 3.050 \\
-0.5  & 0.3   & 1.425 & 0.036 & \textbf{2.304} & 1.422 & 0.040 & 2.185 & 1.422 & 0.039 & 2.223 \\
-0.3  & 0.3   & 1.459 & 0.169 & \textbf{1.248} & 1.449 & 0.167 & 1.216 & 1.459 & 0.177 & 1.229 \\
-0.1  & 0.3   & 1.225 & 0.597 & 0.327 & 1.690 & 11.223 & 0.415 & 1.487 & 2.943 & \textbf{0.415} \\
0.0   & 0.3   & 1.039 & 0.857 & 0.044 & 1.206 & 36.278 & 0.066 & 1.183 & 23.798 & \textbf{0.066} \\
0.1   & 0.3   & 1.169 & 0.745 & 0.202 & 1.457 & 14.390 & 0.266 & 1.462 & 8.590 & \textbf{0.271} \\
0.3   & 0.3   & 1.455 & 0.215 & \textbf{1.064} & 1.488 & 0.356 & 1.044 & 1.486 & 0.296 & 1.053 \\
0.5   & 0.3   & 1.438 & 0.049 & \textbf{2.073} & 1.435 & 0.053 & 1.975 & 1.436 & 0.051 & 2.008 \\
-0.5  & 0.4   & 1.446 & 0.084 & \textbf{1.679} & 1.443 & 0.089 & 1.612 & 1.444 & 0.086 & 1.638 \\
-0.3  & 0.4   & 1.395 & 0.231 & 0.873 & 1.529 & 0.619 & 0.899 & 1.537 & 0.622 & \textbf{0.905} \\
-0.1  & 0.4   & 1.175 & 0.846 & 0.229 & NA    & NA    & NA    & 1.754 & 24.745 & \textbf{0.306} \\
0.0   & 0.4   & 1.026 & 0.967 & 0.025 & NA    & NA    & NA    & 1.170 & 1.66E+09 & \textbf{0.048} \\
0.1   & 0.4   & 1.110 & 0.973 & 0.129 & NA    & NA    & NA    & 1.661 & 59.872 & \textbf{0.200} \\
0.3   & 0.4   & 1.402 & 0.416 & 0.720 & 1.646 & 3.402 & 0.771 & 1.607 & 1.743 & \textbf{0.776} \\
0.5   & 0.4   & 1.456 & 0.108 & \textbf{1.510} & 1.469 & 0.135 & 1.458 & 1.467 & 0.128 & 1.479 \\
\bottomrule
\end{tabular}%
\label{tab:mv 1-2500}%
\end{table}%

\subsection{Ergodic linear--quadratic control}
\label{sec:application lq ergodic}	
Consider the ergodic LQ control problem where state responds to actions in a linear way
\[ \dd X_t = (AX_t + Ba_t)\dd t + (CX_t + D a_t)\dd W_t,\ X_0 = x_0, \]
and the goal is  to maximize the long term average quadratic payoff
\[ \liminf_{T\to \infty}\frac{1}{T}\E\left[\int_0^T r(X_t,a_t)\dd t | X_0 = x_0 \right], \]
with $r(x,a) = -(\frac{M}{2}x^2 + Rxa + \frac{N}{2}a^2 + Px + Qa)$.

The exploratory formulation of this problem with entropy regularizer is equivalent to
\[ \liminf_{T\to \infty}\frac{1}{T}\E\bigg[ \int_0^T  r({X}_t^{\bm{\pi}},a_t^{\bm\pi}) \dd t - \gamma \log\bm{\pi}(a_t^{\bm\pi}|{X}_t^{\bm{\pi}}) \dd t \Big| {X}_0^{\bm{\pi}} = x_0\bigg] .  \]

\begin{figure}[h]
\centering
\includegraphics[width=0.75\textwidth]{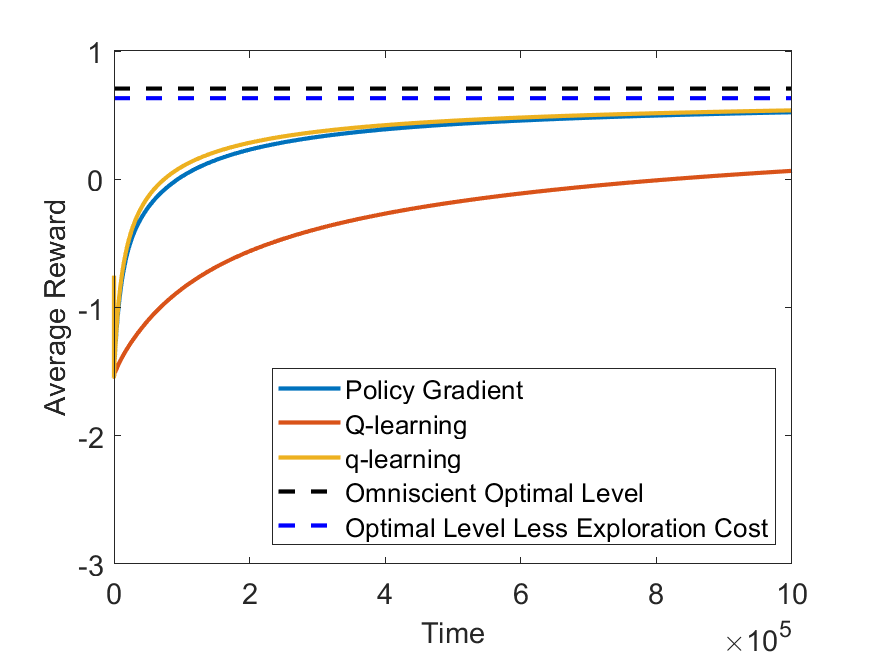}
\caption{\textbf{Running average rewards of three RL algorithms.} A single state trajectory is generated with length $T = 10^6$ and discretized at $\Delta t=0.1$ to which three  online algorithms apply: ``Policy Gradient" described in
Algorithm 3  in \cite{jia2021policypg}, ``Q-Learning" described  in Appendix B, and ``q-Learning" described in Algorithm \ref{algo:ergodic incremental}.
We repeat the experiments for 100 times for each method and plot the average reward over time with the shaded area indicating standard deviation. Two dashed horizontal lines are respectively the omniscient optimal average reward without exploration when the model parameters are known and the omniscient optimal average reward less the exploration cost. }
\label{fig:ergodic lq}
\end{figure}

\begin{figure}[h]
\centering
\begin{subfigure}{0.32\textwidth}
\centering
\includegraphics[width = 1\textwidth]{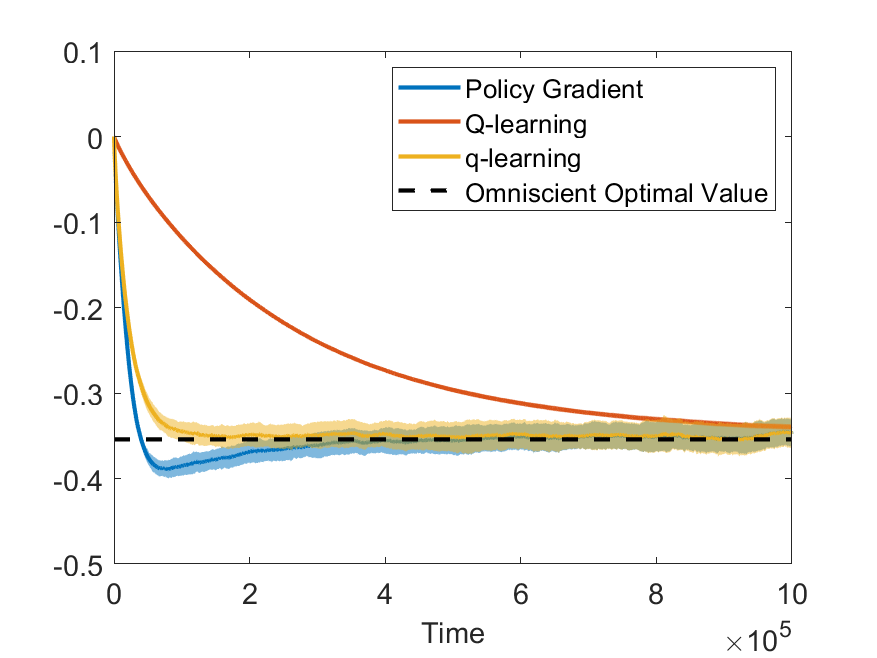}
\caption{The path of learned $\psi_1$.}
\end{subfigure}
\begin{subfigure}{0.32\textwidth}
\centering
\includegraphics[width = 1\textwidth]{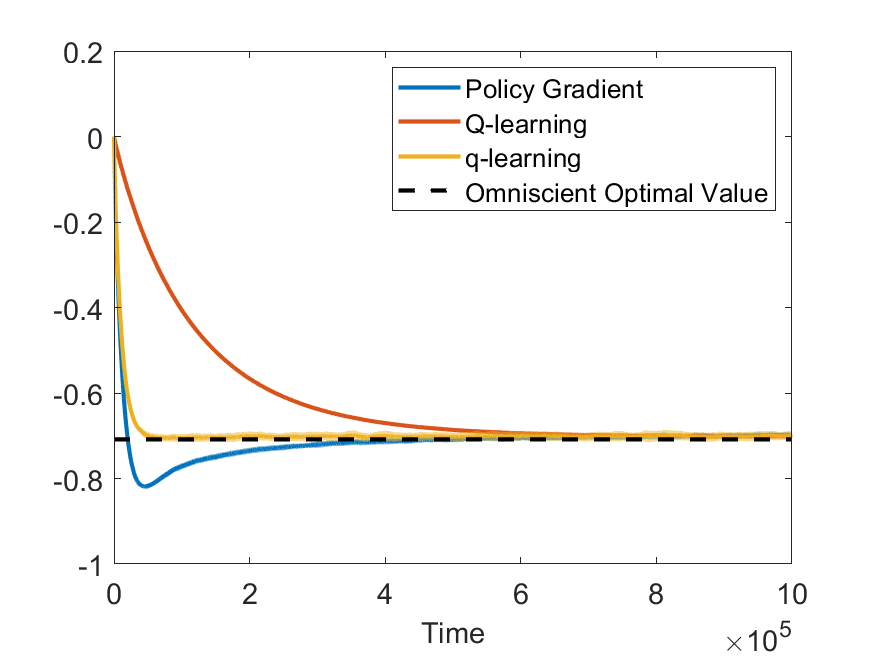}
\caption{The path of learned $\psi_2$.}
\end{subfigure}
\begin{subfigure}{0.32\textwidth}
\centering
\includegraphics[width = 1\textwidth]{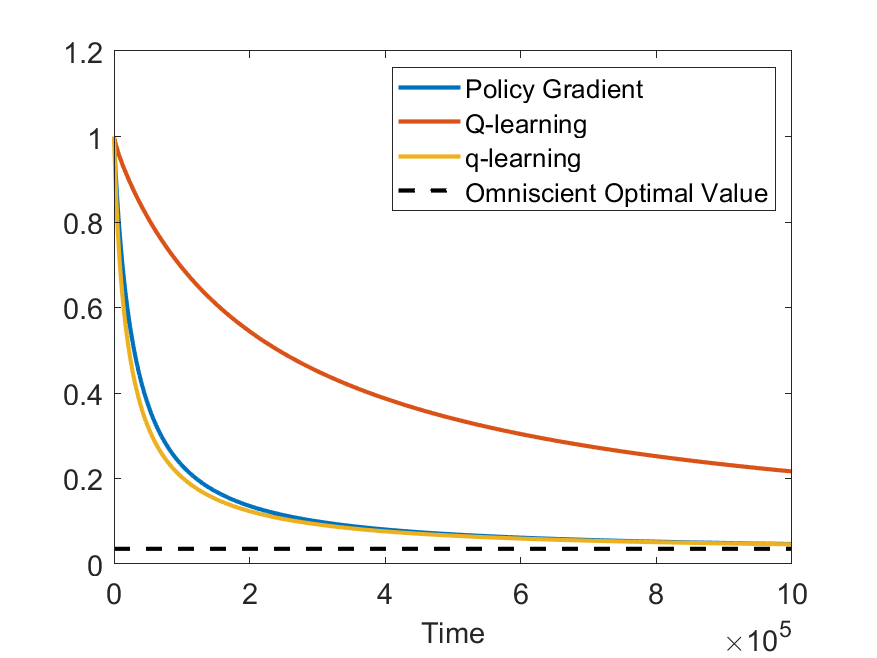}
\caption{The path of learned $e^{\psi_3}$.}
\end{subfigure}
\caption{\textbf{Paths of learned parameters of three RL algorithms.} A single state trajectory is generated with length $T = 10^6$ and discretized at $\Delta t=0.1$ to which three  online algorithms apply: ``Policy Gradient" described in
Algorithm 3  in \cite{jia2021policypg}, ``Q-Learning" described  in Appendix B, and ``q-Learning" described in Algorithm \ref{algo:ergodic incremental}. All the policies are restricted to be in the parametric form of $\bm\pi^{\psi}(\cdot|x) = \mathcal{N}(\psi_1 x + \psi_2, e^{\psi_3})$. The omniscient optimal policy is $\psi_1^*\approx-0.354$, $\psi_2^*\approx-0.708,e^{\psi_3^*}\approx 0.035$, shown in the dashed line. We repeat the experiments for 100 times for each method and plot as the shaded area the standard deviation of the learned parameters. The width of each shaded area is twice the corresponding standard deviation.}

\label{fig:lq parameter}
\end{figure}

\begin{figure}[h]
\centering
\begin{subfigure}{0.32\textwidth}
\centering
\includegraphics[width = 1\textwidth]{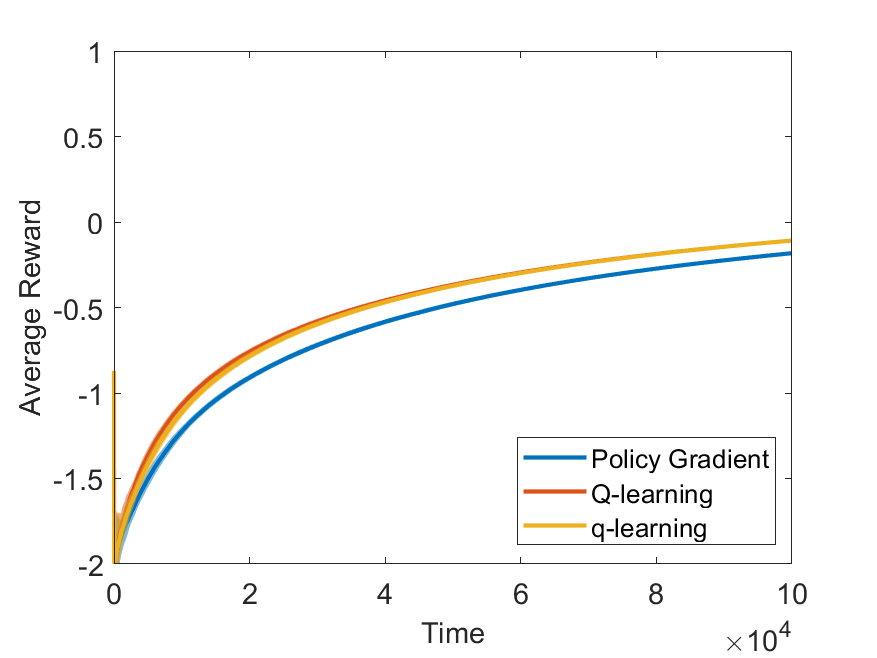}
\caption{$\Delta t= 1$}
\end{subfigure}
\begin{subfigure}{0.32\textwidth}
\centering
\includegraphics[width = 1\textwidth]{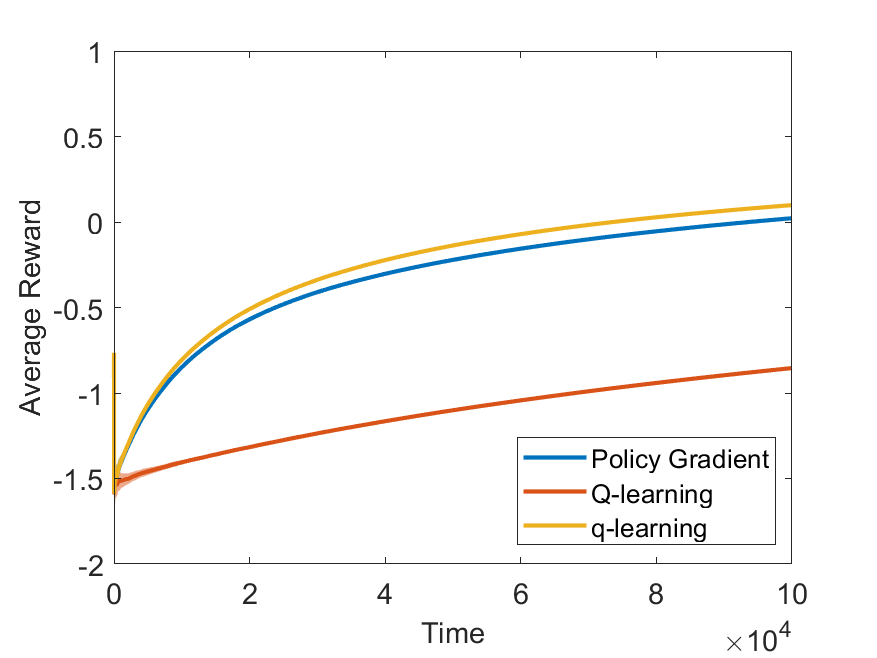}
\caption{$\Delta t= 0.1$}
\end{subfigure}
\begin{subfigure}{0.32\textwidth}
\centering
\includegraphics[width = 1\textwidth]{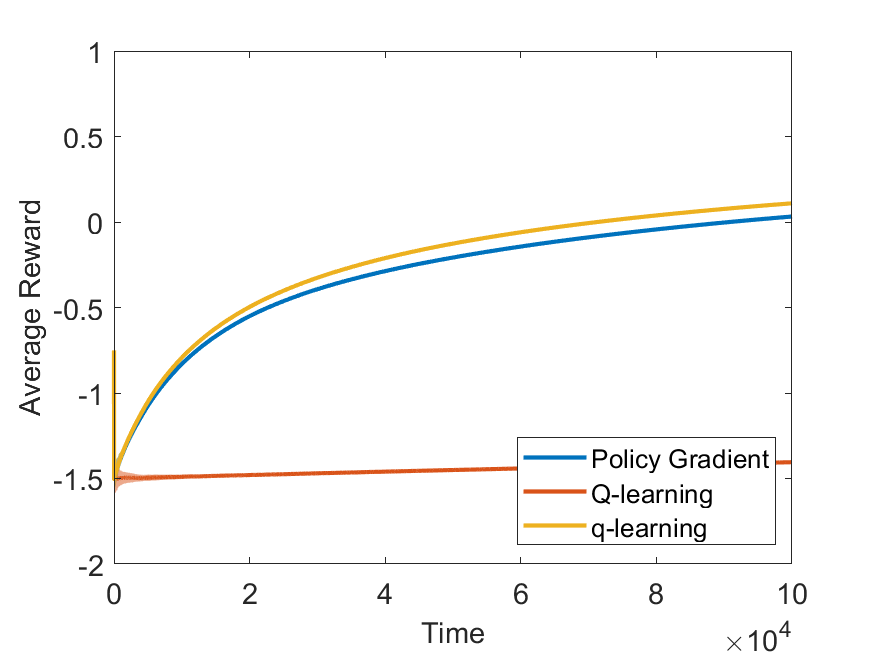}
\caption{$\Delta t= 0.01$}
\end{subfigure}
\caption{\textbf{Running average rewards of three RL algorithms with different time discretization sizes.} A single state trajectory is generated with length $T = 10^5$ and discretized at different step sizes: $\Delta t=1$ in (a), $\Delta t=0.1$ in (b), and $\Delta t=0.01$ in (c). For each step size, we apply  three  online algorithms: ``Policy Gradient" described in
Algorithm 3  in \cite{jia2021policypg}, ``Q-Learning" described  in Appendix B, and ``q-Learning" described in Algorithm \ref{algo:ergodic incremental}.
We repeat the experiments for 100 times for each method and plot the average reward over time with the shaded area indicating standard deviation. }
\label{fig:lq time scales}
\end{figure}

This problem falls into the general formulation of an ergodic task studied in Section \ref{sec:extension}; so we directly implement Algorithm \ref{algo:ergodic incremental} in our simulation. In particular, our function approximators for $J$ and $q$ are quadratic functions:
\[J^{\theta}(x) = \theta_1 x^2 + \theta_2 x,\;\;\;q^{\psi}(x,a) = -\frac{e^{-\psi_3}}{2}(a - \psi_1 x - \psi_2)^2 - \frac{\gamma}{2}(\log2\pi\gamma + \psi_3 ) ,\]
where the form of the q-function is derived to satisfy the constraint \eqref{eq:constraints qv}, as discussed in Example \ref{eg:linear quadratic}. The policy associated with this parameterized q-function is $\bm\pi^{\psi}(\cdot|x) = \mathcal{N}(\psi_1 x + \psi_2, \gamma e^{\psi_3})$. In addition, we have another parameter $V$ that stands for the long-term average.

We then compare our online q-learning algorithm against two theoretical benchmarks and two other online algorithms, in terms of the running average reward during the learning process. The first benchmark is the omniscient optimal level, namely, the maximum long term average reward that can be achieved with perfect knowledge about the environment and reward (and hence exploration is unnecessary and only deterministic policies are considered). The second benchmark is the omniscient optimal level less the exploration cost, which is the maximum long term average reward that can be achieved by the agent who  is however forced to explore under entropy regularization. The other two algorithms are the PG-based ergodic algorithm proposed in \cite[Algorithm 3]{jia2021policypg} and the $\Delta t$-based  Q-learning algorithm presented in Appendix B. 

In our simulation, to ensure the stationarity of the controlled state process, we set $A = -1,B = C= 0$ and $D = 1$. 
Moreover, we set $x_0=0$,  $M = N=Q =2$, $R=P =1$, and  $\gamma = 0.1$. Learning rates for all algorithms are initialized as $\alpha_{\psi} =0.001$, and decay according to $l(t) = \frac{1}{\max\{1,\sqrt{\log t} \}}$. We also repeat the experiment for 100 times under different seeds to generate samples.

First, we fix $\Delta t = 0.1$ and run the three learning algorithms for sufficiently long time $T=10^6$. All the parameters to be learned are initialized as 0 except for the initial variance of the policies which is set to be 1. Figure \ref{fig:ergodic lq} plots  the running average reward trajectories, along with their standard deviations (which are all too small to be visible in the figure), as the learning process proceeds. Among the three methods, both the PG and q-learning algorithms perform significantly better and converge to the optimal level much faster than the Q-learning.   The q-learning algorithm is indeed slightly better than the PG one. Figure \ref{fig:lq parameter} shows the iterated values of the leaned parameters for the three algorithms. Eminently, the Q-learning has the slowest  convergence. The other two algorithms, while converging at similar speed eventually, exhibit quite different behaviors on their ways to  convergence in learning $\psi_1$ and $\psi_2$.  The PG seems to learn these parameters faster than the q-learning at the initial stage,  but then quickly overshoot the optimal level and take a long time to correct it, while the q-learning appears much smoother and does not have such an issue. 

At last, we vary the $\Delta t$ value to examine its impact on the performances of the three algorithms. We vary $\Delta t\in\{1,0.1,0.01\}$ and set $T=10^5$ for each experiment. The learning rates and parameter initializations are fixed and the same for all the there methods. We repeat the experiment for 100 times and present the results in terms of the average running reward in Figure \ref{fig:lq time scales}. We use the shaded area to denote the standard deviation. It is clear that the  Q-learning algorithm is very  sensitive to $\Delta t$. In particular, its performance worsens significantly as $\Delta t$  becomes smaller. When $\Delta t= 0.01$,
the algorithm almost has no improvement at all over time. This drawback of sensitivity in $\Delta t$  is consistent with the observations made in \cite{tallec2019making}. On the other hand, the other two algorithms show remarkable robustness across different discretization sizes.

\subsection{Off-policy ergodic linear--quadratic control}
\label{sec:application lq ergodic offpolicy}	
The experiments reported  in Subsection \ref{sec:application lq ergodic} are for  on-policy learning.
In this subsection we revisit the problem but assuming that we now have to work off-policy. Specifically,  we have a sufficiently long observation of the trajectories of state, action and reward generated under a behavior policy that is not optimal. In our experiment, we take the behavior policy to be $a_t\sim \mathcal{N}(0, 1)$.

We still examine the three algorithms: ``Policy Gradient" described in
Algorithm 3  in \cite{jia2021policypg}, ``Q-Learning" described  in Appendix B, and ``q-Learning" described in Algorithm \ref{algo:ergodic incremental}, with different time discretization sizes. Because  off-policy learning  is often necessitated  in cases where online interactions are costly \citep{uehara2022review}, we restrict ourselves to an {\it offline} dataset generated under the behavior policy. We vary $\Delta t\in\{1,0.1,0.01\}$ and set $T=10^6$ for each experiment. We repeat the experiments for 100 times whose results are presented in Figure \ref{fig:lq off policy}. Because it is generally hard or impossible to implement the learned policy and observe the resulting state and reward data  in an off-policy setting, we  focus on the learned {\it parameters} of the policy to see if they converge to the desired optimal value. The horizontal axis in Figure \ref{fig:lq off policy} stands for the number of iterations that are used to update the policy parameters.

In general, policy gradient methods are not applicable in the off-policy setting. This  is confirmed by our experiments which show that the corresponding algorithm diverges quickly in all the cases. Q-learning still suffers from the sensitivity with respect to time discretization. The q-learning algorithm is the most stable one and converges in all scenarios. It is worth noting that in Figure \ref{fig:lq off policy}-(a),  the Q-learning and q-learning algorithms converge to the same limits, which are however different from  the true optimal parameter values of the continuous-time problem. This is because the theoretical integrals involved in the value functions are not well approximated by finite sums with coarse time discretization ($\Delta t=1$); hence the corresponding martingale conditions under the learned parameters are not  close to the theoretical continuous-time martingale conditions.

\begin{figure}[h]
\centering
\begin{subfigure}{0.32\textwidth}
\centering
\includegraphics[width = 1\textwidth]{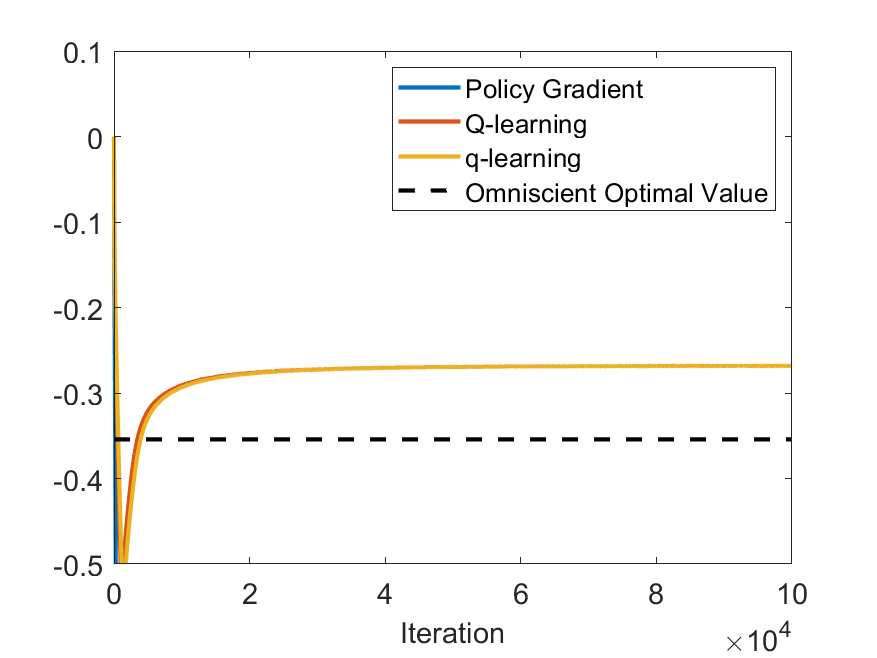}
\end{subfigure}
\begin{subfigure}{0.32\textwidth}
\centering
\includegraphics[width = 1\textwidth]{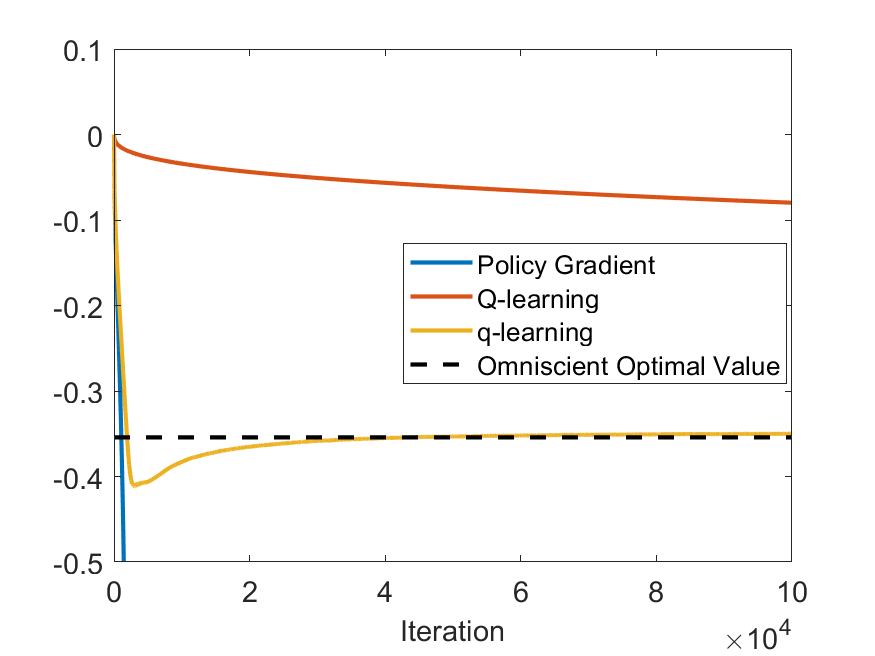}
\end{subfigure}
\begin{subfigure}{0.32\textwidth}
\centering
\includegraphics[width = 1\textwidth]{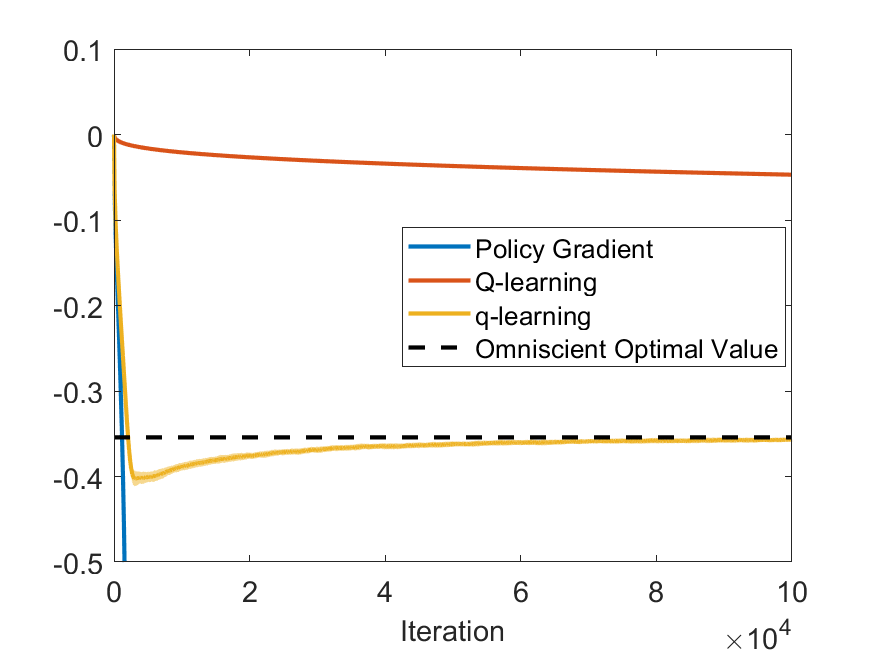}
\end{subfigure}
\begin{subfigure}{0.32\textwidth}
\centering
\includegraphics[width = 1\textwidth]{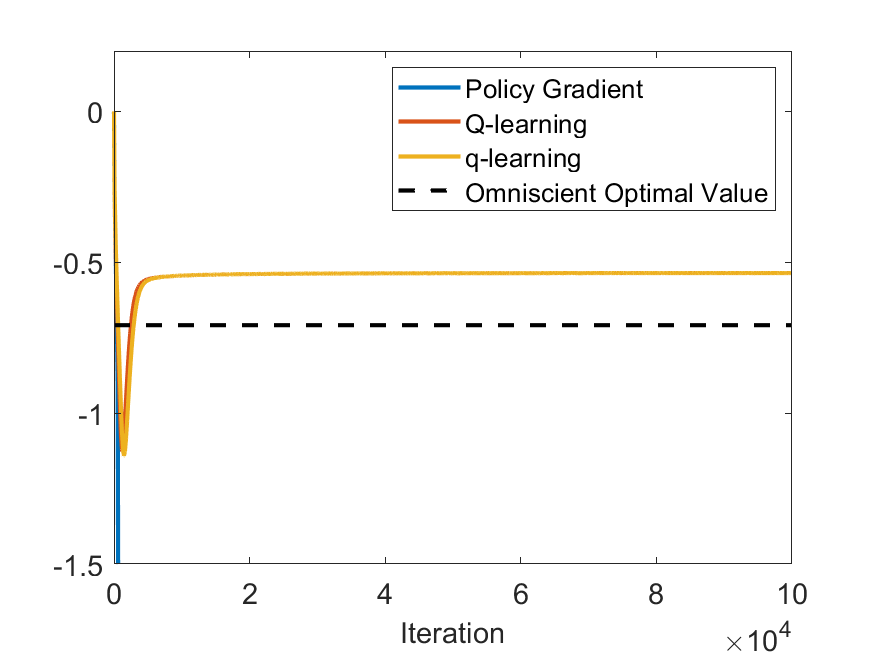}
\end{subfigure}
\begin{subfigure}{0.32\textwidth}
\centering
\includegraphics[width = 1\textwidth]{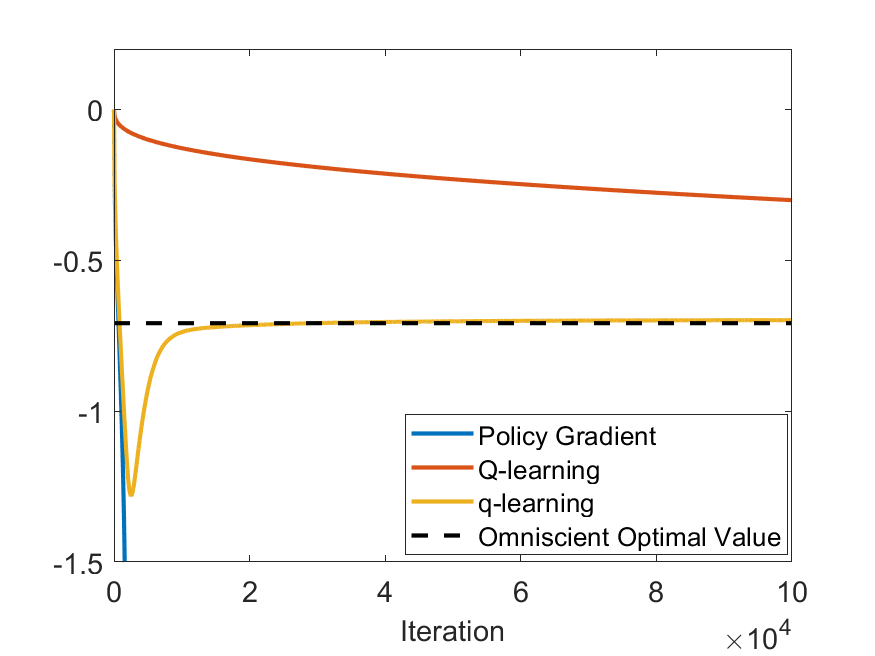}
\end{subfigure}
\begin{subfigure}{0.32\textwidth}
\centering
\includegraphics[width = 1\textwidth]{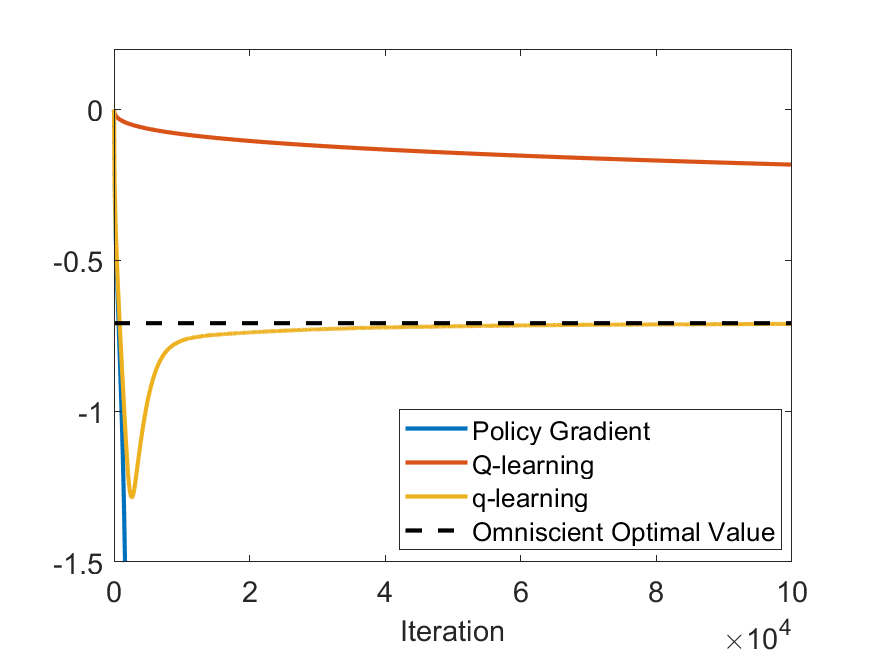}
\end{subfigure}
\begin{subfigure}{0.32\textwidth}
\centering
\includegraphics[width = 1\textwidth]{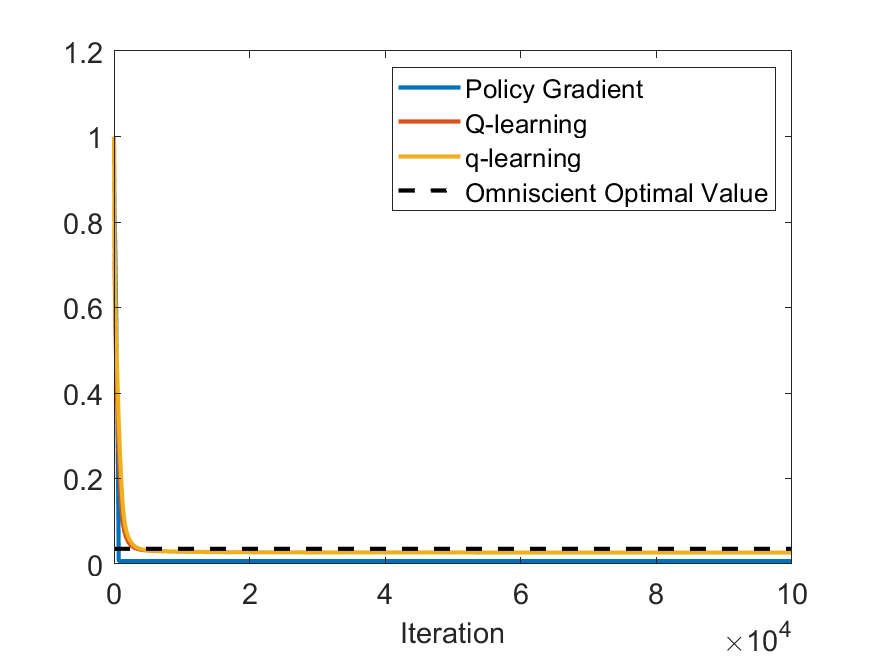}
\caption{$\Delta t= 1$}
\end{subfigure}
\begin{subfigure}{0.32\textwidth}
\centering
\includegraphics[width = 1\textwidth]{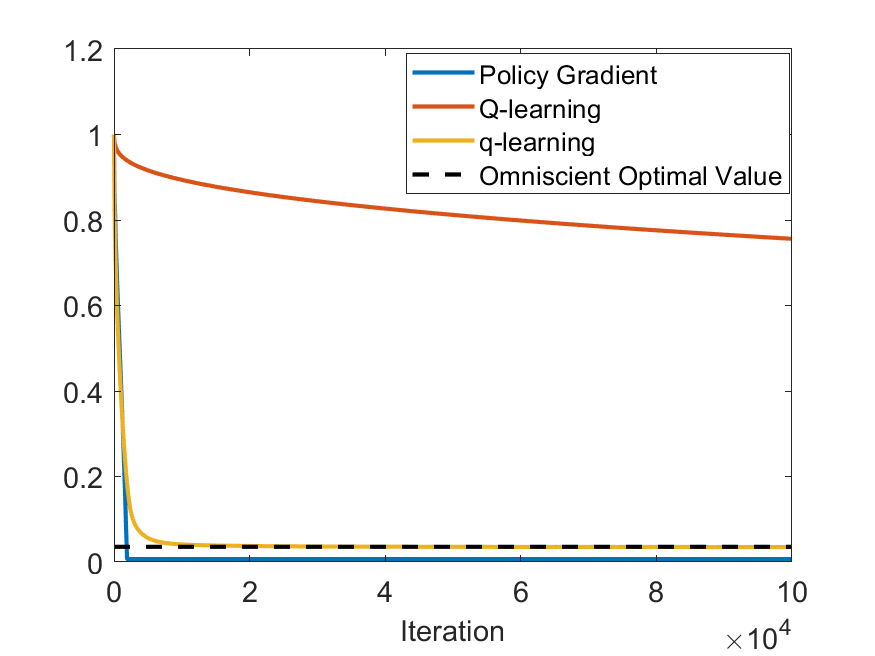}
\caption{$\Delta t= 0.1$}
\end{subfigure}
\begin{subfigure}{0.32\textwidth}
\centering
\includegraphics[width = 1\textwidth]{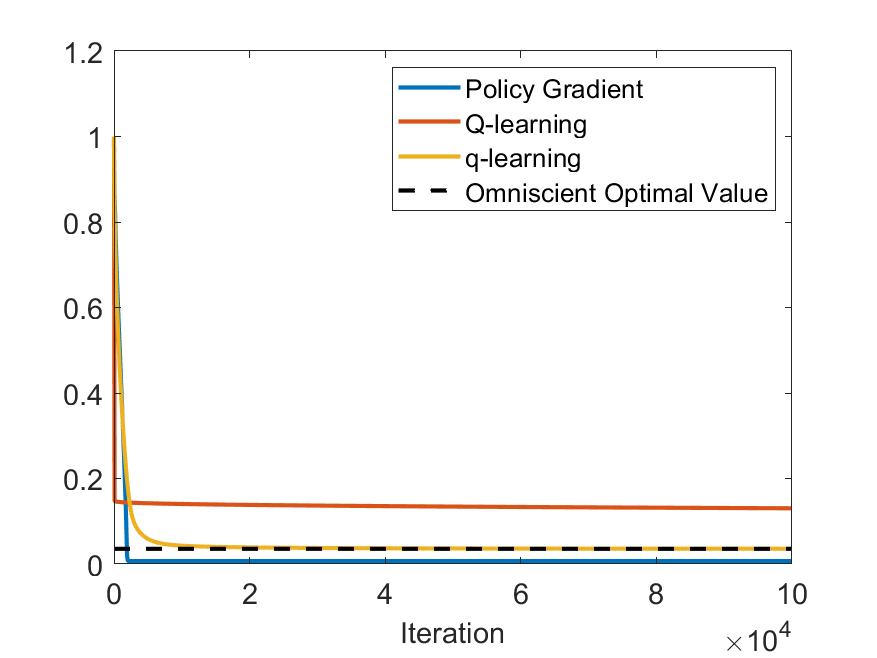}
\caption{$\Delta t= 0.01$}
\end{subfigure}
\caption{\textbf{Paths of learned parameters of three RL algorithms with different time discretization sizes.} A single state trajectory is generated with length $T = 10^6$ and discretized at different step sizes: $\Delta t=1$ in (a), $\Delta t=0.1$ in (b), and $\Delta t=0.01$ in (c),  under the behavior policy $a_t\sim \mathcal{N}(0,1)$. From top to bottom are the paths of learned $\psi_1,\psi_2,e^{\psi_3}$ respectively. For each step size, we apply  three algorithms: ``Policy Gradient" described in
Algorithm 3  in \cite{jia2021policypg}, ``Q-Learning" described  in Appendix B, and ``q-Learning" described in Algorithm \ref{algo:ergodic incremental}.
We repeat the experiments for 100 times for each method and plot the average reward over time with the shaded area indicating standard deviation. }
\label{fig:lq off policy}
\end{figure}

\section{Conclusion}
\label{sec:conclusion}

The previous trilogy on continuous-time RL under the entropy-regularized exploratory diffusion process framework \citep{wang2020reinforcement,jia2021policy,jia2021policypg} have respectively studied three key elements: optimal samplers for exploration, PE and PG. PG is a special instance of policy improvement, and in the discrete-time MDP literature, it can be done through Q-function. However, there was no satisfying Q-learning theory in continuous time, so a different route was taken in \citet{jia2021policypg} -- to turn the PG task into a PE problem.

This paper fills this gap and adds yet another essential building block to the theoretical foundation of continuous-time RL, by developing a q-learning theory commensurate with the continuous-time setting, attacking the general policy improvement problem, and covering both on- and off-policy learning. Although the conventional Q-function provides no information
about actions in continuous time, its first-order approximation, which we call the q-function, does. Moreover, it turns out that the essential component  of the q-function relevant to policy updating is the Hamiltonian associated with the problem, the latter having already played a vital role in the classical model-based stochastic control theory. This fact, together with the expression of the optimal stochastic policies explicitly derived in
\cite{wang2020reinforcement}, in turn, explains and justifies the widely used Boltzmann exploration in classical RL for MDPs.

We characterize the q-function as the compensator to maintain the martingality of a process consisting of the value function and cumulative reward, both in the on-policy and off-policy settings.  This characterization enables us to use the martingale-based PE techniques developed in \cite{jia2021policy} to simultaneously learn the q-function and the value function. Interestingly, the martingality is on the same process as in PE \citep{jia2021policy} but with respect to an enlarged filtration containing the policy randomization. The TD version of the resulting algorithm links to the well-known SARSA algorithm in classical  Q-learning.

A significant and outstanding research question is an analysis on the convergence rates of q-learning algorithms. The existing convergence rate results for discrete-time Q-learning cannot be extended to continuous-time q-learning due to the continuous state space and general nonlinear function approximations, along with the fact that the behaviors of the Q-function and the q-function can be fundamentally different. A possible remedy is to carry out the convergence analysis within the general framework of stochastic approximation, which is poised to  be the subject of a substantial future study.



We make a final observation to conclude this paper. The classical model-based approach typically separates ``estimation" and ``optimization", {\it \`a la} Wonham's separation theorem \citep{W68} or the ``plug-in" method. This approach first learns a model (including formulating and estimating its coefficients/parameters) and then optimizes over the learned model. Model-free (up to some basic dynamic structure such as diffusion processes or Markov chains) RL takes a different route: it skips estimating a model and learns optimizing policies directly via PG or Q/q-learning. The q-learning theory in this paper suggests that, to learn policies, one needs to learn the q-function or the Hamiltonian. In other words, it is the Hamiltonian which is a specific {\it aggregation} of the model coefficients, rather than each and every individual model coefficient, that needs to be learned/estimated for optimization. Clearly, from a pure computational standpoint, estimating a single function -- the Hamiltonian -- is much more efficient and robust than estimating multiple functions ($b,\sigma,r,h$ in the present paper) in terms of avoiding or reducing over-parameterization, sensitivity to errors and accumulation of errors.
More importantly, however, \eqref{eq:q rate approximation} hints that the Hamiltonian/q-function can be learned through temporal differences of the value function, so the task of learning and optimizing can be accomplished in a data-driven way. This would not be the case if we chose to learn individual model coefficients separately. This observation, we believe, highlights the fundamental differences between the model-based and model-free approaches.

\acks{Zhou is supported  by a start-up grant and the Nie Center
for Intelligent Asset Management at Columbia University.
His work is also part of a Columbia-CityU/HK collaborative project that is supported by the InnoHK Initiative, The Government of the HKSAR, and the AIFT Lab. We  are grateful for comments from seminar and conference participants at Chinese University of Hong Kong, University of Iowa, Columbia University, Seoul National University, POSTECH, Ritsumeikan University, Shanghai University of Finance and Economics, Fudan University, The 2022 INFORMS Annual Meeting in Indianapolis, The 11th Annual Meeting of FE Branch in OR Society of China in Shijiazhuang, Conference in Memory of Tomas Bj\"ork in Stockholm, and Post-Bachelier Congress Workshop in Hong Kong. In particular, we  benefit from discussions with Jiro Akahori, Xuefeng Gao, Bong-Gyu Jang, Hyeng Keun Koo, Lingfei Li, Hideo Nagai, Jun Sekine,  Wenpin Tang, and David Yao.
We are especially indebted  to the three anonymous referees for their constructive and detailed  comments that have led to an improved version of the paper.
}


\newpage

\appendix

\section*{Appendix A. Soft Q-learning in Discrete Time}
We review the soft Q-learning for discrete-time Markov decision processes (MDPs) here and present the analogy to some of the major  results developed in the main text. It is interesting that, to our best knowledge, the martingale perspective for MDPs is not explicitly presented. This in turn suggests that a continuous-time study may provide new aspects and insights even for discrete-time RL.

For simplicity, we consider a time-homogeneous MDP $X=\{X_t,t=0,1,2,\cdots\}$ with a state space $\mathcal{X}$, an action space $\mathcal{A}$, and a transition matrix $\p(X_1 = x'|X_0 = x, a_0 = a) = : p(x'|x,a)$. Both $\mathcal{X}$ and $\mathcal{A}$ are finite sets. The expected reward at $(x,a)$ is $r(x,a)$ with a discount factor $\beta\in (0,1)$. The agent's total expected reward is $\E\left[\sum_{t=0}^{\infty}\beta^t r(X_t,a_t) \right]$. A (stochastic) policy is denoted by $\bm\pi(\cdot|x)\in \mathcal{P}(\mathcal{A})$, which is a probability mass function on  $\mathcal{A}$.

\subsection*{Appendix A1. Q-function associated with an arbitrary policy}
Define the value function associated with a given policy $\bm\pi$ by
\begin{equation}
\label{eq:value mdp}
\begin{aligned}
J(x;\bm\pi) = & \E\left[\sum_{t=0}^{\infty}\beta^t \left[ r(X_t^{\bm\pi},a^{\bm\pi}_t) - \gamma\log\bm\pi(a^{\bm\pi}_t|X^{\bm\pi}_t)\right]\Big| X_0^{\bm\pi} = x \right] \\
= & \E\left[ r(x,a_0^{\bm\pi}) - \gamma\log\bm\pi(a_0^{\bm\pi}|x) \right] + \beta \E\left[J(X_1^{\bm\pi};\bm\pi)\Big| X_0^{\bm\pi} = x\right],
\end{aligned}
\end{equation}
and the Q-function associated with $\bm\pi$ by
\begin{equation}
\label{eq:Q mdp}
\begin{aligned}
Q(x,a;\bm\pi) = & r(x,a) +  \E\left[\sum_{t=1}^{\infty}\beta^t \left[ r(X_t^{\bm\pi},a^{\bm\pi}_t) - \gamma\log\bm\pi(a^{\bm\pi}_t|X^{\bm\pi}_t)\right]\Big| X_0^{\bm\pi} = x, a_0^{\bm\pi} = a \right] \\
= & r(x,a) + \beta \E\left[J(X_1^{a};\bm\pi)\Big| X_0^{\bm\pi} = x, a_0^{\bm\pi} = a \right].
\end{aligned}
\end{equation}
Adding the entropy term $- \gamma\log\bm\pi(a|x)$ to both sides of \eqref{eq:Q mdp}, integrating over $a$ and noting \eqref{eq:value mdp}, we obtain a relation between the Q-function and the value function :
\begin{equation}
\label{eq:constraint mdp}
\E\left[ Q(x,a^{\bm\pi};\bm\pi) - J(x;\bm\pi) - \gamma\log\bm\pi(a^{\bm\pi}|x) \right] = 0 .
\end{equation}
Moreover, substituting $J(X_1^{a};\bm\pi)$ with \eqref{eq:constraint mdp} to the right hand side of \eqref{eq:Q mdp}, we further obtain
\begin{equation}
\label{eq:sarsa mdp}
Q(x,a;\bm\pi) = r(x,a) + \beta \E\left[ Q(X_1^{a},a_1^{\bm\pi};\bm\pi) - \gamma\log\bm\pi(a_1^{\bm\pi}|X_1^{a})  \Big| X_0^{\bm\pi} = x, a_0^{\bm\pi} = a  \right] .
\end{equation}
Applying suitable algorithms (e.g., stochastic approximation) to solve the equation \eqref{eq:sarsa mdp} leads to the classical SARSA algorithm.

On the other hand, rewrite \eqref{eq:Q mdp} as
\begin{equation}
\label{eq:martingale mdp}
J(x;\bm\pi) = r(x,a) - [Q(x,a;\bm\pi) - J(x;\bm\pi)] + \beta  \E\left[J(X_1^{a};\bm\pi)\Big| X_0^{\bm\pi} = x, a_0^{\bm\pi} = a \right].
\end{equation}


Recall that $A(x,a;\bm\pi): = Q(x,a;\bm\pi) - J(x;\bm\pi)$ is called the \textit{advantage function} for MDPs. The q-function in the main text (see  \eqref{eq:q rate derivative}) is hence an advantage {\it rate} function in the continuous-time setting. In particular, \eqref{eq:constraint mdp} is analogous to the second constraint on q-function in \eqref{eq:constraints qv}.

Further, \eqref{eq:martingale mdp} implies that
$$M_s^{\bm\pi'} := \beta^{s} J(X^{\bm\pi'}_s;\bm\pi) + \sum_{t = 0}^{s-1} \beta^{t} \left[ r(X_t^{\bm\pi'},a_t^{\bm\pi'}) - A(X_t^{\bm\pi'},a_t^{\bm\pi'};\bm\pi)  \right] $$
is an $(\{\f_s\}_{s\geq 0},\p)$-martingale under {\it any}  policy $\bm\pi'$, where $\f_s$ is the $\sigma$-algebra generated by $X_0^{\bm\pi'},a_0^{\bm\pi'},\cdots, X_s^{\bm\pi'},a_s^{\bm\pi'}$. To see this, we compute
\begin{equation}
\label{eq:dismart}
\begin{aligned}
& \E\left[  \beta^{s} J(X^{\bm\pi'}_s;\bm\pi) + \sum_{t = 0}^{s-1} \beta^{t} \left[ r(X_t^{\bm\pi'},a_t^{\bm\pi'}) - A(X_t^{\bm\pi'},a_t^{\bm\pi'};\bm\pi)  \right]\Big| \f_{s-1}   \right] \\
= & \beta^s \E\left[J(X^{\bm\pi'}_s;\bm\pi) \Big| X^{\bm\pi'}_{s-1},a^{\bm\pi'}_{s-1} \right] + \sum_{t = 0}^{s-1} \beta^{t} \left[ r(X_t^{\bm\pi'},a_t^{\bm\pi'}) - A(X_t^{\bm\pi'},a_t^{\bm\pi'};\bm\pi)  \right] \\
= & \beta^{s-1} \left[ J(X_{s-1}^{\bm\pi'};\bm\pi) - r(X_{s-1}^{\bm\pi'},a_{s-1}^{\bm\pi'}) + A(X_{s-1}^{\bm\pi'}, a_{s-1}^{\bm\pi'};\bm\pi) \right] + \sum_{t = 0}^{s-1} \beta^{t} \left[ r(X_t^{\bm\pi'},a_t^{\bm\pi'}) - A(X_t^{\bm\pi'},a_t^{\bm\pi'};\bm\pi)  \right] \\
= & \beta^{s-1} J(X^{\bm\pi'}_{s-1};\bm\pi) + \sum_{t = 0}^{s-2} \beta^{t} \left[ r(X_t^{\bm\pi'},a_t^{\bm\pi'}) - A(X_t^{\bm\pi'},a_t^{\bm\pi'};\bm\pi)  \right] .
\end{aligned}
\end{equation}

Note that here $\f_s$ is larger than the usual ``historical information set'' up to time $s$, denoted by $\mathcal{H}_s$, which is the $\sigma$-algebra generated by $X_0^{\bm\pi'},a_0^{\bm\pi'},\cdots, X_s^{\bm\pi'}$ without observing the last $a_s^{\bm\pi'}$ (i.e. $\mathcal{H}_s$ is $\f_s$ {\it excluding} $a_s^{\bm\pi'}$). Because  $M_s^{\bm\pi'}$ is measurable to the smaller $\sigma$-algebra $\mathcal{H}_s$, it is automatically an $(\{\mathcal H_s\}_{s\geq 0}, \p)$-martingale.

The above analysis shows that $M_s^{\bm\pi'}$ is a martingale for {\it any} $\bm\pi'$ -- whether the target policy $\bm\pi$ or a behavior one -- so long as the advantage function $A$ is taken as the compensator in defining $M_s^{\bm\pi'}$. This is the essential reason why Q-learning works for both on- and off-policy learning. However, the conclusion is not necessarily true in other types of learning methods. For example, let us take a different compensator and investigate the process $\beta^{s} J(X^{\bm\pi'}_s;\bm\pi) + \sum_{t = 0}^{s-1} \beta^{t} \left[ r(X_t^{\bm\pi'},a_t^{\bm\pi'}) - \gamma\log\bm\pi(a_t^{\bm\pi'}|X_t^{\bm\pi'})  \right]$. The continuous-time counterpart of this process is a martingale required in the policy gradient method; see \cite[Theorem 4]{jia2021policypg}. 

%
%


Now, conditioned on $\mathcal{H}_{s-1}$ and when $\bm\pi' = \bm\pi$, we have
\[ \begin{aligned}
& \E\left[  \beta^{s} J(X^{\bm\pi}_s;\bm\pi) + \sum_{t = 0}^{s-1} \beta^{t} \left[ r(X_t^{\bm\pi},a_t^{\bm\pi}) - \gamma\log\bm\pi(a_t^{\bm\pi}|X_t^{\bm\pi})  \right]\Big| \mathcal{H}_{s-1}   \right] \\
= & \beta^s \E\left[ \E\left[ J(X^{\bm\pi}_s;\bm\pi) \big| X^{\bm\pi}_{s-1},a^{\bm\pi}_{s-1} \right] \Big|X^{\bm\pi}_{s-1} \right] + \sum_{t = 0}^{s-1} \beta^{t} \left[ r(X_t^{\bm\pi},a_t^{\bm\pi}) - \gamma\log\bm\pi(a_t^{\bm\pi}|X_t^{\bm\pi})  \right] \\
= & \beta^{s-1} \left[ J(X_{s-1}^{\bm\pi};\bm\pi) - r(X_{s-1}^{\bm\pi},a_{s-1}^{\bm\pi}) + A(X_{s-1}^{\bm\pi}, a_{s-1}^{\bm\pi};\bm\pi)\Big|X^{\bm\pi}_{s-1} \right] + \sum_{t = 0}^{s-1} \beta^{t} \left[ r(X_t^{\bm\pi},a_t^{\bm\pi}) - \gamma\log\bm\pi(a_t^{\bm\pi}|X_t^{\bm\pi}) \right] \\
= & \beta^{s-1} J(X^{\bm\pi}_{s-1};\bm\pi) + \sum_{t = 0}^{s-2} \beta^{t} \left[ r(X_t^{\bm\pi},a_t^{\bm\pi}) - \gamma\log\bm\pi(a_t^{\bm\pi}|X_t^{\bm\pi}) \right],
\end{aligned} \]
where the last equality is due to \eqref{eq:constraint mdp}, which can be applied because  $a_{s-1}^{\bm\pi}$ is generated under the policy $\bm\pi$. This shows that
the process is an $(\{\mathcal H_s\}_{s\geq 0}, \p)$-martingale when $\bm\pi' = \bm\pi$, which in turn underpins the on-policy learning.

However, when conditioned on $\f_{s-1}$, we have
\[ \begin{aligned}
& \E\left[  \beta^{s} J(X^{\bm\pi'}_s;\bm\pi) + \sum_{t = 0}^{s-1} \beta^{t} \left[ r(X_t^{\bm\pi'},a_t^{\bm\pi'}) - \gamma\log\bm\pi(a_t^{\bm\pi'}|X_t^{\bm\pi'})  \right]\Big| \f_{s-1}   \right] \\
= & \beta^s \E\left[J(X^{\bm\pi'}_s;\bm\pi) \Big| X^{\bm\pi'}_{s-1},a^{\bm\pi'}_{s-1} \right] + \sum_{t = 0}^{s-1} \beta^{t} \left[ r(X_t^{\bm\pi'},a_t^{\bm\pi'}) - \gamma\log\bm\pi(a_t^{\bm\pi'}|X_t^{\bm\pi'})  \right] \\
= & \beta^{s-1} \left[ J(X_{s-1}^{\bm\pi'};\bm\pi) - r(X_{s-1}^{\bm\pi'},a_{s-1}^{\bm\pi'}) + A(X_{s-1}^{\bm\pi'}, a_{s-1}^{\bm\pi'};\bm\pi) \right] + \sum_{t = 0}^{s-1} \beta^{t} \left[ r(X_t^{\bm\pi'},a_t^{\bm\pi'}) - \gamma\log\bm\pi(a_t^{\bm\pi'}|X_t^{\bm\pi'}) \right] \\
= & \beta^{s-1} J(X^{\bm\pi'}_{s-1};\bm\pi) + \sum_{t = 0}^{s-2} \beta^{t} \left[ r(X_t^{\bm\pi'},a_t^{\bm\pi'}) - \gamma\log\bm\pi(a_t^{\bm\pi'}|X_t^{\bm\pi'}) \right] \\
& + \beta^{s-1}\left[ A(X_{s-1}^{\bm\pi'}, a_{s-1}^{\bm\pi'};\bm\pi) -\gamma\log\bm\pi(a_t^{\bm\pi'}|X_t^{\bm\pi'}) \right] .
\end{aligned} \]
So the same process is {\it not} an $(\{\f_s\}_{s\geq 0}, \p)$-martingale in general, unless $A(X_{s-1}^{\bm\pi'}, a_{s-1}^{\bm\pi'};\bm\pi) =\gamma\log\bm\pi(a_t^{\bm\pi'}|X_t^{\bm\pi'})$ for any $X_{s-1}^{\bm\pi'}, a_{s-1}^{\bm\pi'}$. Two conclusions can be drawn from this example: on one hand, policy gradient works for on-policy but not off-policy, and on the other hand, it is important to choose the correct filtration to ensure martingality and hence the correct test functions in designing learning algorithms.

The analysis in \eqref{eq:dismart} yields a joint martingale characterization of $(J,A)$, analogous to  Theorem \ref{thm:qv learn}. It is curious that, to our best knowledge, such a martingale condition, while not hard to derive  in the MDP setting, has not been explicitly introduced nor utilized  to study Q-learning in the existing RL literature. 

\subsection*{Appendix A2. Q-function associated with the optimal policy}
Now consider the value function and Q-function that are associated with the optimal policy, denoted by $J^*$ and $Q^*$ respectively. Recall that the Bellman equation implies
\begin{equation}
\label{eq:bellman mdp}
J^*(x) = \sup_{\bm\pi}\E_{a\sim\bm\pi}\left\{ r(x,a) - \gamma\log\bm\pi(a|x) + \beta \E\left[  J^*(X_1^{a})\Big| X_0 = x, a_0 = a  \right] \right\} .
\end{equation}
It follows from \eqref{eq:Q mdp} that
\begin{equation}
\label{eq:bellman Q}
Q^*(x,a) = r(x,a) +  \beta \E\left[  J^*(X_1^{a})\Big| X_0 = x, a_0 = a  \right].
\end{equation}
Hence \eqref{eq:bellman mdp} becomes
\[ J^*(x) = \sup_{\bm\pi}\E_{a\sim\bm\pi}\left[ Q^*(x,a) - \gamma\log\bm\pi(a|x) \right] . \]
Solving the optimization on the right hand side of the above we get the optimal policy $\bm\pi^*(a|x) \propto \exp\{ \frac{1}{\gamma}Q^*(x,a) \}$ or
$ \bm\pi^*(a|x)= \frac{\exp\{ \frac{1}{\gamma}Q^*(x,a) \}}{\sum_{a\in \mathcal{A}}\exp\{ \frac{1}{\gamma}Q^*(x,a) \}}$. Denote
\[
\operatorname{soft}_{\gamma}\max_{a} Q^*(x,a): =   \E_{a\sim\bm\pi^*}\left[ Q^*(x,a) - \gamma\log\bm\pi^*(a|x) \right] \equiv  \gamma \log \sum_{a\in \mathcal{A}}\exp\{ \frac{1}{\gamma}Q^*(x,a) \} . \]
Then the Bellman equation becomes $J^*(x) = \operatorname{soft}_{\gamma}\max_{a} Q^*(x,a)$ and \eqref{eq:bellman Q} reduces to
\begin{equation}
\label{eq:bellman q mdp}
Q^*(x,a) = r(x,a) + \beta \E\left[ \operatorname{soft}_{\gamma}\max_{a'} Q^*(X_1^{a},a') \Big| X_0 = x, a_0 = a  \right] ,
\end{equation}
which is  the foundation for (off-policy) Q-learning algorithms; see e.g. \cite[p. 131]{sutton2011reinforcement}.

On the other hand, because $J^*(x) = \operatorname{soft}_{\gamma}\max_{a} Q^*(x,a) = \gamma \log \sum_{a\in \mathcal{A}}\exp\{ \frac{1}{\gamma}Q^*(x,a) \}$, we have
\begin{equation}
\label{eq:constraint mdp 2}
\sum_{a\in \mathcal{A}}\exp\left\{\frac{1}{\gamma}[Q^*(x,a) - J^*(x) ] \right\} = 1.
\end{equation}
Recall the advantage function $A^*(x,a) = Q^*(x,a) - J^*(x)$. Thus \eqref{eq:constraint mdp 2} is analogous  to \eqref{eq:optimal q hjb}. Moreover, rearranging \eqref{eq:bellman Q} yields
\[ J^*(x) = r(x,a) - [Q^*(x,a) - J^*(x)] + \beta \E\left[  J^*(X_1^{a})\Big| X_0 = x, a_0 = a  \right]. \]
Based on a similar derivation to that in Appendix A1, we obtain
$$\beta^{s} J^*(X^{\bm\pi}_s) + \sum_{t = 0}^{s-1} \beta^{t} \left[ r(X_t^{\bm\pi},a_t^{\bm\pi}) - A^*(X_t^{\bm\pi},a_t^{\bm\pi})  \right] $$ is an $(\{\f_s\}_{s\geq 0},\p)$-martingale for any policy $\bm\pi$; hence it
is analogous  to the martingale characterization of optimal q-function in Theorem \ref{thm:q optimal}.

\section*{Appendix B. $\Delta t$-Parameterized Q-Learning}

Instead of the (little) q-learning approach developed in this paper, one can also apply the conventional (big) Q-learning
method to the Q-function ${Q}_{\Delta t}(t,x,a;\bm\pi)$ defined in \eqref{qdt}. Note that  $\Delta t>0$ becomes a {\it parameter} in the latter approach. In this Appendix,
we review this $\Delta t$-parameterized Q-learning and the associated Q-learning algorithms, which are used in the simulation experiments of the paper for comparison purpose.

%
%

Given a policy $\bm\pi$ and a time step $\Delta t>0$,  Q-learning focuses on learning ${Q}_{\Delta t}(t,x,a;\bm\pi)$. As with  q-learning, there are  two distinctive scenarios: when the normalizing constant is easily available and when it is not. 
In the following, we denote by ${Q}^{\psi}_{\Delta t}$  the function approximator to the Q-function where $\psi\in\Psi\subset \mathbb{R}^{L_{\psi}}$ is the finite dimensional parameter vector to be learned, 
and by $\bm\pi^{\psi}$ the associated policy where $\bm\pi^{\psi}(a|t,x)\propto \exp\{ \frac{1}{\gamma \Delta t}{Q}^{\psi}_{\Delta t}(t,x,a) \}$.

\subsection*{Appendix B1. When normalizing constant is available}
When the normalization constant $\int_{\mathcal{A}} \exp\{ \frac{1}{\gamma \Delta t}{Q}^{\psi}_{\Delta t}(t,x,a) \} \dd a$ can be explicitly calculated, we can obtain the exact expression of $\bm\pi^{\psi}$ as
\[\bm\pi^{\psi}(a|t,x) = \frac{\exp\{ \frac{1}{\gamma \Delta t}{Q}^{\psi}_{\Delta t}(t,x,a) \}}{\int_{\mathcal{A}} \exp\{ \frac{1}{\gamma \Delta t}{Q}^{\psi}_{\Delta t}(t,x,a) \} \dd a} . \]

Then, the conventional Q-learning method (e.g, SARSA, \citealt[Chapter 10]{sutton2011reinforcement}) leads to the following updating rule of the Q-function parameters $\psi$:
\begin{equation}
\label{eq:conventional q learning update rule}
\begin{aligned}
\psi \leftarrow \psi + \alpha_{\psi} \bigg( & {Q}^{\psi}_{\Delta t}(t+\Delta t,X_{t+\Delta t}^{a_t},a_{t+\Delta t}) - \gamma\log\bm\pi^{\psi}(a_{t+\Delta t}|t+\Delta t,X_{t+\Delta t}^{a_t})\Delta t \\
& - {Q}^{\psi}_{\Delta t}(t,X_t,a_t) +  r(t,X_t,a_t) \Delta t - \beta {Q}^{\psi}_{\Delta t}(t,X_t,a_t) \Delta t \bigg) \frac{\partial {Q}^{\psi}_{\Delta t}}{\partial \psi}(t,X_t,a_t) ,
\end{aligned}
\end{equation}
where $a_t\sim \bm\pi^{\psi}(\cdot|t,X_t)$ and $a_{t+\Delta t}\sim \bm\pi^{\psi}(\cdot|t+\Delta t,X_{t+\Delta t})$.

\subsection*{Appendix B2. When normalizing constant is unavailable}
When the normalizing constant is not available, we cannot compute the exact expression of $\bm\pi^{\psi}$. In this case we follow an analysis analogous to that in Subsection \ref{subsec:stronger PIT} by using a family of policies $\{ \bm\pi^{\phi}(a|t,x)\}_{\phi\in \Phi}$ whose densities can be easily calculated to approximate $\bm\pi^{\psi}$. 

Theorem \ref{lemma:policy improvement hjb constraint} implies that if we start from a policy $\bm\pi^{\phi}$, $\phi\in \Phi$, and evaluate its corresponding Q-function $\frac{1}{\gamma \Delta t}{Q}^{\psi}_{\Delta t}(t,x,a) \approx \frac{1}{\gamma}H\big( t,x,\cdot,\frac{\partial J}{\partial x}(t,x;\bm\pi^\phi), \frac{\partial^2 J}{\partial x^2}(t,x;\bm\pi^{\phi}) \big) + \frac{1}{\gamma \Delta t}J(t,x;\bm\pi^{\phi})  $, then a new policy $\bm\pi^{\phi'}$, $\phi'\in \Phi$, improves $\bm\pi^{\phi}$ so long as
\[
D_{KL}\left( \bm\pi^{\phi'}(\cdot|t,x)\Big|\Big| \exp\{ \frac{1}{\gamma \Delta t}{Q}^{\psi}_{\Delta t}(t,x,\cdot)\} \right) \leq D_{KL}\left( \bm\pi^{\phi}(\cdot|t,x)\Big|\Big| \exp\{ \frac{1}{\gamma \Delta t}{Q}^{\psi}_{\Delta t}(t,x,\cdot)\} \right).
\]

Thus, following the same lines of argument as in Subsection \ref{subsec:stronger PIT},
%
%
%
we derive the  updating rule of  $\phi$ every step as
\[ \phi \leftarrow \phi - \alpha_{\phi}\left[\log\bm\pi^{\phi}(a_t|t,X_t) - \frac{1}{\gamma \Delta t}{Q}^{\psi}_{\Delta t}(t,X_t,a_t)  \right] \frac{\partial }{\partial \phi}\log\bm\pi^{\phi}(a_t|t,X_t) , \]
where $a_t\sim \bm\pi^{\phi}(\cdot|t,X_t)$.

\subsection*{Appendix B3. Ergodic tasks}

For the general regularized  ergodic problem formulated in Section \ref{sec:extension}, we define the Q-function:
\begin{equation}
\label{eq:q bar ergodic}
\begin{aligned}
& {Q}_{\Delta t}(x,a;\bm\pi)
= J(x;\bm{\pi}) + \big[ H\big( x,a,\frac{\partial J}{\partial x}(x;\bm{\pi}), \frac{\partial^2 J}{\partial x^2}(x;\bm{\pi})\big) - V(\bm\pi) \big]  \Delta t - \bar{J} .
\end{aligned}
\end{equation}
Then we have all the parallel results regarding relation among  the Q-function, the  value function and the intertemporal increment of the Q-function, and devise algorithms accordingly.

For illustration, we only present the results when \[\bm\pi^{\psi}(a|x) = \frac{\exp\{ \frac{1}{\gamma \Delta t}{Q}^{\psi}_{\Delta t}(x,a) \}}{\int_{\mathcal{A}} \exp\{ \frac{1}{\gamma \Delta t}{Q}^{\psi}_{\Delta t}(x,a) \} \dd a}  \]
is explicitly available. The updating rules of the  parameters $\psi$ and $V$ are
\begin{equation}
\label{eq:conventional q learning ergodic update rule}
\begin{aligned}
\psi \leftarrow \psi + \alpha_{\psi} \bigg( & {Q}^{\psi}_{\Delta t}(X_{t+\Delta t}^{a_t},a_{t+\Delta t}) - \gamma\log\bm\pi^{\psi}(a_{t+\Delta t}|X_{t+\Delta t}^{a_t})\Delta t \\
& - {Q}^{\psi}_{\Delta t}(X_t,a_t) +  r(X_t,a_t) \Delta t - V \Delta t \bigg) \frac{\partial {Q}^{\psi}_{\Delta t}}{\partial \psi}(X_t,a_t) \\
V\leftarrow V + \alpha_V\bigg( & {Q}^{\psi}_{\Delta t}(X_{t+\Delta t}^{a_t},a_{t+\Delta t}) - \gamma\log\bm\pi^{\psi}(a_{t+\Delta t}|X_{t+\Delta t}^{a_t})\Delta t \\
& - {Q}^{\psi}_{\Delta t}(X_t,a_t) +  r(X_t,a_t) \Delta t - V \Delta t \bigg),
\end{aligned}
\end{equation}
where $a_t\sim \bm\pi^{\psi}(\cdot|X_t)$ and $a_{t+\Delta t}\sim \bm\pi^{\psi}(\cdot|X_{t+\Delta t})$.
%
%


\subsection*{Appendix B4. Q-learning for mean--variance portfolio selection}
Motivated by the form of the solution in \citet{wang2020continuous}, we parameterize the  Q-function by
\[{Q}_{\Delta t}^{\psi}(t,x,a;w) = -e^{-\psi_3(T-t)}[(x-w)^2 + \psi_2 a(x-w) + \frac{1}{2}e^{-\psi_1} a^2] + \psi_4(t^2 - T^2) + \psi_5 (t-T) + (w-z)^2. \]
Here we use $-e^{-\psi_1}<0$ to 
guarantee the concavity of the Q-function in $a$ so that the function can be maximized.

This Q-function in turn gives rise to an explicit form of policy, which is Gaussian:
\[ \bm\pi^{\psi}(\cdot|t,x;w)= \mathcal{N}\left( - \psi_2 e^{\psi_1}(x-w), \gamma \Delta t e^{\psi_3(T-t) + \psi_1}  \right)  \propto \exp\{ \frac{1}{\gamma \Delta t}{Q}_{\Delta t}^{\psi}(t,x,\cdot) \} . \]

We need  the following derivative in the corresponding Q-learning algorithms:
\[ \frac{\partial {Q}_{\Delta t}^{\psi}}{\partial \psi}(t,x,a;w) =
\begin{pmatrix}
\frac{1}{2}e^{-\psi_3(T-t)}e^{-\psi_1} a^2 \\
-e^{-\psi_3(T-t)}e^{-\psi_1} a(x-w)\\
(T-t)e^{-\psi_3(T-t)}[(x-w)^2 + \psi_2 a(x-w) + \frac{1}{2}e^{-\psi_1} a^2]  \\
t^2 - T^2\\
t-T
\end{pmatrix}
. \]

The parameters in the Q-function can then be updated based on \eqref{eq:conventional q learning update rule} and the Lagrange multiplier can be updated similarly as in Algorithm \ref{algo:offline episodic mv}.

\subsection*{Appendix B5. Q-learning for ergodic LQ control}
For the ergodic LQ control problem formulated in Subsection \ref{sec:application lq ergodic}, we can prove that the value function and the Hamiltonian are both quadratic in $(x,a)$; see,  e.g., \cite{wang2020reinforcement}. Therefore we parameterize the Q-function by a general quadratic function:
\[{Q}_{\Delta t}^{\psi}(x,a) = -\frac{e^{-\psi_3}}{2}(a - \psi_1 x - \psi_2)^2 + \psi_4 x^2 + \psi_5 x. \]
Here we omit the constant term since the value function is unique only up to a constant. We also use $-e^{-\psi_3}<0$ to ensure that  the Q-function is concave in $a$. The optimal value $V$ is an extra parameter.

The Q-function leads to
\[ \bm\pi^{\psi}(\cdot|x) = \mathcal{N}\left( \psi_1 x + \psi_2, \gamma \Delta t e^{\psi_3}   \right)\propto \exp\{ \frac{1}{\gamma \Delta t}{Q}_{\Delta t}^{\psi}(x,\cdot)  \} . \]

Finally, we compute the following derivative for use in Q-learning algorithms:
\[ \frac{\partial {Q}_{\Delta t}^{\psi}}{\partial \psi}(x,a) = \left( e^{-\psi_3}(a - \psi_1 x - \psi_2) x, e^{-\psi_3}(a - \psi_1 x - \psi_2) ,\frac{e^{-\psi_3}}{2}(a - \psi_1 x - \psi_2)^2, x^2,x \right)^\top . \]

Then the parameters can be updated based on \eqref{eq:conventional q learning ergodic update rule}.

\section*{Appendix C. Proofs of Statements}

\subsection*{Proof of Theorem \ref{lemma:policy improvement hjb}}
We first prove a preliminary result about the entropy maximizing density.
\begin{lemma}
\label{lemma:entropy max}
Let $\gamma>0$ and a measurable function $q:\mathcal{A}\to\mathbb{R}$ with $\int_{{\cal A}} \exp\{\frac{1}{\gamma}q(a)\} \dd a < \infty$ be given. Then  $\pi^*(a) = \frac{\exp\{\frac{1}{\gamma}q(a)\}}{\int_{{\cal A}} \exp\{\frac{1}{\gamma}q(a)\} \dd a} \in \mathcal{P}(\mathcal{A})$ is the unique maximizer of the following problem
\begin{equation}
\label{eq:entropy max}
\max_{\pi(\cdot)\in
\mathcal{P}(\mathcal{A})}\int_{\mathcal{A}} [q(a) - \gamma\log\pi(a)]\pi(a)\dd a .
\end{equation}	
\end{lemma}
\begin{proof}
Set $\nu = 1 - \log\int_{{\cal A}} \exp\{\frac{1}{\gamma}q(a)\} \dd a $ and consider a new problem:
\begin{equation}
\label{eq:entropy max dual}
\max_{\pi(\cdot)\in \mathcal{P}(\mathcal{A})}\int_{\mathcal{A}} \big\{ [q(a) - \gamma\log\pi(a)]\pi(a) + \gamma \nu\pi(a) \big\} \dd a - \gamma\nu .
\end{equation}

Noting $\int_{{\cal A}} \pi(a)\dd a = 1$ due to $\pi(\cdot)\in \mathcal{P}(\mathcal{A})$,  problems \eqref{eq:entropy max} and \eqref{eq:entropy max dual} are equivalent; 
hence we focus on problem \eqref{eq:entropy max dual}. Relaxing the constraint of being a density function we deduce
\begin{equation}
\label{eq:entropy max relax}
\begin{aligned}
& \max_{\pi(\cdot)\in \mathcal{P}(\mathcal{A})}\int_{\mathcal{A}} \big\{ [q(a) - \gamma\log\pi(a)]\pi(a) + \gamma \nu\pi(a) \big\} \dd a - \gamma\nu \\ \leq & \max_{\pi(\cdot) > 0}\int_{\mathcal{A}} \big\{ [q(a) - \gamma\log\pi(a)]\pi(a) + \gamma \nu\pi(a) \big\} \dd a - \gamma\nu \\
\leq & \int_{\mathcal{A}}  \max_{\pi > 0}\big\{ [q(a) - \gamma\log\pi]\pi + \gamma \nu\pi \big\} \dd a - \gamma\nu .
\end{aligned}
\end{equation}
The unique maximizer of the inner optimization on the right hand side of the above is 
\[ \exp\{ \frac{1}{\gamma}q(a) + \nu - 1  \} = \frac{\exp\{\frac{1}{\gamma}q(a)\}}{\int_{{\cal A}} \exp\{\frac{1}{\gamma}q(a)\} \dd a}  =: \pi^*(a). \]
As $\pi^*\in \mathcal{P}(\mathcal{A})$, it is the optimal solution to \eqref{eq:entropy max dual} or, equivalently, \eqref{eq:entropy max}. The uniqueness follows from that of the inner optimization in \eqref{eq:entropy max relax}.
\end{proof}

Now we are ready to prove Theorem \ref{lemma:policy improvement hjb}. To ease notation we use the q-function $q$ even though it had not been introduced prior to  Theorem \ref{lemma:policy improvement hjb} in the main text. We also employ the results of Theorem \ref{thm:q rate function property} whose proof does not rely on Theorem \ref{lemma:policy improvement hjb} nor its proof here.

Recall from \eqref{eq:q rate}, $q(t,x,a;\bm\pi)$ is defined to be
\[ q(t,x,a;\bm{\pi}) = \frac{\partial J}{\partial t}(t,x;\bm{\pi}) + H\left( t,x,a,\frac{\partial J}{\partial x}(t,x;\bm{\pi}), \frac{\partial^2 J}{\partial x^2}(t,x;\bm{\pi})\right)  -\beta J(t,x;\bm{\pi}) , \]
with the Hamiltonian $H$ defined in \eqref{eq:H}.
\vspace{1em}

\begin{proof}[Proof of Theorem \ref{lemma:policy improvement hjb}]

We work on the equivalent formulation \eqref{eq:model relaxed}--\eqref{eq:objective relaxed}. For two given admissible policies $\bm\pi$ and $\bm\pi'$, and any
$0\leq t\leq T$,  apply It\^o's lemma to the process $e^{-\beta s}J(s,X_s^{\bm\pi'};\bm\pi)$ from $s\in [t,T]$, which is the value function under $\bm\pi$ but over the state process under $\bm\pi'$:
\[\begin{aligned}
& e^{-\beta T}J(T,\tilde{X}_{T}^{\bm\pi'};\bm\pi) - e^{-\beta t}J(t,\tilde{X}_{t}^{\bm\pi'};\bm\pi) + \int_t^{T}e^{-\beta s}\int_{\mathcal{A}}  [r(s,\tilde{X}_s^{\bm\pi'},a) - \gamma\log\bm\pi'(a|s,\tilde{X}_s^{\bm\pi'}) ]\bm\pi'(a|s,\tilde{X}_s^{\bm\pi'})\dd a \dd s  \\
= & \int_t^{T} e^{-\beta s} \int_{\mathcal{A}} \big[\frac{\partial J}{\partial t}(s,\tilde{X}_s^{\bm\pi'};\bm{\pi}) + H\big( s,\tilde{X}_s^{\bm\pi'},a,\frac{\partial J}{\partial x}(s,\tilde{X}_s^{\bm\pi'};\bm{\pi}), \frac{\partial^2 J}{\partial x^2}(s,\tilde{X}_s^{\bm\pi'};\bm{\pi})\big)  -\beta J(s,\tilde{X}_s^{\bm\pi'};\bm{\pi})\\
& - \gamma\log\bm\pi'(a|s,X_s^{\bm\pi'}) \big]\bm\pi'(a|s,\tilde{X}_s^{\bm\pi'})\dd a \dd s  + \int_t^{T} e^{-\beta s}\frac{\partial }{\partial x}J(s,\tilde{X}_s^{\bm\pi'};\bm\pi) \circ \tilde{\sigma}\big(s,\tilde{X}_s^{\bm\pi'},\bm\pi'(\cdot|s,\tilde{X}_s^{\bm\pi'}) \big) \dd W_s \\
= & \int_t^{T} e^{-\beta s} \int_{\mathcal{A}} \big[q(s,\tilde{X}^{\bm\pi'}_s,a;\bm\pi) - \gamma\log\bm\pi'(a|s,X_s^{\bm\pi'}) \big]\bm\pi'(a|s,\tilde{X}_s^{\bm\pi'})\dd a \dd s \\
& + \int_t^{T} e^{-\beta s}\frac{\partial }{\partial x}J(s,\tilde{X}_s^{\bm\pi'};\bm\pi) \circ \tilde{\sigma}\big(s,\tilde{X}_s^{\bm\pi'},\bm\pi'(\cdot|s,\tilde{X}_s^{\bm\pi'}) \big) \dd W_s. \\
\end{aligned}   \]
Because $\bm\pi' = \mathcal{I}\bm\pi$, it follows from Lemma \ref{lemma:entropy max} that for any $(s,y)\in [0,T]\times \mathbb{R}^d$, we have
\[
\int_{\mathcal{A}} \big[q(s,y,a;\bm\pi) - \gamma\log\bm\pi'(a|s,y) \big]\bm\pi'(a|s,y)\dd a \geq \int_{\mathcal{A}} \big[q(s,y,a;\bm\pi) - \gamma\log\bm\pi(a|s,y) \big]\bm\pi(a|s,y)\dd a = 0,\]
where the equality is due to \eqref{eq:q hjb} in Theorem \ref{thm:q rate function property}.
Thus,
\[\begin{aligned}
& e^{-\beta T}J(T,\tilde{X}_{T}^{\bm\pi'};\bm\pi) - e^{-\beta t}J(t,\tilde{X}_{t}^{\bm\pi'};\bm\pi) + \int_t^{T}e^{-\beta s}\int_{\mathcal{A}}  [r(s,\tilde{X}_s^{\bm\pi'},a) - \gamma\log\bm\pi'(a|s,\tilde{X}_s^{\bm\pi'}) ]\bm\pi'(a|s,\tilde{X}_s^{\bm\pi'})\dd a \dd s  \\
\geq &  \int_t^{T} e^{-\beta s}\frac{\partial }{\partial x}J(s,\tilde{X}_s^{\bm\pi'};\bm\pi) \circ \tilde{\sigma}\big(s,\tilde{X}_s^{\bm\pi'},\bm\pi'(\cdot|s,\tilde{X}_s^{\bm\pi'}) \big) \dd W_s .
\end{aligned}   \]

The above argument and the resulting inequalities are also valid  when $T$ is replaced by $T\wedge u_n$, where $u_n = \inf\{s\geq t:|\tilde{X}_s^{\bm{\pi}'}| \geq n \}$ is a sequence of stopping times. Therefore,
\begin{equation}
\label{eq:proof value improve inequ}
\begin{aligned}
J(t,x;\bm\pi) \leq & \E^{\p^W}\bigg[e^{-\beta(T\wedge u_n-t)}h(\tilde{X}_{T\wedge u_n}^{\bm\pi'}) \\
& + \int_t^{T\wedge u_n} e^{-\beta(s-t)}\int_{\mathcal{A}}  [r(s,\tilde{X}_s^{\bm\pi'},a) - \gamma\log\bm\pi'(a|s,\tilde{X}_s^{\bm\pi'}) ]\bm\pi'(a|s,\tilde{X}_s^{\bm\pi'})\dd a \dd s \Big| \tilde{X}_t^{\bm\pi'} = x \bigg].
\end{aligned}
\end{equation}

It follows from Assumption \ref{ass:dynamic}-(iv), Definition \ref{ass:dynamic}-(iii) and the moment estimate in \citet[Lemma 2]{jia2021policypg} that there exist constants $\mu, C > 0$ such that
\[ \begin{aligned}
& \left|  \E^{\p^W}\bigg[ \int_t^{T\wedge u_n} e^{-\beta(s-t)}\int_{\mathcal{A}}  [r(s,\tilde{X}_s^{\bm\pi'},a) - \gamma\log\bm\pi'(a|s,\tilde{X}_s^{\bm\pi'}) ]\bm\pi'(a|s,\tilde{X}_s^{\bm\pi'})\dd a \dd s \Big| \tilde{X}_t^{\bm\pi'} = x \bigg] \right| \\
\leq & \E^{\p^W}\bigg[ \int_t^{T} \int_{\mathcal{A}}  |r(s,\tilde{X}_s^{\bm\pi'},a) - \gamma\log\bm\pi'(a|s,\tilde{X}_s^{\bm\pi'}) |\bm\pi'(a|s,\tilde{X}_s^{\bm\pi'})\dd a \dd s \Big| \tilde{X}_t^{\bm\pi'} = x \bigg] \\
\leq & C \int_t^T \E^{\p^W}\bigg[ (1 + |\tilde{X}_s^{\bm\pi'}|^{\mu}) \Big| \tilde{X}_t^{\bm\pi'} = x \bigg] \dd t \leq C(1 + |x|^{\mu}).
\end{aligned} \]
By the dominated convergence theorem, we have as $n\to \infty$,
\[\begin{aligned}
& \E^{\p^W}\bigg[ \int_t^{T\wedge u_n} e^{-\beta(s-t)}\int_{\mathcal{A}}  [r(s,\tilde{X}_s^{\bm\pi'},a) - \gamma\log\bm\pi'(a|s,\tilde{X}_s^{\bm\pi'}) ]\bm\pi'(a|s,\tilde{X}_s^{\bm\pi'})\dd a \dd s \Big| \tilde{X}_t^{\bm\pi'} = x \bigg] \\ \to & \E^{\p^W}\bigg[ \int_t^{T} e^{-\beta(s-t)}\int_{\mathcal{A}}  [r(s,\tilde{X}_s^{\bm\pi'},a) - \gamma\log\bm\pi'(a|s,\tilde{X}_s^{\bm\pi'}) ]\bm\pi'(a|s,\tilde{X}_s^{\bm\pi'})\dd a \dd s \Big| \tilde{X}_t^{\bm\pi'} = x \bigg] .
\end{aligned} \]

In addition,
\[ \begin{aligned}
& \E^{\p^W}\bigg[e^{-\beta(T\wedge u_n-t)}h(\tilde{X}_{T\wedge u_n}^{\bm\pi'}) \Big| \tilde{X}_t^{\bm\pi'} = x \bigg] \\ \leq & \E^{\p^W}\bigg[h(\tilde{X}_{T}^{\bm\pi'}) \one_{\{ \max_{t\leq s\leq T}|\tilde{X}_T^{\bm\pi'}| < n \}}\Big| \tilde{X}_t^{\bm\pi'} = x \bigg] + \E^{\p^W}\bigg[h(\tilde{X}_{u_n}^{\bm\pi'}) \one_{\{ \max_{t\leq s\leq T}|\tilde{X}_T^{\bm\pi'}| \geq n \}}\Big| \tilde{X}_t^{\bm\pi'} = x \bigg] \\
\leq & C\E^{\p^W}\bigg[ (1 + |\tilde{X}_T^{\bm\pi'}|^{\mu})\Big| \tilde{X}_t^{\bm\pi'} = x \bigg] + C(1+n^{\mu}) \E^{\p^W}\bigg[ \one_{\{ \max_{t\leq s\leq T}|\tilde{X}_T^{\bm\pi'}| \geq n \}}\Big| \tilde{X}_t^{\bm\pi'} = x \bigg] \\
\leq & C(1+|x|^{\mu}) + C(1+n^{\mu}) \frac{\E^{\p^W}\bigg[\max_{t\leq s\leq T}|\tilde{X}_T^{\bm\pi'}|^{\mu+1}\Big| \tilde{X}_t^{\bm\pi'} = x \bigg]}{n^{\mu+1}} \\
\leq & C(1+|x|^{\mu}) + \frac{C(1+n^{\mu})(1 + |x|^{\mu+1})}{n^{\mu+1}} < \infty.
\end{aligned} \]
Again, by the dominated convergence theorem, we have as $n\to \infty$,
\[\E^{\p^W}\bigg[e^{-\beta(T\wedge u_n-t)}h(\tilde{X}_{T\wedge u_n}^{\bm\pi'}) \Big| \tilde{X}_t^{\bm\pi'} = x \bigg] \to \E^{\p^W}\bigg[e^{-\beta(T-t)}h(\tilde{X}_{T}^{\bm\pi'}) \Big| \tilde{X}_t^{\bm\pi'} = x \bigg] .  \]
Hence, sending $n\to\infty$, we conclude from \eqref{eq:proof value improve inequ} that $J(t,x;\bm\pi)\leq J(t,x;\bm\pi')$. This proves that $\bm\pi' = \mathcal I\bm\pi$ improves upon $\bm\pi$. Moreover, if  $\mathcal I\bm\pi\equiv \bm\pi'  = \bm\pi$, then $J(t,x;\bm\pi)\equiv J(t,x;\bm\pi')=:J^*(t,x)$, which satisfies the PDE (\ref{eq:explore fk}). However,
Lemma \ref{lemma:entropy max} shows that
\[ \begin{aligned}
&\int_{\mathcal{A}} \left[ H\big( t,x,a,\frac{\partial J^*}{\partial x}(t,x), \frac{\partial^2 J^*}{\partial x^2}(t,x) \big) - \gamma \log \bm{\pi'}(a|t,x) \right]  \bm{\pi'}(a|t,x)\dd a\\
=&\sup_{\bm \pi \in \mathcal{P}(\mathcal{A})} \int_{\mathcal{A}} \left[ H\big( t,x,a,\frac{\partial J^*}{\partial x}(t,x), \frac{\partial^2 J^*}{\partial x^2}(t,x) \big) - \gamma \log \bm{\pi}(a) \right]  \bm{\pi}(a)\dd a.
\end{aligned}\]
This means that $J^*$ also satisfies the HJB equation \eqref{eq:explore hjb}, implying that $J^*$ is the optimal value function and hence $\bm\pi$ is the optimal policy.
\end{proof}

\subsection*{Proof of Proposition \ref{prop:Q expand}}
\begin{proof}
Breaking the full time period $[t,T]$ into subperiods $[t,t+\Delta t)$ and $[t+\Delta t, T]$, while  conditioning on the state at $t+\Delta t$ for the second subperiod, we obtain:
\begin{equation}\label{eq:proof qdt}
\begin{aligned}
& Q_{\Delta t}(t,x,a;\bm{\pi} )\\
= & \E^{\p}\bigg[ \int_t^{t+\Delta t} e^{-\beta(s-t)}r(s,X_s^{a},a)\dd s\\
& + \E^{\p}\big[ \int_{t+\Delta t}^T e^{-\beta(s-t)}[r(s,X_s^{\bm{\pi}},a_s^{\bm \pi})-\gamma\log \bm{\pi}(a_s^{\bm \pi}|s,X_s^{\bm{\pi}}) ]\dd s + e^{-\beta(T-t)}h(X_T^{\bm{\pi}})|X_{t+\Delta t}^a\big]\Big|X_t^{\tilde{\bm{\pi}}} = x \bigg] \\
= & \E^{\p^W}\bigg[ \int_t^{t+\Delta t} e^{-\beta(s-t)} r(s,X_s^{a},a) \dd s + e^{-\beta \Delta t}J(t+\Delta t, X_{t+ \Delta t}^{a};\bm{\pi})\Big|X_t^{\tilde{\bm{\pi}}} = x \bigg] \\
= & \E^{\p^W}\bigg[ \int_t^{t+\Delta t} e^{-\beta(s-t)} r(s,X_s^{a},a)\dd s + e^{-\beta \Delta t}J(t+\Delta t, X_{t+ \Delta t}^{a};\bm{\pi}) - J(t,x;\bm{\pi})\Big|X_t^{\tilde{\bm{\pi}}} = x \bigg] + J(t,x;\bm{\pi}) \\
= & \E^{\p^W}\bigg[ \int_t^{t+\Delta t} e^{-\beta(s-t)} \big[ \frac{\partial J}{\partial t}(s,X_s^a;\bm{\pi}) + H\big( s,X_s^{a},a,\frac{\partial J}{\partial x}(s,X_s^a;\bm{\pi}), \frac{\partial^2 J}{\partial x^2}(s,X_s^a;\bm{\pi})\big) \\
& -\beta J(s,X_s^a;\bm{\pi}) \big]\dd s\Big|X_t^{\tilde{\bm{\pi}}} = x \bigg]   + J(t,x;\bm{\pi}) \\
= & J(t,x;\bm{\pi}) + \left[ \frac{\partial J}{\partial t}(t,x;\bm{\pi}) + H\big( t,x,a,\frac{\partial J}{\partial x}(t,x;\bm{\pi}), \frac{\partial^2 J}{\partial x^2}(t,x;\bm{\pi})\big)  -\beta J(t,x;\bm{\pi}) \right]  \Delta t + o( \Delta t ),
\end{aligned}
\end{equation}
where the second to the last equality follows from It\^o's lemma and the last equality and the error order are due to the approximation of the integral involved.
\end{proof}
\subsection*{Proof of Corollary \ref{coro:q as derivative}}
\begin{proof}
Its proof directly follows from the results of Proposition \ref{prop:Q expand}.
\end{proof}
\subsection*{Proof of Theorem \ref{thm:q rate function property}}
\begin{proof}
First, \eqref{eq:q hjb} follows readily from its definition in Definition \ref{def:q rate function}, the Feynman--Kac formula  \eqref{eq:explore fk}, and the fact that $\frac{\partial J}{\partial t}(t,x;\bm\pi)$ and $\beta J(t,x;\bm\pi)$ both do not depend on action $a$.

\begin{enumerate}[(i)]
\item We now focus on \eqref{eq:martingale with q function}. Applying It\^o's lemma to the process $e^{-\beta s}J(s,X_s^{\bm\pi};\bm\pi)$, we obtain for $0\leq t <s\leq T$:
\[ \begin{aligned}
& e^{-\beta s} J(s,{X}_{s}^{\bm\pi};\bm\pi) - e^{-\beta t} J(t,x;\bm\pi) + \int_t^{s} e^{-\beta u} [r(u,{X}_{u}^{\bm\pi},a^{\bm\pi}_{u}) - \hat{q}(u,{X}_{u}^{\bm\pi},a^{\bm\pi}_{u}) ]\dd u \\
= & \int_t^{s} e^{-\beta u} \Big[\frac{\partial J}{\partial t}(u,{X}_{u}^{\bm\pi},a^{\bm\pi}_{u}) + H\big( u,{X}_{u}^{\bm\pi},a^{\bm\pi}_{u},\frac{\partial J}{\partial x}(u,{X}_{u}^{\bm\pi};\bm{\pi}), \frac{\partial^2 J}{\partial x^2}(u,{X}_{u}^{\bm\pi};\bm{\pi})\big)  -\beta J(u,{X}_{u}^{\bm\pi};\bm{\pi})\\
&  - \hat{q}(u,{X}_{u}^{\bm\pi},a^{\bm\pi}_{u}) \Big]\dd u  + \int_t^{s}e^{-\beta u} \frac{\partial }{\partial x}J(u,{X}_{u}^{\bm\pi};\bm\pi) \circ \dd W_{u} \\
=& \int_t^{s} e^{-\beta u} \big[q(u,{X}_{u}^{\bm\pi},a^{\bm\pi}_{u};\bm\pi) - \hat{q}(u,{X}_{u}^{\bm\pi},a^{\bm\pi}_{u}) \big]\dd u  + \int_t^{s}e^{-\beta u} \frac{\partial }{\partial x}J(u,{X}_{u}^{\bm\pi};\bm\pi) \circ \sigma(u,{X}_{u}^{\bm\pi},a^{\bm\pi}_{u}) \dd W_{u} .
\end{aligned}\]
Recall that $\{a_s^{\bm\pi}, t\leq s\leq T\}$ is $\{\f_s\}_{s\geq0}$-progressively measurable. So, if $\hat q\equiv q(\cdot,\cdot,\cdot;\bm\pi)$, then the above process, and hence \eqref{eq:martingale with q function}, is  an $(\{\f_s\}_{s\geq 0},\p)$-martingale on $[t,T]$.

Conversely, if the right hand side of the above is a martingale,  then,  because its second term is a local martingale, we have that
$\int_t^{s} e^{-\beta u} \big[q(u,{X}_{u}^{\bm\pi},a^{\bm\pi}_{u};\bm\pi) - \hat{q}(u,{X}_{u}^{\bm\pi},a^{\bm\pi}_{u}) \big]\dd u$ is a continuous local martingale with finite variation and hence zero quadratic variation. Therefore, $\p$-almost surely, $\int_t^{s} e^{-\beta u} \big[q(u,{X}_{u}^{\bm\pi},a^{\bm\pi}_{u};\bm\pi) - \hat{q}(u,{X}_{u}^{\bm\pi},a^{\bm\pi}_{u}) \big]\dd u = 0$ for all $s\in [t,T]$;  see, e.g., \cite[Chapter 1, Exercise 5.21]{karatzas2014brownian}.

Denote $f(t,x,a) = q(t,x,a;\bm\pi) - \hat{q}(t,x,a)$. Then $f$  is a continuous function that maps $[0,T]\times \mathbb{R}^d \times\mathcal{A}$ to $\mathbb{R}$. Suppose the desired conclusion is not true, then 
there exists a triple $(t^*,x^*,a^*)$ and $\epsilon>0$ such that $f(t^*,
x^*,a^*) > \epsilon$. Because $f$ is continuous, there exists $\delta > 0$ such that $f(u,x',a') > \epsilon/2$ for all $(u,x',a')$ with $ |u-t^*|\vee |x'-x^*|\vee|a'-a^*| < \delta$. Here ``$\vee$'' means taking the larger one, i.e., $u \vee v = \max\{u,v\}$. 

Now consider the state process, still denoted by $X^{\bm\pi}$, starting from $(t^*,x^*,a^*)$,  namely,  $\{{X}_s^{\bm\pi}, t^*\leq s\leq T\}$ follows  (\ref{eq:model pi}) with ${X}_{t^*}^{\bm\pi}=x^*$ and ${a}_{t^*}^{\bm\pi}=a^*$.
Define
\[ \tau = \inf\{ u\geq t^*: |u-t^*| > \delta\text{ or } |{X}_{u}^{\bm\pi} - x^*| > \delta  \} = \inf\{ u\geq t^*:  |{X}_{u}^{\bm\pi} - x^*| > \delta  \} \wedge (t^*+\delta) .\]
The continuity of $X^{\bm\pi}$ implies that $u > t^*$,  $\p$-almost surely. Here ``$\wedge$'' means taking the smaller one, i.e., $u \wedge v = \min\{u,v\}$.

We have already proved that  there exists  $\Omega_0\in \f$ with $\p(\Omega_0) = 0$ such that for all $\omega \in \Omega\setminus\Omega_0$,
$\int_{t^*}^{s} e^{-\beta u} f(u,{X}_{u}^{\bm\pi}(\omega),a^{\bm\pi}_{u}(\omega))\dd u = 0$ for all $s\in [t^*,T]$. It follows from Lebesgue's differentiation theorem that for any $\omega \in \Omega\setminus\Omega_0$,
\[ f(s,{X}_{s}^{\bm\pi}(\omega),a^{\bm\pi}_{s}(\omega)) = 0 ,\text{ a.e. } s\in [t^*,\tau(\omega)] .\]
Consider the set $Z(\omega) =  \{s\in [t^*,\tau(\omega)]:  a^{\bm\pi}_{s}(\omega)\in \mathcal{B}_{\delta}(a^*) \} \subset [t^*,\tau(\omega)]$, where $\mathcal{B}_{\delta}(a^*) = \{a'\in \mathcal{A}: |a'-a^*| > \delta\}$ is the neighborhood of $a^*$. Because $f(s,{X}_{s}^{\bm\pi}(\omega),a^{\bm\pi}_{s}(\omega)) > \frac{\epsilon}{2}$
when $s\in Z(\omega)$, we conclude that $Z(\omega)$ has Lebesgue measure zero for
any $\omega \in \Omega\setminus\Omega_0$.  
That is,
\[ \int_{[t^*,T]} \one_{\{ s \leq \tau(\omega) \}}   \one_{\{ a^{\bm\pi}_{s}(\omega)\in \mathcal{B}_{\delta}(a^*) \}} \dd s=0. \]
Integrating $\omega$ with respect to $\p$ and applying Fubini's theorem, we obtain
\[\begin{aligned}
0 = & \int_{\Omega} \int_{[t^*,T]} \one_{\{ s \leq \tau(\omega) \}}   \one_{\{ a^{\bm\pi}_{s}(\omega)\in \mathcal{B}_{\delta}(a^*) \}} \dd s \p(\dd \omega) = \int_{[t^*,T]}\E[ \one_{\{ s \leq \tau \}}   \one_{\{ a^{\bm\pi}_{s}\in \mathcal{B}_{\delta}(a^*) \}} ] \dd s \\
= & \int_{t^*}^T \E\left[\one_{\{ s \leq \tau \}} \p\left(a^{\bm\pi}_{s}\in \mathcal{B}_{\delta}(a^*)  |\f_s\right) \right]\dd s = \int_{t^*}^T \E\left[\one_{\{ s \leq \tau \}} \int_{ \mathcal{B}_{\delta}(a^*) } \bm{\pi}(a|s,X^{\bm\pi}_s) \dd a \right]\dd s \\
\geq & \min_{|x'-x^*| < \delta, |u-t^*|<\delta} \left\{ \int_{ \mathcal{B}_{\delta}(a^*) } \bm{\pi}(a|u,x') \dd a \right\} \int_{t^*}^T \E\left[\one_{\{ s \leq \tau \}} \right]\dd s \\
= & \min_{|x'-x^*| < \delta, |u-t^*|<\delta} \left\{ \int_{ \mathcal{B}_{\delta}(a^*) } \bm{\pi}(a|u,x') \dd a \right\} \E[(\tau\wedge T) - t^*] \geq 0 .
\end{aligned} \]
Since $\tau > t^*$  $\p$-almost surely, the above  implies $\min_{|x'-x^*| < \delta, |u-t^*|<\delta} \left\{ \int_{ \mathcal{B}_{\delta}(a^*) } \bm{\pi}(a|u,x') \dd a \right\} = 0$. However, this contradicts Definition \ref{ass:admissible} about an admissible policy. Indeed, Definition \ref{ass:admissible}-(i) stipulates $\operatorname{supp}\bm\pi(\cdot|t,x) = \mathcal{A}$ for any $(t,x)$; hence $\int_{ \mathcal{B}_{\delta}(a^*) } \bm{\pi}(a|t,x) \dd a > 0$. Then the continuity in
Definition \ref{ass:admissible}-(ii) yields
$\min_{|x'-x^*| < \delta, |u-t^*|<\delta} \left\{ \int_{ \mathcal{B}_{\delta}(a^*) } \bm{\pi}(a|u,x') \dd a \right\} > 0$, a contradiction.
Hence we conclude $q(t,x,a;\bm\pi) = \hat{q}(t,x,a)$ for every $(t,x,a)$.

\item Applying It\^o's lemma to $e^{-\beta s}J(s,X_s^{\bm\pi'};\bm\pi)$, we get
\[  \begin{aligned}
& e^{-\beta s} J(s,{X}_{s}^{\bm\pi'};\bm\pi) - e^{-\beta t} J(t,x;\bm\pi) + \int_t^{s} e^{-\beta u} [r(u,{X}_{u}^{\bm\pi'},a^{\bm\pi'}_{u}) - \hat{q}(u,{X}_{u}^{\bm\pi'},a^{\bm\pi'}_{u}) ]\dd u \\
=& \int_t^{s} e^{-\beta u} \big[q(u,{X}_{u}^{\bm\pi'},a^{\bm\pi'}_{u};\bm\pi) - \hat{q}(u,{X}_{u}^{\bm\pi'},a^{\bm\pi'}_{u}) \big]\dd u  + \int_t^{s}e^{-\beta u} \frac{\partial }{\partial x}J(u,{X}_{u}^{\bm\pi'};\bm\pi) \circ \sigma(u,{X}_{u}^{\bm\pi'},a^{\bm\pi'}_{u}) \dd W_{u} .
\end{aligned} \]
So, when $\hat q\equiv q(\cdot,\cdot,\cdot;\bm\pi)$,  the above process is  an $(\{\f_s\}_{s\geq 0},\p)$-martingale on $[t,T]$.

\item Let $\bm\pi'\in \bm\Pi$ be given satisfying the assumption in this part. It then  follows from  (ii) that $\int_t^s e^{-\beta u} [\hat{q}(u,{X}_{u}^{\bm\pi'},a^{\bm\pi'}_{u}) - q(u,{X}_{u}^{\bm\pi'},a^{\bm\pi'}_{u};\bm\pi) ]\dd u$ is an $(\{\f_s\}_{s\geq 0}, \p)$-martingale. If the desired conclusion is not true, then the same argument in (i)  still applies to conclude that $\min_{|x'-x^*| < \delta, |u-t^*|<\delta} \left\{ \int_{ \mathcal{B}_{\delta}(a^*) } \bm{\pi}'(a|u,x') \dd a \right\} = 0$, which is a contradiction because $\bm\pi'$ is admissible. 
\end{enumerate}

\end{proof}

\subsection*{Proof of Theorem \ref{thm:qv learn}}
\begin{proof}
\begin{enumerate}[(i)]
\item We only prove the ``only if" part because the ``if" part is straightforward following the same argument as in the proof of Theorem \ref{thm:q rate function property}.

So, we assume $e^{-\beta t} \hat{J}(t,X_t^{{\bm\pi}}) + \int_0^t e^{-\beta s} [r(s,X_s^{{\bm\pi}},a_s^{{\bm\pi}}) - \hat{q}(s,X_s^{{\bm\pi}},a_s^{{\bm\pi}}) ]\dd s$
is an $(\{\f_s\}_{s\geq 0},\p)$-martingale. Hence for any initial state $(t,x)$, we have
\[ \E^{\p}\left[  e^{-\beta T} \hat{J}(T,X_T^{{\bm\pi}}) + \int_t^T e^{-\beta s} [r(s,X_s^{{\bm\pi}},a_s^{{\bm\pi}}) - \hat{q}(s,X_s^{{\bm\pi}},a_s^{{\bm\pi}}) ]\dd s \Big| \f_t \right] = e^{-\beta t}\hat{J}(t,x) . \]
We integrate over the action randomization with respect to the policy $\bm\pi$, and then obtain
\[ \E^{\p^W}\left[  e^{-\beta T} \hat{J}(T,\tilde X_T^{{\bm\pi}}) + \int_t^T e^{-\beta s} \int_{{\cal A}} [r(s,\tilde X_s^{{\bm\pi}},a) - \hat{q}(s,\tilde X_s^{{\bm\pi}},a) ]\bm\pi(a|s,\tilde{X}_s^{\bm\pi})\dd a \dd s \Big| \f^{W}_t \right] = e^{-\beta t}\hat{J}(t,x) \]
%


This, together with the terminal condition $\hat{J}(T,x) = h(x)$, and constraint \eqref{eq:q hjb2} yields that
\[ \hat{J}(t,x) = \E^{\p^W}\left[  e^{-\beta (T-t)} h(\tilde X_T^{{\bm\pi}}) + \int_t^T e^{-\beta (s-t)} \int_{{\cal A}} [r(s,\tilde X_s^{{\bm\pi}},a) - \gamma\log\bm\pi(s,\tilde X_s^{{\bm\pi}},a) ]\bm\pi(a|s,\tilde{X}_s^{\bm\pi})\dd a \dd s \Big| \f^{W}_t \right]. \]
Hence $\hat{J}(t,x) = J(t,x;{\bm\pi})$ for all $(t,x) \in [0,T]\times \mathbb{R}^d$. 
Furthermore, based on Theorem \ref{thm:q rate function property}, the martingale condition implies that 
$\hat{q}(t,x,a) = q(t,x,a;{\bm\pi})$ for all $(t,x,a) \in [0,T]\times \mathbb{R}^d \times \mathcal{A}$.

\item This follows immediately from Theorem \ref{thm:q rate function property}-(ii).
\item Let $\bm\pi'\in \bm\Pi$ be given satisfying the assumption in this part.
Define $$\hat r(t,x,a): = \frac{\partial \hat{J}}{\partial t}(t,x) + b(t,x,a)\circ \frac{\partial \hat{J}}{\partial x}(t,x) + \frac{1}{2}\sigma\sigma^\top(t,x,a)\circ \frac{\partial^2 \hat{J}}{\partial x^2}(t,x)  - \beta \hat J(t,x).$$
Then
\[ e^{-\beta s} \hat{J}(s,{X}_s^{\bm\pi'}) - \int_t^s e^{-\beta u} \hat r(u,{X}_{u}^{\bm\pi'},a^{\bm\pi'}_{u}) \dd u \]
is an $(\{\f_s\}_{s\geq 0}, \p)$-local martingale, which follows from applying It\^o's lemma to the above process and arguing similarly to the proof of Theorem \ref{thm:q rate function property}-(ii).
As a result,
$\int_t^s e^{-\beta u} [r(u,{X}_{u}^{\bm\pi'},a^{\bm\pi'}_{u}) - \hat{q}(u,{X}_{u}^{\bm\pi'},a^{\bm\pi'}_{u}) + \hat r(u,{X}_{u}^{\bm\pi'},a^{\bm\pi'}_{u})  ]\dd u$ is an $(\{\f_s\}_{s\geq 0}, \p)$-martingale. The same argument as in the proof of Theorem \ref{thm:q rate function property} applies and we conclude
\[\begin{aligned}
\hat{q}(t,x,a) = & r(t,x,a) + \hat r(t,x,a)  \\
= & \frac{\partial \hat{J}}{\partial t}(t,x) + b(t,x,a)\circ \frac{\partial \hat{J}}{\partial x}(t,x) + \frac{1}{2}\sigma\sigma^\top(t,x,a)\circ \frac{\partial^2 \hat{J}}{\partial x^2}(t,x)  - \beta \hat J(t,x) + r(t,x,a) \\
= & \frac{\partial \hat{J}}{\partial t}(t,x) + H\big( t,x,a,\frac{\partial \hat J}{\partial x}(t,x), \frac{\partial^2 \hat J}{\partial x^2}(t,x) \big)  - \beta \hat J(t,x)
\end{aligned}\]
for every $(t,x,a)$.
Now the constraint \eqref{eq:q hjb2} reads
\[ \int_{\mathcal A}  \left[ \frac{\partial \hat{J}}{\partial t}(t,x) + H\big( t,x,a,\frac{\partial \hat J}{\partial x}(t,x), \frac{\partial^2 \hat J}{\partial x^2}(t,x) \big)  - \beta \hat J(t,x) - \gamma\log\bm\pi(a|t,x) \right]   \bm\pi(a|t,x) \dd a= 0, \]
for all $(t,x)$, which,  together with the terminal condition $\hat J(T,x) = h(x)$, is the Feynman--Kac PDE \eqref{eq:explore fk} for $\hat J$. Therefore, it follows from the uniqueness of the solution to \eqref{eq:explore fk} that $\hat J \equiv J(\cdot,\cdot;\bm\pi)$.  Moreover,  it follows from Theorem \ref{thm:q rate function property}-(iii) that $\hat q \equiv q(\cdot,\cdot,\cdot;\bm\pi)$.
\end{enumerate}

Finally, if ${\bm\pi}(a|t,x) = \frac{\exp\{  \frac{1}{\gamma}\hat{q}(t,x,a) \}}{\int_{{\cal A}} \exp\{  \frac{1}{\gamma}\hat{q}(t,x,a) \} \dd a }$, then
${\bm\pi}=\mathcal I {\bm\pi}$ where $\mathcal I$ is the map defined in Theorem \ref{lemma:policy improvement hjb}. This in turn implies 
${\bm\pi}$ is the optimal policy and, hence, $\hat{J}$ is the optimal value function
\end{proof}

\subsection*{Proof of Proposition \ref{qstar1}}
\begin{proof}
Divide both sides of  \eqref{eq:optimal q} by $\gamma$, take exponential and then integrate over ${\mathcal A}$. Comparing the resulting equation with
the exploratory HJB equation \eqref{eq:explore hjb simplified}, we get
\[
\gamma\log\left[\int_{\mathcal A} \exp\{ \frac{1}{\gamma} q^*(t,x,a) \}\dd a\right] = 0.
\]
This yields \eqref{eq:optimal q hjb}. The expression \eqref{eq:optimal pi and q} follows then from \eqref{eq:gibbs}.
\end{proof}

\subsection*{Proof of Theorem \ref{thm:q optimal}}
\begin{proof}
\begin{enumerate}[(i)]
\item
Let $\widehat{J^*} = J^*$ and $\widehat{q^*} = q^*$ be the optimal value function and optimal q-function respectively.
For any ${\bm\pi}\in {\bm\Pi}$, applying It\^o's lemma to the process $e^{-\beta s}{J^*}(s,X_s^{\bm\pi})$, we obtain for $0\leq t <s\leq T$:
\[ \begin{aligned}
& e^{-\beta s} {J^*}(s,{X}_{s}^{\bm\pi}) - e^{-\beta t} J^*(t,x) + \int_t^{s} e^{-\beta u} [r(u,{X}_{u}^{\bm\pi},a^{\bm\pi}_{u}) - {q^*}(u,{X}_{u}^{\bm\pi},a^{\bm\pi}_{u}) ]\dd u \\
= & \int_t^{s} e^{-\beta u} \Big[\frac{\partial {J^*}}{\partial t}(u,{X}_{u}^{\bm\pi},a^{\bm\pi}_{u}) + H\big( u,{X}_{u}^{\bm\pi},a^{\bm\pi}_{u},\frac{\partial {J^*}}{\partial x}(u,{X}_{u}^{\bm\pi}), \frac{\partial^2 {J^*}}{\partial x^2}(u,{X}_{u}^{\bm\pi})  -\beta {J^*}(u,{X}_{u}^{\bm\pi})\\
&  - {q^*}(u,{X}_{u}^{\bm\pi},a^{\bm\pi}_{u}) \Big]\dd u  + \int_t^{s}e^{-\beta u} \frac{\partial }{\partial x}{J^*}(u,{X}_{u}^{\bm\pi}) \circ \dd W_{u} \\
=&  \int_t^{s}e^{-\beta u} \frac{\partial }{\partial x}{J^*}(u,{X}_{u}^{\bm\pi}) \circ \sigma(u,{X}_{u}^{\bm\pi},a^{\bm\pi}_{u}) \dd W_{u} ,
\end{aligned}\]
where the $\int \cdots \dd u$ term vanishes due to the definition of $q^*$ in \eqref{eq:optimal q}. Hence \eqref{eq:martingale with q function2 optimal} is a martingale.

\item The second constraint in \eqref{eq:q hjb2 optimal} implies that $\widehat{\bm\pi^*}(a|t,x): = \exp\{  \frac{1}{\gamma}\widehat{q^*}(t,x,a) \}$ is a probability density function, and $\widehat{q^*}(t,x,a) = \gamma \log\widehat{\bm\pi^*}(a|t,x)$.
So $\widehat{q^*}(t,x,a)$ satisfies the second constraint in \eqref{eq:q hjb2} with respect to the policy $\widehat{\bm\pi^*}$. When \eqref{eq:martingale with q function2 optimal} is an $(\{\f_s\}_{s\geq 0},\p)$-martingale under the given  admissible policy $\bm\pi$, it follows from Theorem \ref{thm:qv learn}--(iii) that $\widehat{J^*}$ and $\widehat{q^*}$ are respectively the value function and  the q-function associated with $\widehat{\bm\pi^*}$. Then the improved policy is $\mathcal{I} \widehat{\bm\pi^*}(a|t,x): = \frac{\exp\{  \frac{1}{\gamma}\widehat{q^*}(t,x,a) \}}{\int_{{\cal A}} \exp\{  \frac{1}{\gamma}\widehat{q^*}(t,x,a) \} \dd a } = \exp\{  \frac{1}{\gamma}\widehat{q^*}(t,x,a) \} = \widehat{\bm\pi^*}(a|t,x)$. However,  Theorem \ref{lemma:policy improvement hjb} yields that $\widehat{\bm\pi^*}$ is optimal, completing the proof.

\end{enumerate}
\end{proof}

\subsection*{Proof of Theorem \ref{lemma:policy improvement hjb constraint}}
\begin{proof}
Let us compute the KL-divergence:
\[
\begin{aligned}
& \int_{\mathcal{A}} \big[\log \bm\pi'(a|t,x) - \frac{1}{\gamma}H\big( t,x,a,\frac{\partial J}{\partial x}(t,x;\bm\pi), \frac{\partial^2 J}{\partial x^2}(t,x;\bm\pi) \big)  \big] \bm\pi'(a|t,x) \dd a \\
= & D_{KL}\left( \bm\pi'(\cdot|t,x)\Big|\Big| \exp\{ \frac{1}{\gamma}H\big( t,x,\cdot,\frac{\partial J}{\partial x}(t,x;\bm\pi), \frac{\partial^2 J}{\partial x^2}(t,x;\bm\pi) \big)    \} \right) \\
\leq & D_{KL}\left( \bm\pi(\cdot|t,x)\Big|\Big| \exp\{ \frac{1}{\gamma}H\big( t,x,\cdot,\frac{\partial J}{\partial x}(t,x;\bm\pi), \frac{\partial^2 J}{\partial x^2}(t,x;\bm\pi) \big)    \} \right)\\
= & \int_{\mathcal{A}} \big[\log \bm\pi(a|t,x) - \frac{1}{\gamma}H\big( t,x,a,\frac{\partial J}{\partial x}(t,x;\bm\pi), \frac{\partial^2 J}{\partial x^2}(t,x;\bm\pi) \big)  \big] \bm\pi(a|t,x) \dd a.
\end{aligned}
\]
Hence $\int_{\mathcal{A}} \big[q(t,x,a;\bm\pi)-\gamma\log \bm\pi'(a|t,x) \big] \bm\pi'(a|t,x) \dd a \geq \int_{\mathcal{A}} \big[q(t,x,a;\bm\pi)-\gamma\log \bm\pi(a|t,x) \big] \bm\pi(a|t,x) \dd a$.
Following the same argument as in the proof of Theorem \ref{lemma:policy improvement hjb}, we conclude  $J(t,x;\bm\pi') \geq J(t,x;\bm\pi)$.

\end{proof}

\subsection*{Proof of Lemma \ref{lemma:ergodic feynmann-kac}}
\begin{proof}
The first statement 
has been prove in \citet[Lemma 7]{jia2021policypg}. We focus on the proof of the second  statement, which is similar to that of Theorem \ref{lemma:policy improvement hjb}.

For the two admissible policies $\bm\pi$ and $\bm\pi'$, we apply It\^o's lemma to the value function under $\bm\pi$ along the process $\tilde{X}^{\bm\pi'}$ to obtain
\[\begin{aligned}
& J(\tilde{X}_{u}^{\bm\pi'};\bm\pi) - J(\tilde{X}_{t}^{\bm\pi'};\bm\pi) + \int_t^{u}\int_{\mathcal{A}}  [r(\tilde{X}_s^{\bm\pi'},a) - \gamma\log\bm\pi'(a|\tilde{X}_s^{\bm\pi'}) ]\bm\pi'(a|\tilde{X}_s^{\bm\pi'})\dd a \dd s  \\
= & \int_t^{u}\int_{\mathcal{A}} \big[ H\big( \tilde{X}_s^{\bm\pi'},a,\frac{\partial J}{\partial x}(\tilde{X}_s^{\bm\pi'};\bm{\pi}), \frac{\partial^2 J}{\partial x^2}(\tilde{X}_s^{\bm\pi'};\bm{\pi})\big) - \gamma\log\bm\pi'(a|X_s^{\bm\pi'}) \big]\bm\pi'(a|\tilde{X}_s^{\bm\pi'})\dd a \dd s \\
&  + \int_t^{u} \frac{\partial }{\partial x}J(\tilde{X}_s^{\bm\pi'};\bm\pi) \circ \tilde{\sigma}\big(\tilde{X}_s^{\bm\pi'},\bm\pi'(\cdot|\tilde{X}_s^{\bm\pi'}) \big) \dd W_s \\
\geq & \int_t^{u}\int_{\mathcal{A}} \big[ H\big( \tilde{X}_s^{\bm\pi'},a,\frac{\partial J}{\partial x}(\tilde{X}_s^{\bm\pi'};\bm{\pi}), \frac{\partial^2 J}{\partial x^2}(\tilde{X}_s^{\bm\pi'};\bm{\pi})\big) - \gamma\log\bm\pi(a|X_s^{\bm\pi'}) \big]\bm\pi(a|\tilde{X}_s^{\bm\pi'})\dd a \dd s \\
&  + \int_t^{u} \frac{\partial }{\partial x}J(\tilde{X}_s^{\bm\pi'};\bm\pi) \circ \tilde{\sigma}\big(\tilde{X}_s^{\bm\pi'},\bm\pi'(\cdot|\tilde{X}_s^{\bm\pi'}) \big) \dd W_s \\
= & V(\bm\pi)(u-t) + \int_t^{u}\frac{\partial }{\partial x}J(\tilde{X}_s^{\bm\pi'};\bm\pi) \circ \tilde{\sigma}\big(\tilde{X}_s^{\bm\pi'},\bm\pi'(\cdot|\tilde{X}_s^{\bm\pi'}) \big) \dd W_s ,
\end{aligned}   \]
where the inequality is due to the same argument as in the proof of Theorem \ref{lemma:policy improvement hjb constraint}, and the last equality is due to the Feynman--Kac formula \eqref{eq:relaxed control f-k formula ergodic}.

Therefore, after a localization argument,  we have
\[ \frac{1}{T} \E\bigg[J(\tilde{X}_{T}^{\bm\pi'};\bm\pi) - J(x;\bm\pi) + \int_0^{T}\int_{\mathcal{A}}  [r(\tilde{X}_s^{\bm\pi'},a) - \gamma\log\bm\pi'(a|\tilde{X}_s^{\bm\pi'}) ]\bm\pi'(a|\tilde{X}_s^{\bm\pi'})\dd a \dd s \Big| \tilde{X}_0^{\bm\pi'} =x \bigg] \geq V(\bm\pi) .\]
Taking the limit $T\to \infty$, we obtain $V(\bm\pi') \geq V(\bm\pi)$.
\end{proof}

\subsection*{Proof of Theorem \ref{thm:qv learn ergodic}}
\begin{proof}
The proof  is parallel to that of Theorems \ref{thm:q rate function property} and \ref{thm:qv learn}, which  is omitted.
\end{proof}

\newpage
{ \centering{\Large{\bf{Erratum to ``q-Learning in Continuous Time"}}} }
\section*{1 Introduction}
\citet[\textit{Journal of Machine Learning Research}, 24(161), 1-61]{jia_q-learning_2022} introduce the q-function for continuous-time reinforcement learning (RL) with controlled diffusion processes, and provide martingale characterizations for learning the q-function and the value function in a data-driven fashion. An implicit assumption in \citet{jia_q-learning_2022} is the possibility of continuum independent sampling from a given admissible feedback policy $\bm\pi$. More precisely, at any time--state pair $(t,x)$, the agent generates an action $a_t\sim \bm\pi(\cdot|t,x)$, and then applies this action to the environment instantaneously. This procedure leads to the time--state--action--reward sequences (all continuous-time processes) $\{s,X_s,a_s,r_s:0\leq s\leq T\}$ that satisfy
\[ \dd X_s = b(s,X_s,a_s)\dd s+ \sigma(s,X_s,a_s)\dd W_s, \]
where
\[ a_s\sim \bm\pi(\cdot|s,X_s) ,\ r_s = r(s,X_s,a_s),\ \forall s\in [0,T] . \]
\subsection*{1.1 Measure-Theoretical Issue with Continuum Sampling}
The above sampling procedure requires continuum independent draws from a non-degenerate distribution, 
for which \citet{jia_q-learning_2022} refer to the Fubini extension framework of \citet{sun2006exact} that shows it is possible to extend the Lebesgue measure (in $t$) to accommodate ``essentially pairwise independent" continuum random variables. However, there is a gap in this treatment. Theoretically, the resulting action process $\{a_s:0\leq s\leq T\}$ needs to be progressively measurable for the integral $\int_0^T b(t,X_t,a_t)\dd t$ and the stochastic integral $\int_0^T \sigma(t,X_t,a_t)\dd W_t$ to be well defined. However,  
\citet[Remark 2.1]{szpruch2024} and \citet[Section 3]{bender2024grid} point out that it is not the case in general. 

While this represents a very delicate technical gap, the theoretical results in \citet{jia_q-learning_2022} are so important that we believe an erratum is warranted.

\subsection*{1.2 Discretely Sampled Processes}
\citet{szpruch2024,bender2024grid,jia2025accuracy} all propose  using (different versions of) time-discretely sampled action processes to overcome the measurability issue. In this erratum, we take the recent framework of \citet{jia2025accuracy}.

Consider another probability space $(\Omega^\xi, \f^\xi, \p^{\xi})$ and a measurable function $\phi :[0,T]\times \mathbb R^d\times \Omega^\xi\to \mathcal A$ such that for all $(t,x)\in [0,T]\times \mathbb R^d$, the $\mathcal A$-valued random variable $\phi(t,x,\xi)$ has the distribution $\bm \pi(\cdot|t,x)$. Let $\mathbb N_0 =\mathbb N\cup\{0\}$
and let $(\Omega^{\xi_n}, \f^{\xi_n}, \p^{\xi_n},   \xi_n)_{n\in \mathbb N_0}$
be independent copies of $(\Omega^\xi, \f^\xi, \p^{\xi},   \xi) $. Consider a probability space of the following form:
\begin{align}
	\label{eq:space}
	(\Omega, \f,\p):=\bigg(\Omega^W \times  \prod_{n=0}^\infty \Omega^{\xi_n}, \f^W\otimes \bigotimes_{n=0}^\infty
	\f^{\xi_n},
	\p^W \otimes
	\bigotimes_{n=0}^\infty
	\p^{\xi_n}
	\bigg),
\end{align}
where $(\Omega^W,\mathcal F^W, \mathbb P^W)$ is the probability space where the Brownian motion (representing the environmental noises) lives,
and for each $n\in \mathbb N_0$, $(\Omega^{\xi_n}, \f^{\xi_n}, \p^{\xi_n})$ supports the random variable $\xi_n$  used to generate the random actions. Moreover, we define the filtration $ \mathcal F_t := \sigma\{ (W_s)_{s\leq t}, (\xi_i)_{t_i\leq t}\} $, which is right continuous and satisfies the usual condition.

Given an admissible feedback policy $\bm\pi$ (see \citealt[Definition 1]{jia_q-learning_2022} for the precise definition), denoted by $\bm\pi\in \bm\Pi$, and $(t,x)\in [0,T)\times \mathbb R^d$, consider a time grid $\mathscr G_{t:T}= \{t=s_0<s_1<\ldots <s_n =T\} $   of $[t,T]$. We sample actions from $\bm\pi$ {\it only} at the grid points in $\mathscr G_{t:T}$. The corresponding  state process satisfies, for all $i=0,\ldots, n-1$ and all $s\in [s_i,s_{i+1}]$,
\begin{align}\label{erratum  eq:sampled_sde}
	\begin{split}
		X_s & = X_{s_i} + \int_{s_i}^s b(u, X_u,a_{s_i}) \dd u +
		\int_{s_i}^s \sigma(u, X_u,a_{s_i})\dd  W_u,
		\quad
		\textnormal{with $a_{s_i}= \phi(s_i,X_{s_i},\xi_i)$},
	\end{split}
\end{align}
which will be referred henceforth to as the \textit{discretely sampled state process}.\footnote{The term ``discretely" here is slightly misleading as the state process $\{X_s, t\leq s\leq T\}$ itself is still {\it continuous} in time $s$. It is the action that is sampled discretely in time from the policy $\bm\pi$.} \citet[Lemma 3.1]{jia2025accuracy} show that \eqref{erratum eq:sampled_sde} is a well-posed SDE whose solution has a continuous trajectory and is adapted to a smaller filtration $\mathcal G_s:=
\sigma\{ (W_u)_{u\le s},  (\xi_{i})_ {s_i < s} \}$. In addition, the action process $a_s = \sum_{i=0}^{n-1} \one_{\{ s\in [s_i, s_{i+1}) \}} a_{s_i} $ is a simple process that is adapted to $\f_s$.

In the following, we denote by $a^{\mathscr G,\bm\pi}$ the resulting action process and by $X^{\mathscr G,\bm\pi}$ the solution to \eqref{erratum eq:sampled_sde}, given $X_t=x$, associated with the policy $\bm\pi$ and the grid $\mathscr G_{t:T}$. For simplicity, we may also rewrite \eqref{erratum eq:sampled_sde} as
\begin{equation}\label{erratum eq:sample_sde_abbre}
	\dd X_s = b(s,X_s,a_{\delta(s)})\dd s + \sigma(s,X_s,a_{\delta(s)}) \dd W_s, \quad s\in [t,T];  \quad X_t = x
\end{equation}
with  $\delta(s) = s_i$ for $s\in [s_i,s_{i+1})$, and $a_{s} = a_{\delta(s)}$ given in \eqref{erratum eq:sampled_sde}.

\section*{2 Martingale Characterizations for q-Learning with Discretely Sampled Processes}
We will now state and prove the revised martingale characterizations for q-learning, originally presented in \citet{jia_q-learning_2022}, in terms of the discretely sampled  state--action processes defined in \eqref{erratum eq:sample_sde_abbre}.
Note that the definition of the q-function is solely based on the ``exploratory problem" (the equations (8) and (9) in \citealt{jia_q-learning_2022}) and, hence, is independent of any discrete sampling. Moreover,  the value function, $J(\cdot,\cdot;\bm\pi)$, of a policy $\bm\pi\in \bm\Pi$ is now also based on the exploratory problem, i.e. the equation (9) in \citet{jia_q-learning_2022}.\footnote{In \citet{jia_q-learning_2022}, the value function is first defined on the continuously sampled control and state processes (see (7) therein), and then argued to be equivalent to the one based on the exploratory problem. The equation (7) in \citet{jia_q-learning_2022} has the same measurability issue, but can be replaced by discretely sampled processes. \citet{jia2025accuracy} show that the total expected reward under the discretely sampled policy converges to the value function (defined based on the exploratory problem) as the grid size tends to zero.} However, we will prove that Theorems 6, 7, 9, and 12 in \citet{jia_q-learning_2022} are all valid when ``$(X^{\bm\pi},a^{\bm\pi})$" therein (which
are not rigorously defined due to the measurability issue) is replaced by  ``$(\sX,\sa)$" with any given time grid.

In the following, for reader's convenience, we label the theorems with the same numbers corresponding to those in the original paper \citet{jia_q-learning_2022}. For example, Theorem 6 here is the revision of Theorem 6 therein.

The first theorem characterizes the q-function of a given admissible policy, assuming its value function is accessed.

\setcounter{theorem}{5}
\begin{theorem}
	\label{erratum thm:q rate function property}
	Let a policy $\bm\pi\in \bm\Pi$, its value function $J$  and a continuous function $\hat{q}:[0,T]\times \mathbb{R}^d\times \mathcal{A}\to \mathbb{R}$
	be given. Then 
	\begin{enumerate}[(i)]
		\item $\hat{q}(t,x,a)=q(t,x,a;\bm\pi)$ for all $(t,x,a)\in[0,T]\times\mathbb{R}^d\times \mathcal{A}$ if and only if for all $(t,x)\in[0,T]\times\mathbb{R}^d$ and any time grid $\mathscr{G}_{t:T}$ on $[t,T]$, the following process
		\begin{equation}
			\label{erratum eq:martingale with q function}
			e^{-\beta s} J(s,\sX_s;\bm\pi) + \int_t^s e^{-\beta u} [r(u,\sX_u,\sa_{u}) - \hat{q}(u,\sX_u,\sa_u) ]\dd u
		\end{equation}
		is an $(\{\f_s\}_{s\geq 0},\p)$-martingale, where $\{\sX_s, t\leq s\leq T\}$ is the solution to \eqref{erratum eq:sample_sde_abbre} under $\bm\pi$ with $\sX_t=x$.
		\item If $\hat{q}(t,x,a)=q(t,x,a;\bm\pi)$ for all $(t,x,a)\in[0,T]\times\mathbb{R}^d\times \mathcal{A}$, then given any $\bm\pi'\in \bm\Pi$,  for all $(t,x)\in[0,T]\times\mathbb{R}^d$ and any time grid $\mathscr{G}_{t:T}$ on $[t,T]$, the following process
		\begin{equation}
			\label{erratum eq:martingale with q function-off}
			e^{-\beta s} J(s,\sXp_s;\bm\pi) + \int_t^s e^{-\beta u} [r(u,\sXp_u,\sap_{u}) - \hat{q}(u,\sXp_u,\sap_u) ]\dd u
		\end{equation}
		is an $(\{\f_s\}_{s\geq 0},\p)$-martingale, where $\{\sXp, t\leq s\leq T\}$ is the solution to \eqref{erratum eq:sample_sde_abbre} under $\bm\pi'$ with initial condition $\sXp_t = x$.
		\item If there exists $\bm\pi'\in \bm\Pi$ such that for all $(t,x)\in[0,T]\times\mathbb{R}^d$ and any time grid $\mathscr{G}_{t:T}$ on $[t,T]$, \eqref{erratum eq:martingale with q function-off} is an $(\{\f_s\}_{s\geq 0}, \p)$-martingale with initial condition $\sXp_t = x$, then  $\hat{q}(t,x,a)=q(t,x,a;\bm\pi)$ for all $(t,x,a)\in[0,T]\times\mathbb{R}^d\times \mathcal{A}$.
	\end{enumerate}
	Moreover, in any of the three cases above, the q-function satisfies
	\begin{equation}
		\label{erratum eq:q hjb}
		\int_{\mathcal{A}} \big[ q(t,x,a;\bm{\pi} ) -\gamma \log\bm\pi(a|t,x) \big] \bm\pi(a|t,x)\dd a =0, \;\;\forall (t,x)\in[0,T]\times\mathbb{R}^d.
	\end{equation}	
\end{theorem}

\begin{proof}
	The proof of \eqref{erratum eq:q hjb} is the same as that in \citet{jia_q-learning_2022}. It suffices to show ($i$) to ($iii$). For simplicity,  denote
	\[ \mathcal L^a V(t,x) := \frac{\partial V}{\partial t}(t,x) + b(t,x,a) \circ \frac{\partial V}{\partial x}(t,x) + \frac{1}{2}\sigma\sigma^\top(t,x,a)\circ \frac{\partial^2 V}{\partial x^2}(t,x)  .\]
	
	\begin{enumerate}[(i)]
		\item
		
		First of all, we apply It\^o's lemma to \eqref{erratum eq:sample_sde_abbre} to obtain
		\[\begin{aligned}
			& e^{-\beta s} J(s,\sX_s;\bm\pi) + \int_t^s e^{-\beta u} \left[ r(u,\sX_u,\sa_{u}) - \hat{q}(u,\sX_u,\sa_u) \right] \dd u \\
			= & e^{-\beta t} J(t,x) + \int_t^s e^{-\beta u} \left[ \mathcal L^{\sa_{\delta(u)}} J(u, \sX_u) - \beta J(u,\sX_u) +  r(u,\sX_u,\sa_{\delta(u)}) - \hat{q}(u,\sX_u,\sa_{\delta(u)}) \right]\dd u \\
			&  + \int_t^s e^{-\beta u} \frac{\partial J}{\partial x}(u,\sX_u) \sigma(u,\sX_u, \sa_{\delta(u)})   \dd W_u \\
			= & e^{-\beta t} J(t,x) + \int_t^s e^{-\beta u} \left[q(u,\sX_u,\sa_{\delta(u)};\bm\pi) - \hat{q}(u,\sX_u,\sa_{\delta(u)}) \right]\dd u \\
			& + \int_t^s e^{-\beta u} \frac{\partial J}{\partial x}(u,\sX_u) \sigma(u,\sX_u, \sa_{\delta(u)})   \dd W_u.
		\end{aligned}  \]
		
		If $\hat{q}(t,x,a)=q(t,x,a;\bm\pi)$, then it follows from the moment estimates in  \citet[Lemma 3.1]{jia2025accuracy} for the discretely sampled state process $\sX$  that \eqref{erratum eq:martingale with q function} is a martingale.
		
		Conversely, if \eqref{erratum eq:martingale with q function} is a martingale, then
		\[ \int_t^s e^{-\beta u} \left[q(u,\sX_u,\sa_{\delta(u)};\bm\pi) - \hat{q}(u,\sX_u,\sa_{\delta(u)}) \right]\dd u \]
		is a martingale for all initial $(t,x)$ and any given time grid $\mathscr{G}_{t:T}$. The same argument in \citet{jia_q-learning_2022}, i.e., a martingale with zero quadratic variation has to be a constant, yields that $\p$-almost surely, \[\int_t^s e^{-\beta u} \left[q(u,\sX_u,\sa_{\delta(u)};\bm\pi) - \hat{q}(u,\sX_u,\sa_{\delta(u)}) \right]\dd u = 0\]
		for all $s\in [t,T]$. Denote $f(t,x,a) = q(t,x,a;\bm\pi) - \hat q(t,x,a)$. We prove $f\equiv 0$ by contradiction by assuming that there exists a triple $(t^*,x^*,a^*)\in [0,T) \times \mathbb R^d \times \mathcal A$ and $\epsilon>0$ such that $f(t^*,
		x^*,a^*) > \epsilon$. Because $f$ is continuous, there exists $\delta > 0$ such that $f(u,x',a') > \epsilon/2$ for all $(u,x',a')$ with $ |u-t^*|\vee |x'-x^*|\vee|a'-a^*| < \delta$. Here ``$\vee$'' is the maximum operator, i.e., $u \vee v = \max\{u,v\}$. 
		
		Now consider a discretely  sampled state process, still denoted by $\sX$,  starting from $(t^*,x^*)$ with a time grid $\mathscr{G}_{t:T}$ satisfying $t^*= t_0 < t^* + \delta < t_1 <\cdots$. Define
		\[ \tau = \inf\{ u\geq t^*: |u-t^*| > \delta\text{ or } |X^{\mathscr G,\bm\pi}_u - x^*| > \delta  \} = \inf\{ u\geq t^*:  |X^{\mathscr G,\bm\pi}_u - x^*| > \delta  \} \wedge (t^*+\delta) ,\]
		where ``$\wedge$'' denotes the minimum operator, i.e., $u \wedge v = \min\{u,v\}$. The continuity of $\sX$ implies that $\tau > t^*$,  $\p$-almost surely.
		
		We have already proved that there exists  $\Omega_0\in \f$ with $\p(\Omega_0) = 0$ such that for all $\omega \in \Omega\setminus\Omega_0$,
		$\int_{t^*}^{s} e^{-\beta u} f(u,\sX_u(\omega),\sa_{\delta(u)}(\omega))\dd u = 0$ for all $s\in [t^*,T]$. It follows from Lebesgue's differentiation theorem that for any $\omega \in \Omega\setminus\Omega_0$,
		\begin{equation}\label{erratum zero}
			f(s,\sX_s(\omega),\sa_{\delta(s)}(\omega)) = 0 ,\text{ a.e. } s\in [t^*,\tau(\omega)] .
		\end{equation}
		
		On the other hand, for the grid chosen above, for any $s\in [t^*,\tau(\omega)] \subset [t^*, t^* + \delta]$, $\sa_{\delta(s)}(\omega) = \sa_{t^*}(\omega) = \phi(t^*,x^*,\xi_0(\omega))$.
		Recall the definition of the admissible policy (Definition 1-(i) in \citealt{jia_q-learning_2022}), we have
		\[ \p( \phi(t^*,x^*,\xi_0(\omega)) \in \mathcal{B}_{\delta}(a^*)  ) =  \int_{ \mathcal{B}_{\delta}(a^*) } \bm{\pi}(a|t^*,x^*) \dd a  > 0, \]
		where $\mathcal{B}_{\delta}(a^*) = \{a'\in \mathcal{A}: |a'-a^*| < \delta\}$ is a neighborhood of $a^*$. Hence there exists $\omega\in \Omega\setminus \Omega_0$ such that for every $s\in [t^*,\tau(\omega)] $,
		\[ f(s,\sX_s(\omega),\sa_{\delta(s)}(\omega)) = f(s,\sX_s(\omega),\phi(t^*,x^*,\xi_0(\omega)) ) > \frac{\epsilon}{2} > 0,\]
		contradicting \eqref{zero}. This proves that $q(t,x,a;\bm\pi) = \hat{q}(t,x,a)$ for every $(t,x,a)$.

		\item The proof is parallel to the first part of the proof of (i).
		
		\item The proof is parallel to the second part of the proof of (i).
		
	\end{enumerate}
\end{proof}

The next result underpins both on-policy and off-policy RL algorithms for learning the value function and the q-function {\it jointly}.
\begin{theorem}
	\label{erratum thm:qv learn}
	Let a policy $\bm\pi\in \bm\Pi$, a function $\hat{J}\in C^{1,2}\big([0,T)\times \mathbb{R}^d \big) \cap C\big([0,T]\times \mathbb{R}^d \big)$ with polynomial growth, and a continuous function $\hat{q}:[0,T]\times \mathbb{R}^d\times \mathcal{A}\to \mathbb{R}$
	be given satisfying
	\begin{equation}
		\label{erratum eq:q hjb2}
		\hat{J}(T,x) = h(x),\;\;\; \int_{\mathcal{A}} \big[ \hat{q}(t,x,a) -\gamma \log{\bm\pi}(a|t,x) \big] {\bm\pi}(a|t,x)\dd a =0,\;\;\forall (t,x)\in[0,T]\times\mathbb{R}^d.
	\end{equation}
	Then 
	\begin{enumerate}[(i)]
		\item $\hat{J}$ and $\hat{q}$ are respectively the value function and the q-function associated with ${\bm\pi}$ if and only if for all $(t,x)\in[0,T]\times\mathbb{R}^d$ and any time grid $\mathscr{G}_{t:T}$ on $[t,T]$, the following process
		\begin{equation}
			\label{erratum eq:martingale with q function2}
			e^{-\beta s} \hat J(s,\sX_s) + \int_t^s e^{-\beta u} [r(u,\sX_u,\sa_u) - \hat{q}(u,\sX_u,\sa_u) ]\dd u
		\end{equation}
		is an $(\{\f_s\}_{s\geq 0},\p)$-martingale, where $\{\sX_s, t\leq s\leq T\}$ is the solution to \eqref{erratum eq:sample_sde_abbre} under $\bm\pi$ with $\sX_t=x$.
		\item If $\hat{J}$ and $\hat{q}$ are respectively the value function and the q-function associated with ${\bm\pi}$, then given any $\bm\pi'\in \bm\Pi$, for all $(t,x)\in[0,T]\times\mathbb{R}^d$ and any time grid $\mathscr{G}_{t:T}$ on $[t,T]$, the following process
		\begin{equation}
			\label{erratum eq:martingale with q function2-off}
			e^{-\beta s} \hat J(s,\sXp_s) + \int_t^s e^{-\beta u} [r(u,\sXp_u,\sap_u) - \hat{q}(u,\sXp_u,\sap_u) ]\dd u
		\end{equation}
		is an $(\{\f_s\}_{s\geq 0}, \p)$-martingale, where $\{\sXp_s, t\leq s\leq T\}$ is the solution to \eqref{erratum eq:sample_sde_abbre} under $\bm\pi'$ with $\sXp_t=x$.
		\item If there exists $\bm\pi'\in \bm\Pi$ such that for all $(t,x)\in[0,T]\times\mathbb{R}^d$ and any time grid $\mathscr{G}_{t:T}$ on $[t,T]$, \eqref{erratum eq:martingale with q function2-off} is an $(\{\f_s\}_{s\geq 0}, \p)$-martingale, where $\{\sXp_s, t\leq s\leq T\}$ is the solution to \eqref{erratum eq:sample_sde_abbre} under $\bm\pi'$ with $\sXp_t=x$, then  $\hat{J}$ and $\hat{q}$ are respectively the value function and the q-function associated with ${\bm\pi}$.
	\end{enumerate}
	Moreover, in any of the three cases above,
	if it holds further that ${\bm\pi}(a|t,x) = \frac{\exp\{  \frac{1}{\gamma}\hat{q}(t,x,a) \}}{\int_{{\cal A}} \exp\{  \frac{1}{\gamma}\hat{q}(t,x,a) \} \dd a }$, then ${\bm\pi}$ is the optimal policy and $\hat{J}$ is the optimal value function.
\end{theorem}
\begin{proof}
	\begin{enumerate}[(i)]
		\item We only prove the ``only if" part because the ``if" part is straightforward following the same argument as in the proof of Theorem \ref{erratum thm:q rate function property}.
		
		Define $\hat r(t,x,a): =\mathcal L^a \hat J(t,x)  - \beta \hat J(t,x)$ and consider the process
		\[ M_s = e^{-\beta s} \hat{J}(s,\sX_s) - \int_t^s e^{-\beta u} \hat r(u,\sX_u,\sa_u) \dd u . \]
		By applying It\^o's lemma  and arguing similarly to the first part of the proof of Theorem \ref{erratum thm:q rate function property}-(i), we obtain that $M_s$ is an $(\{\f_s\}_{s\geq 0}, \p)$-martingale. As a result,
		$\int_t^s e^{-\beta u} [r(u,\sX_u,\sa_u) - \hat{q}(u,\sX_u,\sa_u) + \hat r(u,\sX_u,\sa_u)  ]\dd u$ is an $(\{\f_s\}_{s\geq 0}, \p)$-martingale. The same argument as in the second part of the proof of Theorem \ref{erratum thm:q rate function property}-(i) applies, yielding
		\[\begin{aligned}
			\hat{q}(t,x,a) = & r(t,x,a) + \hat r(t,x,a)  \\
			= & \mathcal L^a \hat J - \beta \hat J(t,x) + r(t,x,a) \\
			= & \frac{\partial \hat{J}}{\partial t}(t,x) + H\big( t,x,a,\frac{\partial \hat J}{\partial x}(t,x), \frac{\partial^2 \hat J}{\partial x^2}(t,x) \big)  - \beta \hat J(t,x)
		\end{aligned}\]
		for every $(t,x,a)$. Now the constraint \eqref{erratum eq:q hjb2} reads
		\[ \int_{\mathcal A}  \left[ \frac{\partial \hat{J}}{\partial t}(t,x) + H\big( t,x,a,\frac{\partial \hat J}{\partial x}(t,x), \frac{\partial^2 \hat J}{\partial x^2}(t,x) \big)  - \beta \hat J(t,x) - \gamma\log\bm\pi(a|t,x) \right]   \bm\pi(a|t,x) \dd a= 0, \]
		for all $(t,x)$, which,  together with the terminal condition $\hat J(T,x) = h(x)$, is the Feynman--Kac PDE that characterizes the value function under the policy $\bm\pi$ (equation (11) in \citealt{jia_q-learning_2022}). Therefore, it follows from the uniqueness of the solution to the PDE to conclude that $\hat J \equiv J(\cdot,\cdot;\bm\pi)$.  Moreover,  it follows from Theorem \ref{erratum thm:q rate function property}-(i) that $\hat q \equiv q(\cdot,\cdot,\cdot;\bm\pi)$.

		\item This follows immediately from Theorem \ref{erratum thm:q rate function property}-(ii).
		\item The proof is parallel to the second part of the proof of (i).
	\end{enumerate}
	
	The last conclusion follows from the same argument in \citet{jia_q-learning_2022}.
\end{proof}

The following theorem concerns the optimal value function and optimal q-function.
\setcounter{theorem}{8}
\begin{theorem}
	\label{erratum thm:q optimal}
	Let a function $\widehat{J^*}\in C^{1,2}\big([0,T)\times \mathbb{R}^d \big) \cap C\big([0,T]\times \mathbb{R}^d \big)$ with polynomial growth and a continuous function $\widehat{q^*}:[0,T]\times \mathbb{R}^d\times \mathcal{A}\to \mathbb{R}$
	be given satisfying
	\begin{equation}
		\label{erratum eq:q hjb2 optimal}
		\widehat{J^*}(T,x) = h(x),\;\;\; \int_{\mathcal{A}} \exp\{ \frac{1}{\gamma} \widehat{q^*}(t,x,a) \} \dd a =1,\;\;\forall (t,x)\in[0,T]\times\mathbb{R}^d.
	\end{equation}
	Then 
	\begin{enumerate}[(i)]
		\item If $\widehat{J^*}$ and $\widehat{q^*}$ are respectively the optimal value function and the optimal q-function, then given any $\bm\pi\in \bm\Pi$, for  all $(t,x)\in[0,T]\times\mathbb{R}^d$ and any time grid $\mathscr{G}_{t:T}$ on $[t,T]$, the following process
		\begin{equation}
			\label{erratum eq:martingale with q function2 optimal}
			e^{-\beta s} \widehat{J^*}(s,\sX_s) + \int_t^s e^{-\beta u} [r(u,\sX_u,\sa_u) - \widehat{q^*}(u,\sX_u,\sa_u) ]\dd u
		\end{equation}
		is an $(\{\f_s\}_{s\geq 0},\p)$-martingale, where $\{\sX_s, t\leq s\leq T\}$ is the solution to \eqref{erratum eq:sample_sde_abbre} under $\bm\pi$ with $\sX_t=x$.
		Moreover, in this case, $\widehat{\bm\pi^*}(a|t,x) = \exp\{  \frac{1}{\gamma}\widehat{q^*}(t,x,a) \}$ is the optimal policy.
		\item If there exists one $\bm\pi\in \bm\Pi$ such that for all $(t,x)\in[0,T]\times\mathbb{R}^d$ and any time grid $\mathscr{G}_{t:T}$ on $[t,T]$, \eqref{erratum eq:martingale with q function2 optimal} is an $(\{\f_s\}_{s\geq 0},\p)$-martingale, then $\widehat{J^*}$ and $\widehat{q^*}$ are respectively the optimal value function and the optimal q-function. 
	\end{enumerate}
\end{theorem}
\begin{proof}
	\begin{enumerate}[(i)]
		\item The proof is parallel to the first part of Theorem \ref{erratum thm:q rate function property}-(i), while the optimality of $\widehat{\bm\pi^*}$ follows from Proposition 8 in \citet{jia_q-learning_2022}.
		
		\item The second constraint in \eqref{erratum eq:q hjb2 optimal} implies that $\widehat{\bm\pi^*}(a|t,x): = \exp\{  \frac{1}{\gamma}\widehat{q^*}(t,x,a) \}$ is a probability density function, and $\widehat{q^*}(t,x,a) = \gamma \log\widehat{\bm\pi^*}(a|t,x)$.
		So $\widehat{q^*}(t,x,a)$ satisfies the second constraint in \eqref{erratum eq:q hjb2} with respect to the policy $\widehat{\bm\pi^*}$. When \eqref{erratum eq:martingale with q function2 optimal} is an $(\{\f_s\}_{s\geq 0},\p)$-martingale under the given  admissible policy $\bm\pi$, it follows from Theorem \ref{erratum thm:qv learn}--(iii) that $\widehat{J^*}$ and $\widehat{q^*}$ are respectively the value function and  the q-function associated with $\widehat{\bm\pi^*}$. Then the improved policy is $\mathcal{I} \widehat{\bm\pi^*}(a|t,x): = \frac{\exp\{  \frac{1}{\gamma}\widehat{q^*}(t,x,a) \}}{\int_{{\cal A}} \exp\{  \frac{1}{\gamma}\widehat{q^*}(t,x,a) \} \dd a } = \exp\{  \frac{1}{\gamma}\widehat{q^*}(t,x,a) \} = \widehat{\bm\pi^*}(a|t,x)$. However,  Theorem 2 in \citet{jia_q-learning_2022} yields that $\widehat{\bm\pi^*}$ is optimal, completing the proof.
	\end{enumerate}
\end{proof}

The last result below deals with the case of ergodic tasks.

\setcounter{theorem}{11}
\begin{theorem}
	\label{erratum thm:qv learn ergodic}
	Let an admissible policy $\bm\pi$, a number $\hat{V}$, a function $\hat{J}\in C^{2}\big(\mathbb{R}^d \big)$ with polynomial growth, and a continuous function  $\hat{q}: \mathbb{R}^d\times \mathcal{A}\to \mathbb{R}$
	be given satisfying
	\begin{equation}
		\label{erratum eq:q hjb2eg}
		\lim_{T\to \infty}\frac{1}{T}\E[\hat{J}(\tilde X_T^{\bm\pi})] = 0, \;\;\;\int_{\mathcal{A}} \big[ \hat{q}(x,a) -\gamma \log{\bm\pi}(a|x) \big] {\bm\pi}(a|x)\dd a =0,\;\;\forall x\in \mathbb{R}^d,
	\end{equation}
	where $\tilde X^{\bm\pi}$ follows the exploratory dynamic (the equation (8) in \citet{jia_q-learning_2022}). Then
	\begin{enumerate}[(i)]
		\item $\hat{V}$, $\hat{J}$ and $\hat{q}$ are respectively the value, the value function and the q-function associated with ${\bm\pi}$ if and only if for all $x\in\mathbb{R}^d$ and any time grid $\mathscr{G}_{0:\infty}$, the following process
		\begin{equation}
			\label{erratum eq:martingale with q function2eg}
			\hat{J}(\sX_t) + \int_0^t  [r(\sX_u,\sa_u) - \hat{q}(\sX_u,\sa_u)  - \hat{V}]\dd u
		\end{equation}
		is an $(\{\f_t\}_{t\geq 0},\p)$-martingale, where $\{\sX_t, 0\leq t < \infty\}$ is the solution to \eqref{erratum eq:sample_sde_abbre} under $\bm\pi$ with $\sX_0=x$.
		\item If $\hat V$, $\hat{J}$ and $\hat{q}$ are respectively the value, value function and the q-function associated with ${\bm\pi}$, then given any
		admissible  $\bm\pi'$, for all $x\in\mathbb{R}^d$ and  any time grid $\mathscr{G}_{0:\infty}$, the following process
		\begin{equation}
			\label{erratum eq:martingale with q function2eg-off}
			\hat J(\sXp_t) + \int_0^t  [r(\sXp_u,\sap_u) - \hat{q}(\sXp_u,\sap_u) - \hat V ]\dd u
		\end{equation}
		is an $(\{\f_t\}_{t\geq 0}, \p)$-martingale, where $\{\sXp_t, 0 \leq t < \infty \}$ is the solution to \eqref{erratum eq:sample_sde_abbre} under $\bm\pi'$ with initial condition $\sXp_{0} = x$.
		\item If there exists an admissible $\bm\pi'$ such that for all $x\in\mathbb{R}^d$, \eqref{erratum eq:martingale with q function2eg-off} is an $(\{\f_t\}_{t\geq 0}, \p)$-martingale where ${X}_{0}^{\bm\pi'} = x$, then  $\hat V$, $\hat{J}$ and $\hat{q}$ are respectively the value, value function and the q-function associated with ${\bm\pi}$.
	\end{enumerate}
	Moreover, in any of the three cases above,
	if it holds further that  ${\bm\pi}(a|x) = \frac{\exp\{  \frac{1}{\gamma}\hat{q}(x,a) \}}{\int_{{\cal A}} \exp\{  \frac{1}{\gamma}\hat{q}(x,a) \} \dd a }$, then ${\bm\pi} $ is the optimal policy and $\hat{V}$ is the optimal value.
\end{theorem}
\begin{proof}
	The proof is parallel to those of Theorems \ref{erratum thm:q rate function property} and \ref{erratum thm:qv learn}, and hence omitted.
\end{proof}

Finally, an important remark is that the revision of the theoretical results in this erratum do not impact the algorithms and numerical experiments in \citet{jia_q-learning_2022}, because all the algorithms are naturally based on discretely sampled state processes.

\acks{The measurability issue was raised in \citet{szpruch2024,bender2024grid}, and by conference participants at The First INFORMS Conference on Financial Engineering and FinTech. The proofs of the revised theorems are based on the discussions with Xuefeng Gao and Lingfei Li while they are working on a related paper \citet{gao2024reinforcement} that extends the q-learning theory to jump diffusions. All errors are ours.}

\vskip 0.2in
\bibliography{reference}

\end{document}